\documentclass[journal]{IEEEtran}

\ifCLASSOPTIONcompsoc
  \usepackage[nocompress]{cite}
\else
  \usepackage{cite}
\fi

\usepackage{amsmath,amsfonts}
\usepackage{algorithmic}
\usepackage[ruled,linesnumbered,lined,commentsnumbered]{algorithm2e}
\usepackage{array}
\usepackage[caption=false,font=normalsize,labelfont=sf,textfont=sf]{subfig}
\usepackage{textcomp}
\usepackage{stfloats}
\usepackage{url}
\usepackage{verbatim}
\usepackage{graphicx}
\usepackage{cite}
\usepackage{bm}
\usepackage{amsfonts,amssymb}
\usepackage{amsmath}
\usepackage{mathrsfs}
\usepackage{stmaryrd}
\usepackage{eqparbox}
\usepackage{multirow}
\usepackage{array}
\usepackage{mdwmath}
\usepackage{arydshln}
\usepackage{booktabs}
\usepackage{mdwtab}
\usepackage{xcolor}
\usepackage{graphicx}
\usepackage{amsthm}
\usepackage{bbding}

\usepackage[colorlinks]{hyperref}

\newtheorem{theorem}{Theorem}[section]
\newtheorem{proposition}[theorem]{Proposition}
\newtheorem{lemma}[theorem]{Lemma}

\newtheorem{assumption}[theorem]{Assumption}
\newtheorem{remark}[theorem]{Remark}

\makeatletter

\begin{document}

\title{Feature Statistics with Uncertainty Help Adversarial Robustness}

\author{Ran Wang,~\IEEEmembership{Senior Member,~IEEE,} Xinlei Zhou, Meng Hu, Rihao Li, Wenhui Wu, and Yuheng Jia
\thanks{R. Wang, X. Zhou and M. Hu are with the School of Artificial Intelligence, Shenzhen University, Shenzhen 518060, China. (e-mail: wangran@szu.edu.cn, zhouxnli@szu.edu.cn, humeng@szu.edu.cn).}
\thanks{R. Li is with the School of Mathematical Sciences, Shenzhen University, Shenzhen 518060, China. (e-mail: 2450191005@mails.szu.edu.cn).}
\thanks{W. Wu is with the College of Electronic and Information Engineering, Shenzhen University, Shenzhen 518060, China. (e-mail: wuwenhui@szu.edu.cn).}
\thanks{Y. Jia is with the School of Computer Science and Engineering, Southeast University, Nanjing 210096, China (e-mail: yhjia@seu.edu.cn).}
\thanks{All source code will be made publicly available at \url{https://github.com/TechTrekkerZ/FSU} upon publication of the paper. The well-established models have already been uploaded to this link.}
}

\maketitle


\begin{abstract}
Despite the remarkable success of deep neural networks (DNNs), the security threat of adversarial attacks poses a significant challenge to the reliability of DNNs. In this paper, both theoretically and empirically, we discover a universal phenomenon that has been neglected in previous works, i.e., adversarial attacks tend to shift the distributions of feature statistics. Motivated by this finding, and by leveraging the advantages of uncertainty-aware stochastic methods in building robust models efficiently, we propose an uncertainty-driven feature statistics adjustment module for robustness enhancement, named {\it \textbf{F}eature \textbf{S}tatistics with \textbf{U}ncertainty} (FSU). It randomly resamples channel-wise feature means and standard deviations of examples from multivariate Gaussian distributions, which helps to reconstruct the perturbed examples and calibrate the shifted distributions. The calibration recovers some domain characteristics of the data for classification, thereby mitigating the influence of perturbations and weakening the ability of attacks to deceive models. The proposed FSU module has universal applicability in training, attacking, predicting, and fine-tuning, demonstrating impressive robustness enhancement ability at a trivial additional time cost. For example, by fine-tuning the well-established models with FSU, the state-of-the-art methods achieve up to 17.13\% and 34.82\% robustness improvement against powerful AA and CW attacks on benchmark datasets.
\end{abstract}

\begin{IEEEkeywords}
Adversarial robustness, uncertainty, feature statistics, distribution shift.
\end{IEEEkeywords}

\section{Introduction}

\IEEEPARstart{D}{eep} neural networks (DNNs) are vulnerable to adversarial attacks~\cite{adv_define}, which fool the deployed models into making false predictions by adding subtle perturbations to clean examples. Adversarial training~\cite{pgd,advtrain1} and its variants (e.g., TRADES~\cite{trades}, MART~\cite{mart}, MLCAT~\cite{yu2022Understanding}, F{\small IRE}~\cite{Picot2023adversarial}, AT+RiFT~\cite{Zhu2023Improving}, DKL~\cite{cui2024decoupled}, etc.) are the most popular defense methods to endow a DNN model with robustness. However, the overhead of adversarial training is usually several times or even dozens of times more expensive than natural training. Meanwhile, mitigating accuracy-robustness trade-off is a challenging task in adversarial training~\cite{Zhao2024Mitigating,Wei2024Revisiting}.

Recent studies have revealed that stochastic methods can enable the model to learn some uncertainty by introducing randomness into the networks, thereby improving model robustness efficiently. A group of works focuses on {\it \textbf{feature uncertainty}}. E.g., smoothed classifiers are obtained for certified robustness by applying a random smoothing technique \cite{Cohen2019Certified}. The input data is augmented by random transformation~\cite{random_xie} or Gaussian noise injection~\cite{add_noise_to_input} to resist perturbations. Random noises are injected into the hidden layer features with a hyperparameter diagonal variance matrix~\cite{rse_liu} or a fully parameterized variance matrix~\cite{wca,yang2023simple,yang2023weight}. The deep variational information bottleneck (VIB) method~\cite{Alemi2017Deep} is adopted in~\cite{Wang2024Adversarially} to learn uncertain feature distributions, together with a label embedding module, etc. Another group of works focuses on {\it \textbf{weight uncertainty}}. E.g., Adv-BNN~\cite{advbnn} treats the network parameters as random variables and learns the parameter distributions as in Bayesian neural networks (BNNs)~\cite{Carbone2020Robustness}. Parametric noise injection (PNI) method~\cite{pni_he} designs a learnable noise injection strategy for the weights of each hidden layer. Learn2Perturb (L2P) method~\cite{l2p} adopts alternating back-propagation to consecutively train the noise parameters. Adversarial weight perturbation (AWP) method~\cite{wu2020Adversarial} proposes a double-perturbation mechanism that adversarially perturbs both inputs and weights, etc.

The above stochastic methods, whether focusing on feature uncertainty or weight uncertainty, introduce randomness to individual examples or nodes. Many of them have difficulties capturing class-level or domain-level characteristics, and subsequently fail against powerful attacks, e.g., \cite{cw,Croce2020Reliable}. 

\begin{figure}[t]
\centering
\vspace{-0.2cm}
\subfloat[\small Natural]{\includegraphics[width=0.45\linewidth]{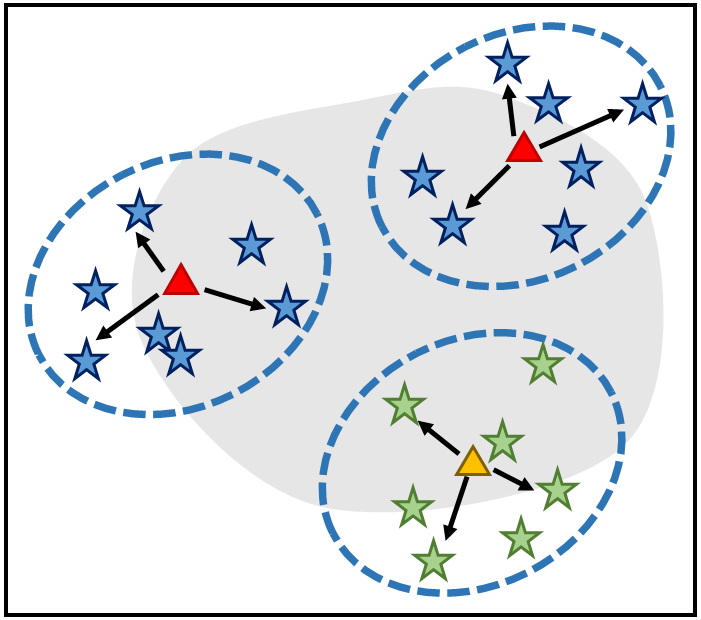}}\quad
\subfloat[\small Adversarial]{\includegraphics[width=0.45\linewidth]{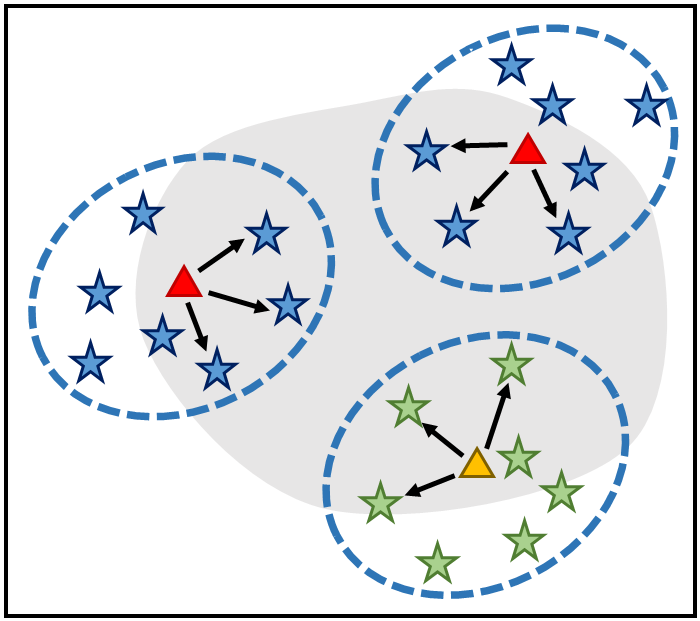}}\\
\vspace{-0.2cm}
\subfloat{\includegraphics[width=0.95\linewidth]{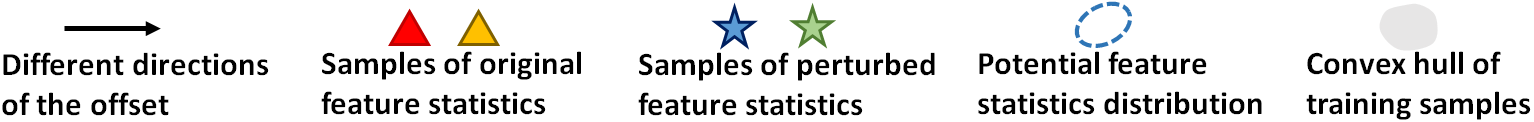}}
\vspace{-0.1cm}
\caption{Influences of natural perturbations and adversarial perturbations on feature statistics. Natural perturbation can be in any direction, while adversarial perturbation usually towards the misclassification direction.}
\vspace{-0.3cm}
\label{fig.motivation}
\end{figure}

It is argued in~\cite{dsu_li} that {\it \textbf{feature statistics (such as mean and variance)}} carry domain characteristics of training data, which can help improve out-of-distribution generalization. Under the assumption that the feature mean and standard deviation follow a multivariate Gaussian distribution, random noises are sampled to model their uncertainty, in order to alleviate the perturbations against potential domain shifts. However, this work only focuses on the domain shifts caused by natural perturbations, such as style shift, synthetic-to-real gap, scenes change, etc., while it pays no attention to adversarial perturbations. As shown in Fig.~\ref{fig.motivation}, the offset of feature statistics caused by natural perturbations can be in any directions within the distribution. In contrast, the offset caused by adversarial perturbations is typically directed toward the misclassification boundary to deceive the model. Inspired by this idea, we are encouraged to explore the distribution shifts of feature statistics caused by adversarial attacks, and promote model robustness by using {\it \textbf{F}eature \textbf{S}tatistics with \textbf{U}ncertainty} ({\bf FSU}). Ideally, the perturbed examples can be reconstructed by resampling their feature means and standard deviations from Gaussian distributions. Batch-level statistical information is used to estimate the distribution parameters, such that some domain characteristics for classification can be recovered. To summarize, our proposed method has the following merits:
\begin{itemize}
\item Both theoretically and empirically, we discover an important phenomenon that adversarial attacks tend to change the variance of feature statistics' distribution, thereby causing some potential distribution shifts. 
\item We propose an FSU module, which calibrates the distribution shifts and reconstructs perturbed examples in latent space by resampling their feature means and standard deviations based on batch-level statistical information. Since FSU has no parameters to learn, it can be employed in both training and testing phases.
\item Since the FSU module does not rely on adversarial information, it can be applied in both natural training and adversarial training. By introducing some directional information at the feature statistics-level, it can also be used for fine-tuning any well-established models.
\item Extensive experiments under various settings demonstrate the universal applicability of FSU for robustness enhancement, e.g., under powerful AA and CW attacks, FSU enables state-of-the-art methods to achieve up to {\bf 17.13}\% and {\bf 34.82}\% robustness improvement on benchmark datasets.
\end{itemize}

The remainder of this paper is organized as follows: Section~\ref{sec.related} provides a brief review of the related work. Section~\ref{sec.theory} makes some theoretical and empirical analyses on the distribution shifts of feature statistics caused by adversarial perturbations. Section~\ref{sec.FSU} proposes the FSU method in detail with some theoretical analyses. Section~\ref{sec.exp} presents extensive experiments to demonstrate the general effectiveness of the proposed method. Finally, Section~\ref{sec.conclusion} concludes this paper.

\section{Related Work}\label{sec.related}
In this section, we will present some related work on adversarial attack, adversarial robustness, and uncertainty-aware defense mechanism.

\subsection{Adversarial Attack}\label{subsec.attack}
Given a dataset $\mathbb{D}=\{(\mathbf{x}_i,y_i)\}_{i=1}^{N}\subset\mathbf{X}\times\mathbf{Y}$, where $\mathbf{x}_i=(x_{i1},\ldots,x_{id})\in\mathbf{X}=\mathbf{R}^d$ is the $i$-th example and $y_i\in\mathbf{Y}=\{1,2,.\ldots,m\}$ is the ground truth label of $\mathbf{x}_i$. Considering a classification model $F_{\theta}:\mathbf{X}\rightarrow\mathbf{Y}$, the essence of generating adversarial example for any $(\mathbf{x},y)\in\mathbb{D}$ through $F_{\theta}$ is to solve the following optimization problem:
\begin{equation}
\max\limits_{\delta} \ \mathcal{L}(F_{\theta}(\mathbf{x}+\delta),y),\ \mathrm{s.t.}\ \|\delta\|<\epsilon,
\end{equation}
such that $F_\theta(\mathbf{x}+\delta) \neq F_\theta(\mathbf{x})$, where $\delta$ is the perturbation generated for $\mathbf{x}$, $\mathbf{x}'=\mathbf{x}+\delta$ is the adversarial example, $\theta$ is the network parameter, $\mathcal{L}(\cdot)$ is the loss function, $\epsilon$ is the upper bound limit of $\delta$, and $\|\cdot\|$ is usually taken as $l_0$, $l_2$ or $l_{\infty}$ norm. 

According to whether the attackers can access the model structure and parameters, adversarial attacks can be categorized into white-box ones and black-box ones. More technically, according to the attacking strategy, adversarial attacks can be summarized into gradient-based attacks, optimization-based attacks, and adaptive ensemble attacks.

\subsubsection{Gradient-Based Attacks} This type of method uses the gradient information of the model to generate adversarial perturbations. The fast gradient sign method (FGSM)~\cite{goodfellow2014explaining} is a foundational work that calculates the gradient sign of the loss function with respect to the input sample in a single step. FGSM is efficient, but its success rate is limited. Subsequent studies improved FGSM through iterative optimization. E.g., the basic iterative method (BIM)~\cite{bim} extends FGSM by applying multi-step small perturbations; the momentum iterative method (MIM)~\cite{mim} introduces a momentum term to accelerate convergence and avoid local optimality; the projected gradient descent (PGD) method~\cite{pgd} applies random initialization and iterative projection, which is widely used as a strong baseline attack. Gradient-based methods rely on first-order gradient information and usually perform well in white-box settings with high computational efficiency.

\subsubsection{Optimization-Based Attacks} This type of method treats the perturbation generation as a constrained optimization problem and finds the optimal perturbation through a numerical algorithm. A typical method is the CW attack~\cite{cw}, i.e., by designing an alternative loss function (such as DLR loss) and using optimizers such as Adam, the perturbation is minimized under $l_2$-norm constraint while ensuring classification errors. CW attack can generate visually imperceptible perturbations and has very strong penetration against existing defense methods, but it requires a lot of forward propagations and gradient calculations, which is very time-consuming. Besides, evolutionary algorithms have been widely used as optimization-based black-box attacks due to their independence of the gradient information, such as the single-pixel attack (SPA)~\cite{Narodytska2016simple}, one-pixel attack (OPA)~\cite{onepixel} and sparse attack~\cite{Williams2023black}. 

\subsubsection{Adaptive Ensemble Attacks} In order to comprehensively evaluate the robustness of models, adaptive ensemble attack has become the mainstream method. The autoattack (AA) proposed by Croce et al.~\cite{Croce2020Reliable} is a representative, which integrates three white-box attacks (i.e., adaptive PGD-CE, adaptive PGD-DLR, fast adaptive boundary (FAB) attack~\cite{Croce2020Minimally}) and one black-box attack (i.e., square attack~\cite{Andriushchenko2020Square}). Due to its ability to automatically schedule different attacking strategies, AA has become a standard tool for robustness evaluation.

\subsection{Adversarial Robustness}\label{subsec.robustness}
For defending against adversarial attacks, the essence is to improve the robustness of model $F$ by optimizing $\theta$ or by reconstructing $\mathbf{x}+\delta$ such that $F_\theta(\mathbf{x}+\delta) = F_\theta(\mathbf{x})=y$. 

There exist different types of methods for improving model robustness: input preprocessing, adversarial training, structure optimization, adversarial example detection, distillation, etc. Among them, adversarial training is the most mainstream method, which adds adversarial examples into the training set and forces the model to maintain correct predictions under perturbations. The theoretical framework of adversarial training can be traced back to the FGSM attack~\cite{goodfellow2014explaining}, which verified that introducing FGSM-perturbed examples into the training set can effectively improve model robustness. Madry et al.~\cite{pgd} further established the standard adversarial training paradigm, which uses PGD attack as a benchmark to generate adversarial examples (PGD-AT), and formally defined the following bi-level optimization problem:

\begin{equation}
\min\limits_{\theta} \mathbb{E}_{\left(\mathbf{x},y\right)\sim \mathbb{D}} \max\limits_{||\delta||_p<\epsilon} \mathcal{L}(F_{\theta}(\mathbf{x}+\delta),y).
\end{equation}
By improving the loss function or the training procedures, standard PGD-AT has evolved into a series of more effective methods. E.g., TRADES~\cite{trades} introduces the KL divergence regularization to align the prediction probabilities of clean and perturbed examples. F{\small IRE}~\cite{Picot2023adversarial} designs a new Fisher-Rao regularization for the cross-entropy loss based on the geodesic distance between natural and perturbed features. MLCAT~\cite{yu2022Understanding} learns the large-loss data as usual and adopts additional measures to increase the loss of the small-loss data. AT+RiFT~\cite{Zhu2023Improving} exploits the redundant capacity for robustness by fine-tuning the adversarially trained model on its non-robust-critical module. RCAT~\cite{Angioni2025Robustness} proposes a robustness-congruent adversarial training method to obtain consistent estimators, etc.

\subsection{Uncertainty-Aware Defense Mechanisms}\label{subsec.uncertainty}
Essentially, uncertainty is closely related to noise~\cite{Nix1994Estimating}. When the noise comes from the natural world, it is reasonable to assume that it follows a Gaussian distribution $\mathcal{N}(\mu,\sigma^2)$, where the standard deviation $\sigma$ directly reflects its uncertainty. By introducing uncertainty to different modules of network, stochastic methods can enhance model robustness effectively and efficiently.

\subsubsection{Feature Uncertainty} A group of work focuses on feature uncertainty. E.g., in~\cite{random_xie} and~\cite{add_noise_to_input}, random transformation or Gaussian noise injection is applied to input data for robustness enhancement. In~\cite{rse_liu}, random noise $\varepsilon \thicksim \mathcal{N}(0,\Sigma)$ is injected into the hidden features of each convolutional layer, where $\Sigma$ is a diagonal hyperparameter matrix. In~\cite{l2p}, $\alpha_l\cdot\varepsilon_l$ is added to the output features of the $l$-th hidden layer, where $\varepsilon_l \thicksim \mathcal{N}(0,\mathbf{I})$ and $\mathbf{I}$ is an identity matrix. This work uses an alternate back-propagation training algorithm to update parameters, solving the problem that the noise variance converges to zero during training. 
Besides, in~\cite{wca}, $\varepsilon \thicksim \mathcal{N}(0,\Sigma)$ is added to the input features of the last layer, where $\Sigma$ is a fully parameterized matrix rather than a diagonal matrix. In this work, a weight-covariance alignment term is added to the cross-entropy loss, which encourages the alignment of the weight vector of the classification layer and the eigenvector of the covariance matrix of injected noise. It has been further improved in \cite{yang2023simple} and \cite{yang2023weight} by some regularization techniques.  Besides, smoothed classifiers are obtained for certified robustness by applying a random smoothing technique \cite{Cohen2019Certified}; the deep VIB technique~\cite{Alemi2017Deep} is adopted in~\cite{Wang2024Adversarially} to learn uncertain feature representations, etc. 

\subsubsection{Weight Uncertainty} Another group of work focuses on weight uncertainty. E.g., Adv-BNN~\cite{advbnn} assumes that network parameters are random variables rather than fixed values, and learns the parameter distributions as in BNNs~\cite{Carbone2020Robustness}. GradDiv~\cite{Lee2023GradDiv} further adds a gradient diversity regularization term to Adv-BNN against attacks using proxy gradients. In PNI method~\cite{pni_he}, $\alpha_l\cdot\varepsilon_l$ is added to the weights of the $l$-th hidden layer, where $\alpha_l$ is a noise intensity parameter, $\varepsilon_l\thicksim\mathcal{N}(0,\Sigma_l)$ is the random noise and $\Sigma_l$ is a diagonal matrix with the diagonal elements being the variance of the weights. In L2P method~\cite{l2p}, an alternating back-propagation process is proposed to consecutively train the noise parameters. In AWP method~\cite{wu2020Adversarial}, the weights are perturbed along the direction that the adversarial loss increases dramatically, etc.

\subsubsection{Out-of-Distribution Generalization} The above stochastic methods, whether focusing on feature uncertainty or weight uncertainty, try to introduce noises into individual samples or nodes. It is difficult for them to capture class-level or domain-level characteristics. They usually show effectiveness against traditional gradient-based attacks, but fail against more powerful ones such as CW and AA. In~\cite{dsu_li}, it is argued that statistical information carries domain characteristics of training data. Modeling the uncertainty of data statistics (such as mean and variance) can improve out-of-distribution generalization and thus alleviate the perturbations caused by potential domain shifts. However, this work only considers the distribution shifts caused by natural perturbations between training and testing data, while it pays no attention to the influence of adversarial perturbations. 

\section{Theoretical and Empirical Findings}\label{sec.theory}
In this section, we will present some theoretical and empirical findings on the distribution shifts of feature statistics caused by adversarial attacks.

\subsection{Basic Assumption}\label{subsec.assumption}

Following the notations in Section~\ref{subsec.attack}, we assume $F_{\theta}:\mathbf{X}\rightarrow\mathbf{Y}$ to be a classification model with certain generalization capability, i.e., it has already established prediction abilities on clean examples. Suppose the set of clean examples $\mathbf{x}_1,\mathbf{x}_2,\ldots,\mathbf{x}_N$ are perturbed into adversarial examples $\mathbf{x}_1^{\prime},\mathbf{x}_2^{\prime},\ldots,\mathbf{x}_N^{\prime}$ through $F_{\theta}$. We denote $\mu_i$ and $\mu_i^{\prime}$ as the feature means of $\mathbf{x}_i$ and $\mathbf{x}_i^{\prime}$, and denote $\sigma_i$ and $\sigma_i^{\prime}$ as the feature standard deviations of $\mathbf{x}_i$ and $\mathbf{x}_i^{\prime}$, i.e.,
\begin{equation}
\left\{
\begin{array}{lll}
\mu_i&=\frac{1}{d}\sum\nolimits_{j=1}^{d}x_{ij},\\
\mu_i^{\prime}&=\frac{1}{d}\sum\nolimits_{j=1}^{d}x_{ij}^{\prime},\\
\sigma_i&=\sqrt{\frac{1}{d}\sum\nolimits_{j=1}^{d}(x_{ij}-\mu_i)^2},\\
\sigma_i^{\prime}&=\sqrt{\frac{1}{d}\sum\nolimits_{j=1}^{d}(x_{ij}^{\prime}-\mu_i^{\prime})^2}.
\end{array}\right.
\end{equation}

\begin{lemma}\label{lamma1}
Let $\Omega$ be the convex hull of examples $\mathbf{x}_1,\mathbf{x}_2,\ldots,\mathbf{x}_N$, then $\Omega$ can be covered by the weighted average of $\mathbf{x}_1,\mathbf{x}_2,\ldots,\mathbf{x}_N$. That is, $\forall\mathbf{x}\in\Omega$, we can find a weight vector $\mathbf{w}=[w_1,w_2,\ldots,w_N],\ w_i\in[0,1],\ \sum_{i=1}^Nw_i=1$ such that $\mathbf{x}=\sum_{i=1}^N w_i\mathbf{x}_i$.
\end{lemma}
\begin{proof}
Please refer to Appendix~\ref{append.lemma1}.
\end{proof}

Lemma~\ref{lamma1} means, within the convex hull of a set of examples, each point can be represented as a linear combination of the set of examples. 

\begin{assumption}\label{assumption1}
Given a set of clean examples $\mathbf{x}_1,\mathbf{x}_2,\ldots,\mathbf{x}_N$ and a classification model $F_{\theta}$, the corresponding adversarial examples $\mathbf{x}_1^{\prime},\mathbf{x}_2^{\prime},\ldots,\mathbf{x}_N^{\prime}$ will be in the convex hull of the clean examples, i.e., $\mathbf{x}_1^{\prime},\mathbf{x}_2^{\prime},\ldots,\mathbf{x}_N^{\prime}\in\Omega(\mathbf{x}_1,\mathbf{x}_2,\ldots,\mathbf{x}_N)$.
\end{assumption}

Assumption~\ref{assumption1} means, according to Lemma~\ref{lamma1}, each adversarial example can be represented as a linear combination of the clean examples, i.e., $\mathbf{x}_i^{\prime}=\sum_{j=1}^N w_{ij}\mathbf{x}_j$, where $w_{ij}\in[0,1),\ \sum_{j=1}^Nw_{ij}=1,\ i,j=1,\ldots,N$. This assumption intuitively holds with a high probability. As shown in Fig.~\ref{fig.withinConvexHull}, the attack on an example is usually towards the decision boundary with another class. When the model $F_{\theta}$ has already established classification abilities on clean examples, such directions are generally clear. Furthermore, the perturbation $\delta$ will be constrained by the upper bound $\epsilon$ to ensure imperceptibility. In such a case, the adversarial example is unlikely to cross the range of the wrong-class and move outside the convex hull. Therefore, we can assume that it is still in the convex hull with a high probability.

\begin{figure}[t] 
\centering
\includegraphics[width=0.8\linewidth]{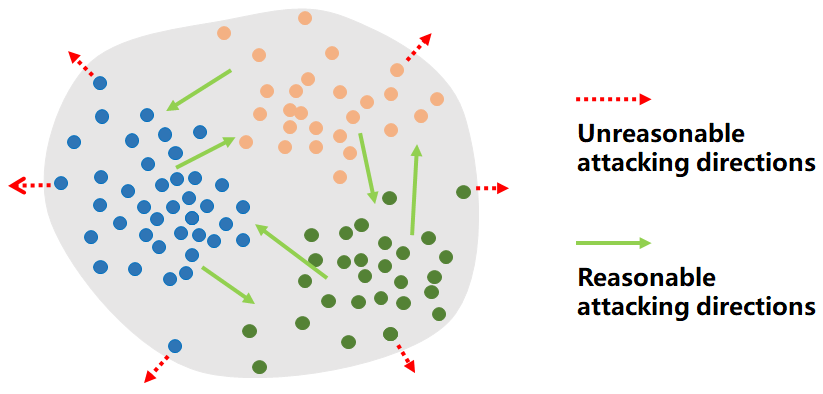}
\caption{An intuitive illustration for {\bf Assumption}~\ref{assumption1}.}
\label{fig.withinConvexHull}
\end{figure}

\subsection{Theoretical Discovery}\label{subsec.theories}

\begin{lemma}\label{lamma2}
Let $\mathbf{W}=[\mathbf{w}_{1};\mathbf{w}_{2};\ldots;\mathbf{w}_{N}]\in \mathcal{R}^{N\times N}$ be a weight matrix that satisfies
\begin{equation}\label{eq.weightCon}
\begin{array}{llll}
&\mathbf{w}_{i}=[w_{i,1},w_{i,2},\ldots,w_{i,N}],\ i=1,\ldots,N,\\
&w_{i,j}\in[0,1),\ i=1,\ldots,N,\ j=1,\ldots,N,\\
&\sum_{j=1}^{N}w_{i,j}=1,\ i=1,\ldots,N,\\
&\sum_{i=1}^{N}\sum_{j=1}^{N}w_{i,j}=N.\\
\end{array}
\end{equation} 
Denote $\mathbb{W}=\{\mathbf{W}\}$ as the set that contains all the weight matrices satisfying (\ref{eq.weightCon}), then $\mathbb{W}$ is closed with respect to multiplication.
\end{lemma}
\begin{proof}
Please refer to Appendix~\ref{append.lemma2}.
\end{proof}

Lemma~\ref{lamma2} states that after two attacking steps $\mathbf{A}$ and $\mathbf{B}$, the perturbed examples will remain in the convex hull of the clean examples.

\begin{lemma}\label{lamma3}
Let $\mathbf{W}^1,\mathbf{W}^2,\ldots,\mathbf{W}^t\in\mathbb{W}$ and $\mathbf{W}^{(t)}=\mathbf{W}^t\ldots\mathbf{W}^2\mathbf{W}^1$, when $t$ is large enough, $\mathbf{W}^{(t)}$ will converge to a fixed matrix, and there is $w^{(t)}_{1,j}=w^{(t)}_{2,j}=\ldots=w^{(t)}_{N,j},\ j=1,\ldots,N$, i.e., the values in each column of $\mathbf{W}^{(t)}$ become equal. 
\end{lemma}
\begin{proof}
Please refer to Appendix~\ref{append.lemma3}.
\end{proof}

Lemma~\ref{lamma3} describes the extreme case for adversarial attack: when the perturbation bound $\epsilon$ and the number of attacking steps $t$ are large enough, all the examples will finally shrink to the same point. In practice, this extreme case will never occur, since to ensure imperceptibility, the perturbation bound $\epsilon$ must be small, and the number of attack steps $t$ is finite. However, {\bf Lemma}~\ref{lamma3} provides some guidelines for analyzing the changing trend of feature statistics.

Suppose the set of clean examples $\mathbf{x}_1,\mathbf{x}_2,\ldots,\mathbf{x}_N$ are perturbed into adversarial examples $\mathbf{x}_1^{\prime},\mathbf{x}_2^{\prime},\ldots,\mathbf{x}_N^{\prime}$ through $F_{\theta}$. With a high probability, the following propositions hold.

\newcommand{\myfont}{\fontsize{6}{7}\selectfont}
\begin{figure*}[t]
\centering
\subfloat[{\myfont Mean Shift (MNIST-Train)}]{\includegraphics[width=0.16\linewidth]{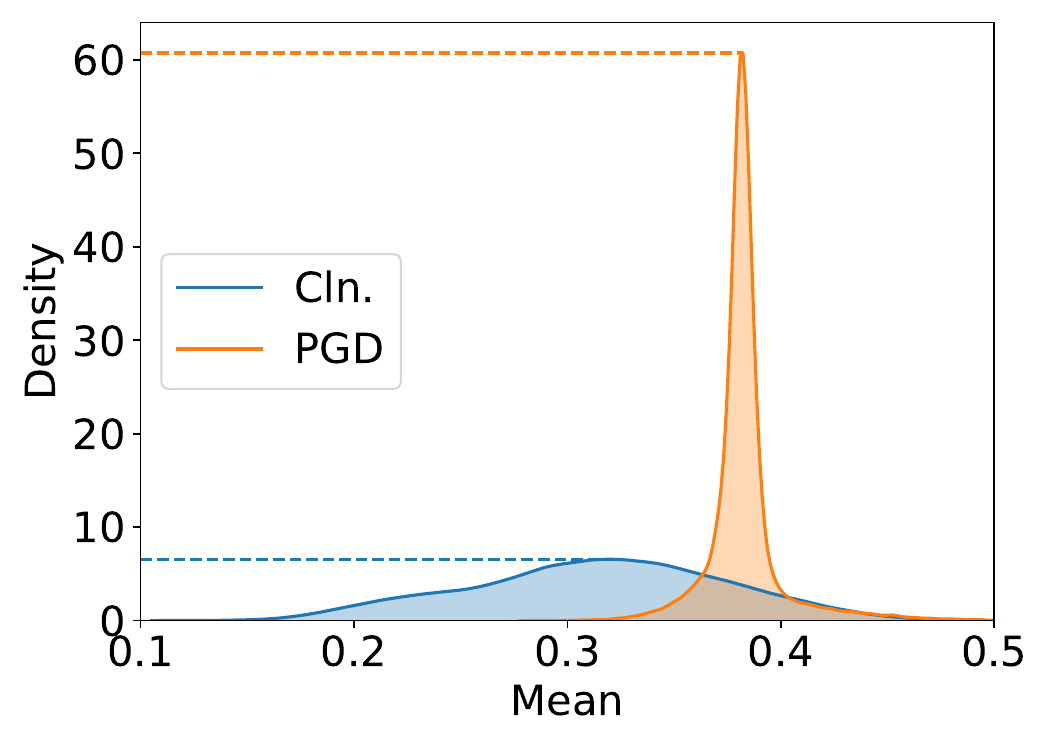}}\
\subfloat[{\myfont Mean Shift (MNIST-Test)}]{\includegraphics[width=0.16\linewidth]{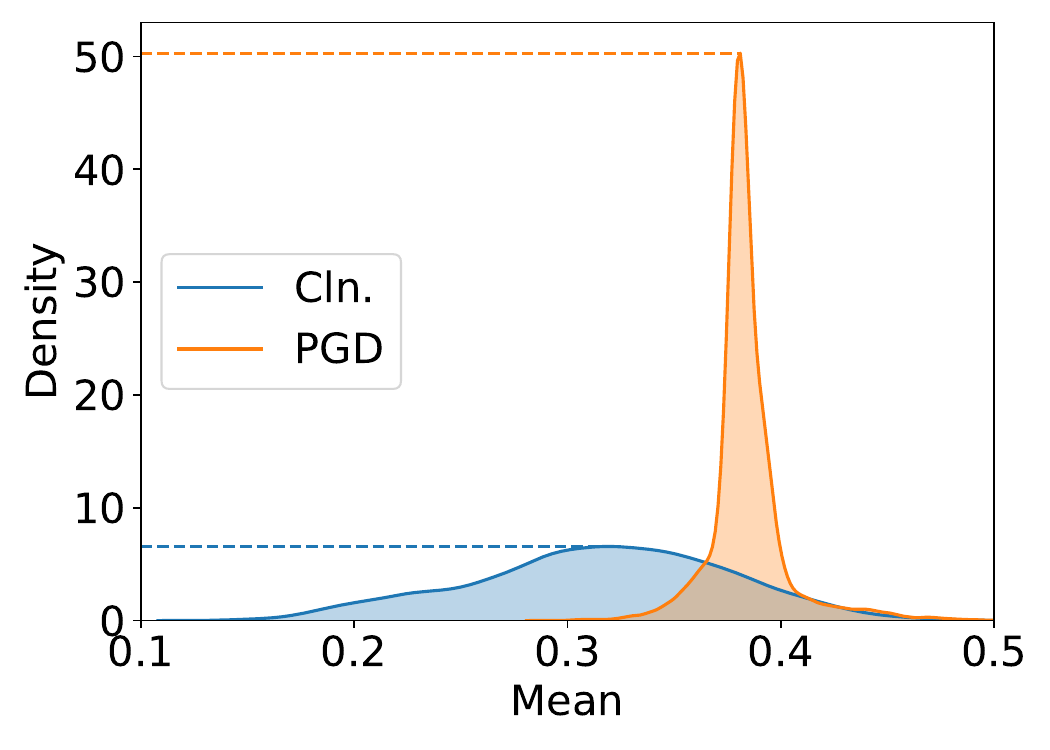}}\
\subfloat[{\tiny Mean Shift (CIFAR10-Train)}]{\includegraphics[width=0.16\linewidth]{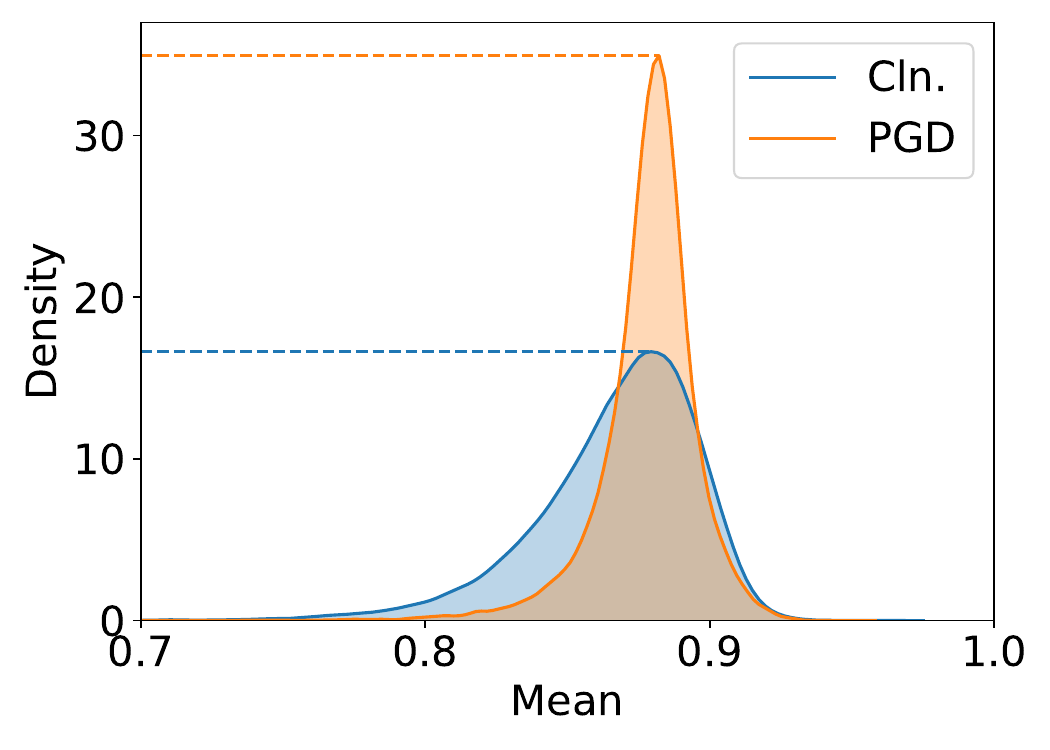}}\
\subfloat[{\tiny Mean Shift (CIFAR10-Test)}]{\includegraphics[width=0.16\linewidth]{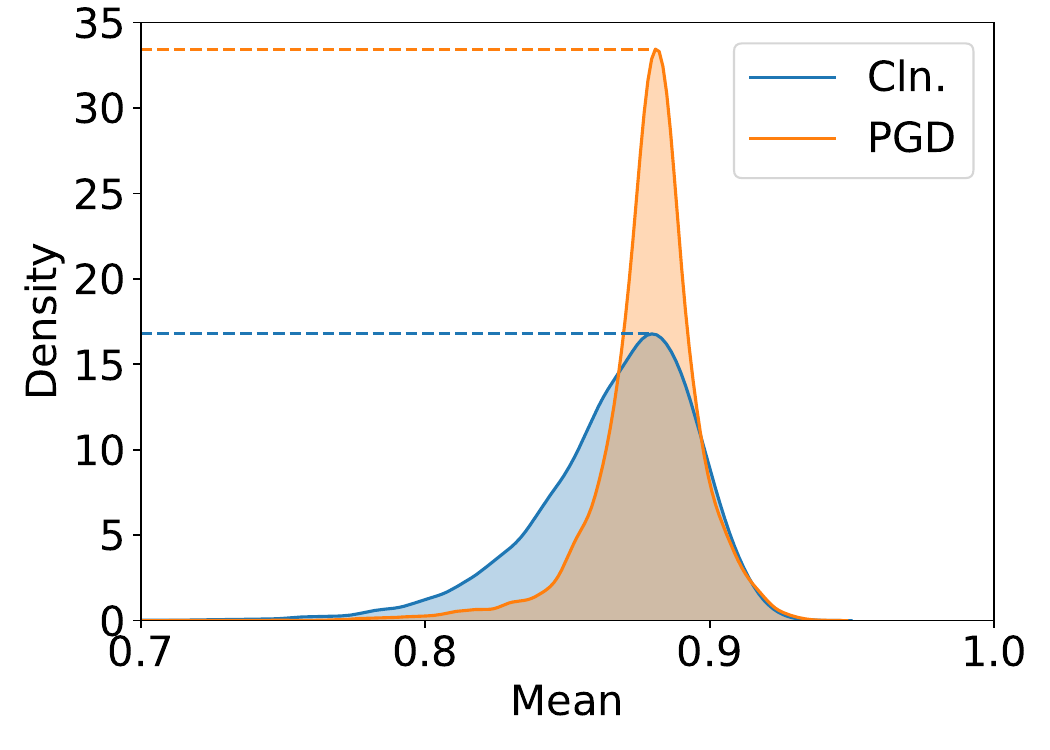}}\
\subfloat[{\myfont Mean Shift (SVHN-Train)}]{\includegraphics[width=0.16\linewidth]{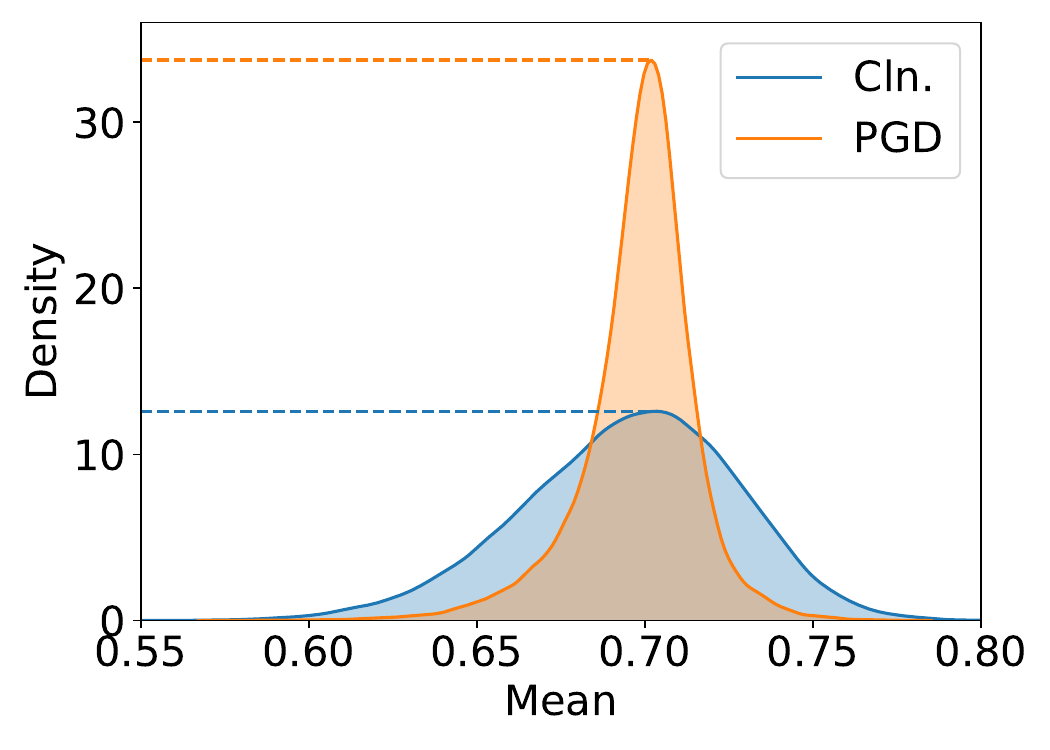}}\
\subfloat[{\myfont Mean Shift (SVHN-Test)}]{\includegraphics[width=0.16\linewidth]{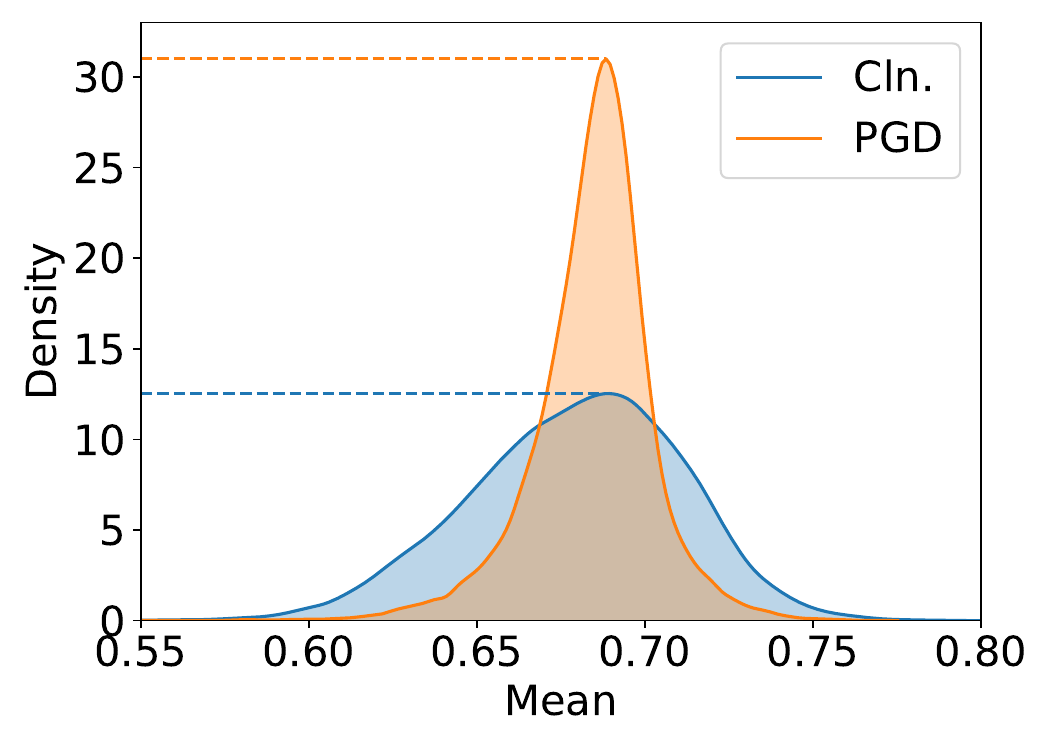}}\\
\subfloat[{\myfont Std Shift (MNIST-Train)}]{\includegraphics[width=0.16\linewidth]{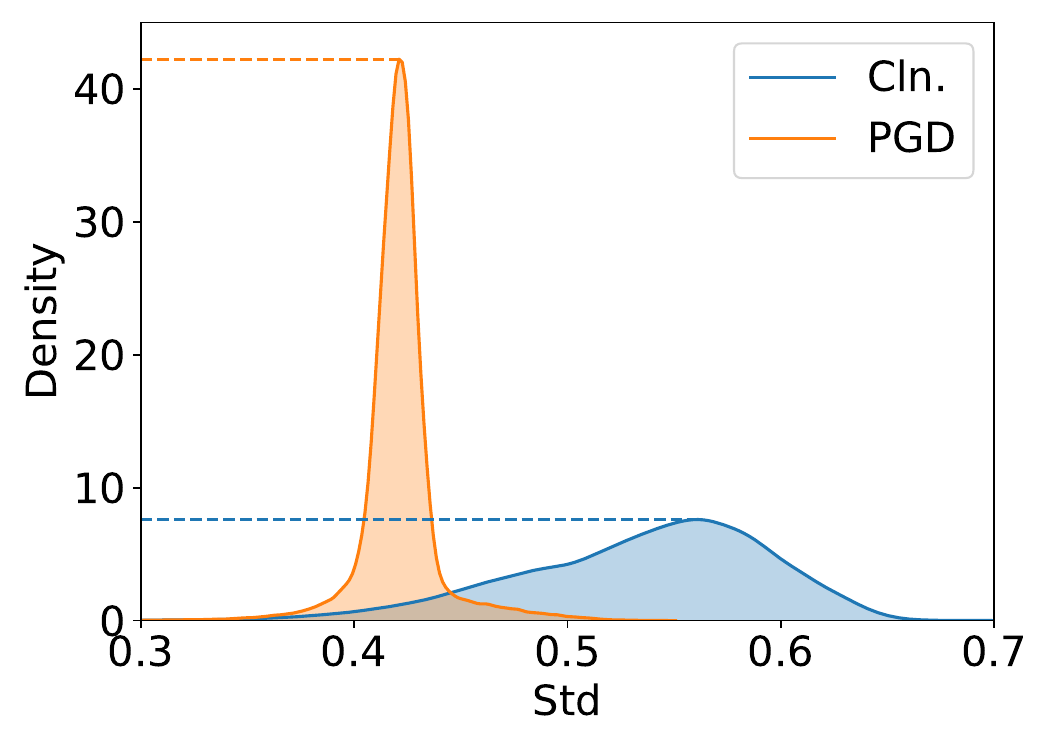}}\
\subfloat[{\myfont Std Shift (MNIST-Test)}]{\includegraphics[width=0.16\linewidth]{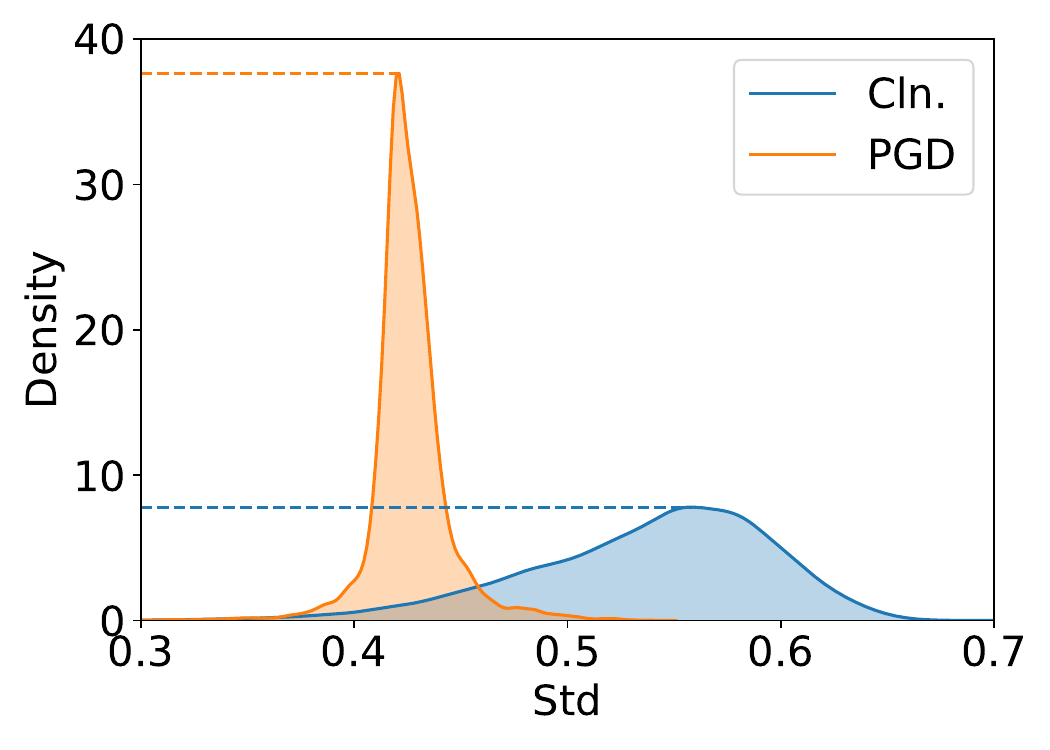}}\
\subfloat[{\myfont Std Shift (CIFAR10-Train)}]{\includegraphics[width=0.16\linewidth]{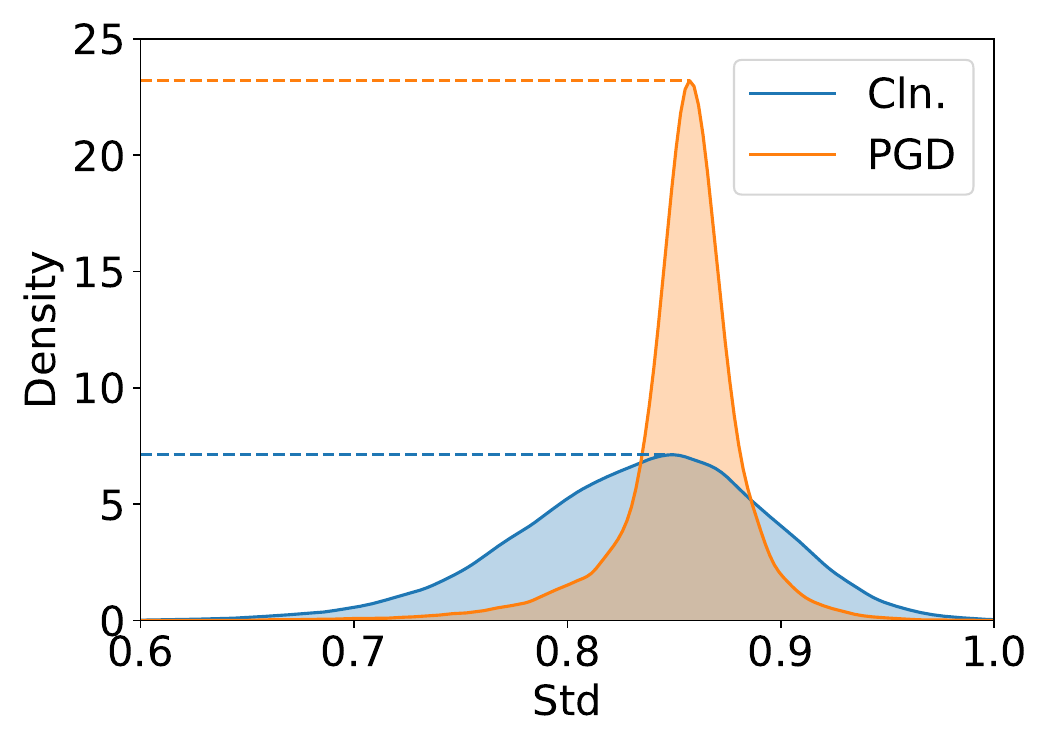}}\
\subfloat[{\myfont Std Shift (CIFAR10-Test)}]{\includegraphics[width=0.16\linewidth]{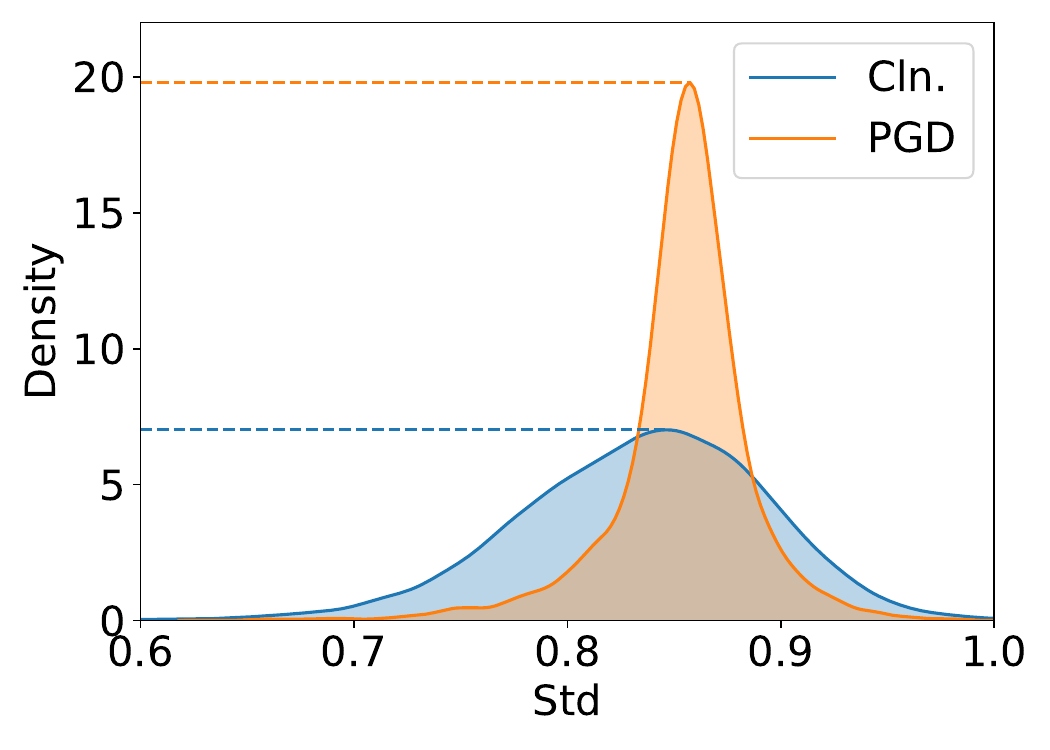}}\
\subfloat[{\myfont Std Shift (SVHN-Train)}]{\includegraphics[width=0.16\linewidth]{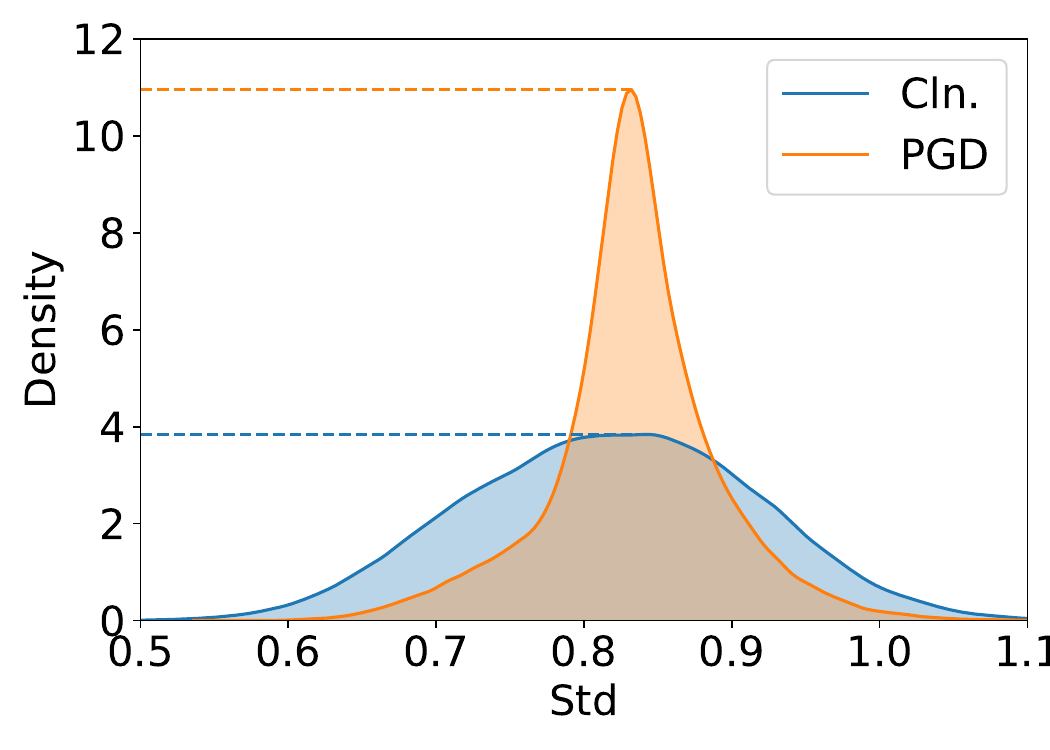}}\
\subfloat[{\myfont Std Shift (SVHN-Test)}]{\includegraphics[width=0.16\linewidth]{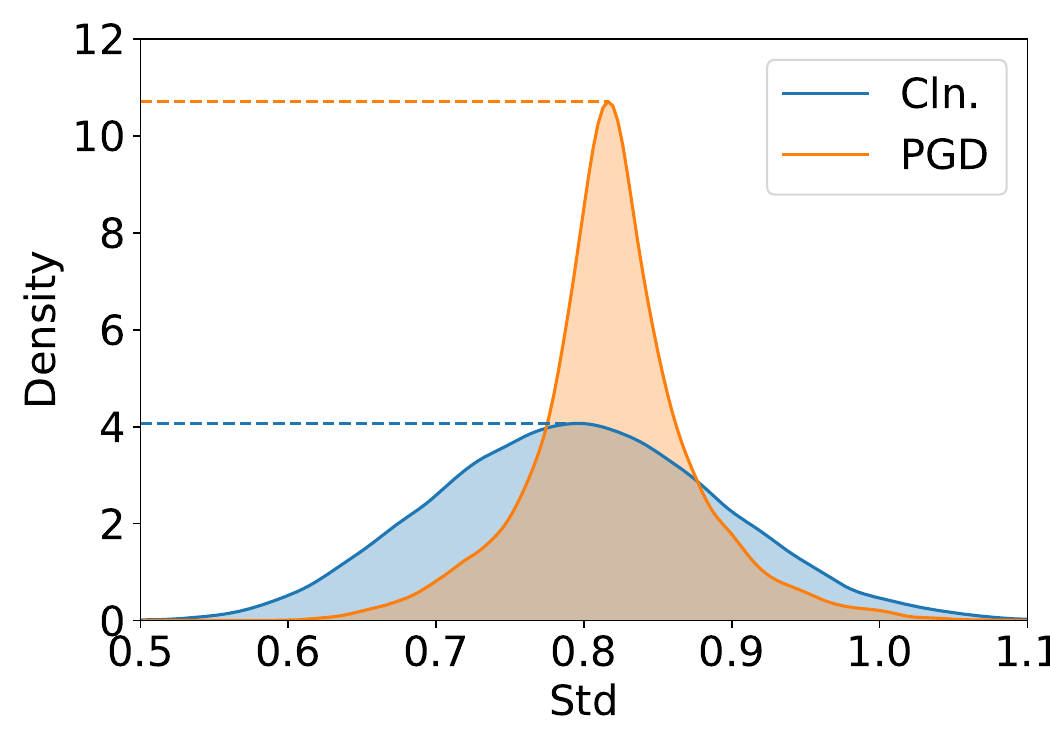}}
\caption{Distribution shifts of feature statistics (mean and Std) of perturbed examples in MNIST (average of 84 channels in latent space, LeNet-5 model), CIFAR10 and SVHN (average of 512 channels in latent space, ResNet-18 model). ``Cln." and ``PGD" represent clean examples and examples perturbed by PGD attack, respectively. ``Train" and ``Test" represent training data and testing data, respectively.}
\label{fig.featureDisShift}
\end{figure*}

\begin{proposition}\label{proposition1}
The variance of feature means becomes smaller after being perturbed, i.e., ${\rm var}(\mu^{\prime})\leq {\rm var}(\mu)$.
\end{proposition}
\begin{proof}
Please refer to Appendix~\ref{append.prop1}.
\end{proof}

\begin{proposition}\label{proposition3}
The variance of feature variances also becomes smaller after being perturbed, i.e., ${\rm var}(\sigma^{\prime 2})\leq {\rm var}(\sigma^2)$.
\end{proposition}
\begin{proof}
Please refer to Appendix~\ref{append.prop3}.
\end{proof}

Propositions~\ref{proposition1}$\sim$\ref{proposition3} describe a clear changing trend for the feature statistics of the perturbed examples, i.e., after being perturbed, the differences of the feature means and feature variances of the examples will all decrease, making the examples difficult to distinguish. Obviously, if these propositions hold, we can recover some domain characteristics of data by enlarging the variances of feature statistics of perturbed examples, so as to alleviate perturbations and enhance robustness.

\subsection{Empirical Verification}\label{subsec.empirical}

Fig.~\ref{fig.featureDisShift} demonstrates the distributions of latent feature mean and standard deviation (Std) of both clean and perturbed examples in datasets MNIST, CIFAR10, and SVHN. The models are naturally trained with LeNet-5 architecture for MNIST, and ResNet-18 architecture for CIFAR10 and SVHN. The examples are perturbed by $l_{\infty}$-norm PGD attack with step size $2/255$ (40 steps with $\epsilon=0.3$ for MNIST, and 10 steps with $\epsilon=8/255$ for CIFAR10 and SVHN). The mean and Std are evaluated in latent space, where the distributions are obtained by kernel density estimation and the average of multiple channels are investigated. As a remark, both feature mean and Std approximately obey Gaussian distributions, no matter before or after being perturbed.

\begin{assumption}\label{assumption2}
When the number of examples is large enough, both $\mu_i$, $\mu_i^{\prime}$, $\sigma_i$ and $\sigma_i^{\prime}$ approximately follow Gaussian distributions.
\end{assumption}

Assumption~\ref{assumption2} is supported empirically by our investigations as in Fig.~\ref{fig.featureDisShift}, and is supported theoretically by the {\it Central Limit Theorem}.

Figs.~\ref{fig.featureDisShift}(a)$\sim$(f) / Figs.~\ref{fig.featureDisShift}(g)$\sim$(l) demonstrate the distributions of feature mean / feature Std of perturbed training and testing examples in the datasets through the naturally trained models. It can be observed that the variances of both feature mean and Std become smaller, i.e., the shapes of the distributions are getting more concentrate, providing empirical support for Proposition~\ref{proposition1} and Proposition~\ref{proposition3}.

\section{Proposed Method}\label{sec.FSU}

Inspired by the theoretical and empirical findings in Section~\ref{sec.theory}, this section will describe the FSU module in detail, and propose the robustness enhancement method.

\subsection{Basic Idea}

Given an image classification task, where the images are fed into a multi-channel network batch-by-batch. Considering a batch of examples, we investigate the feature maps in a multi-channel hidden layer, and denote $\mathbf{z}_{b,c}$ as the feature map of the $b$-th example in the $c$-th channel, i.e.,
\begin{equation}
\label{feature}
    \mathbf{z}_{b,c} = 
        \begin{bmatrix}
            z_{b,c}^{1,1} & \cdots & z_{b,c}^{1,W} \\
            \vdots & \ddots & \vdots \\
            z_{b,c}^{H,1} & \cdots & z_{b,c}^{H,W}
        \end{bmatrix}_{H \times W},
\end{equation}
where $b=1,\ldots,B$ ($B$ is the batch size), $c=1,\ldots,C$ ($C$ is the number of channels), $H$ and $W$ are the height and width of the feature map, respectively. Easily, we can compute
\begin{equation}
\left\{
\begin{array}{lll}
    \mu_{b,c}=\mathbb{E}(z_{b,c}^{i,j})=\frac{1}{HW}\sum\nolimits_{i=1}^{H}\sum\nolimits_{j=1}^{W}z_{b,c}^{i,j}\\
    \sigma_{b,c}=\sqrt{\frac{1}{HW}\sum\nolimits_{i=1}^{H}\sum\nolimits_{j=1}^{W}(z_{b,c}^{i,j}-\mathbb{E}[z_{b,c}^{i,j}])^2}
\end{array}\right..
\label{eq.meanStd}
\end{equation}
Obviously, $\mu_{b,c}$ and $\sigma_{b,c}$ are the two most important feature statistics of $\mathbf{z}_{b,c}$, which respectively reflect the average and dispersion of features in $\mathbf{z}_{b,c}$. According to Assumption~\ref{assumption2}, when the batch size $B$ is large enough, it is reasonable to assume that both $\mu_{b,c}$ and $\sigma_{b,c}$ approximately follow Gaussian distributions.

Given a specific clean feature map $\mathbf{z}_{b,c}$, the numerical values of $\mu_{b,c}$ and $\sigma_{b,c}$ can be easily computed according to Eq.~(\ref{eq.meanStd}). When all the feature maps $\mathbf{z}_{b,c}$ are perturbed by the attacker, both $\mu_{b,c}$ and $\sigma_{b,c}$ change correspondingly, leading to potential distribution shifts as in Fig.~\ref{fig.featureDisShift}. Since the original clean distributions keep some useful domain characteristics for classification, if the shifted distributions can be calibrated, it is possible to reconstruct the perturbed feature maps, thereby mitigating perturbations and enhancing robustness.

Our basic idea is to insert an FSU module after a certain hidden layer of the backbone architecture, which resamples the channel-wise feature means and standard deviations of each example according to the statistical information in the batch. After the new feature means and standard deviations are obtained, the feature maps are reconstructed and inputted into the remaining layers of the network.

\subsection{The FSU Module}

A schematic view of the FSU module is provided in Fig.~\ref{fig.overview}. In specific, given a (possibly perturbed) feature map $\mathbf{z}_{b,c}$, its mean $\mu_{b,c}$ and standard deviation $\sigma_{b,c}$ are supposed to follow Gaussian distributions, i.e.,

\begin{equation}\label{eq.meanAndStd}
\left\{
\begin{array}{lll}
    \mu_{b,c}\sim\mathcal{N}(\Phi_{\mu_c},\Sigma_{\mu_c}^2)\\
    \sigma_{b,c}\sim\mathcal{N}(\Phi_{\sigma_c},\Sigma_{\sigma_c}^2)\\
\end{array}\right.,\ c=1,\ldots,C.
\end{equation}
The channel-wise means of $\mu_{b,c}$ and $\sigma_{b,c}$ are estimated as
\begin{equation}\label{eq.statistics_mean}
\left\{
\begin{array}{lll}
    \Phi_{\mu_c}=\mathbb{E}(\mu_{b,c})=\frac{1}{B}\sum\nolimits_{b=1}^{B}\mu_{b,c}\\
    \Phi_{\sigma_c}=\mathbb{E}(\sigma_{b,c})=\frac{1}{B}\sum\nolimits_{b=1}^{B}\sigma_{b,c}\\
\end{array}\right.,\ c=1,\ldots,C,
\end{equation}
and the channel-wise standard deviations of $\mu_{b,c}$ and $\sigma_{b,c}$ are estimated as
\begin{equation}\label{eq.statistics}
\left\{
\begin{array}{lll}
    \Sigma_{\mu_c}=\sqrt{\frac{1}{B}\sum\nolimits_{b=1}^{B}(\mu_{b,c}-\Phi_{\mu_c})^2}\\
    \Sigma_{\sigma_c}=\sqrt{\frac{1}{B}\sum\nolimits_{b=1}^{B}(\sigma_{b,c}-\Phi_{\sigma_c})^2}\\
\end{array}\right.,\ c=1,\ldots,C.
\end{equation}
Based on the theoretical findings in Propositions~\ref{proposition1}$\sim$\ref{proposition3}, compared to the original clean feature maps, both ${\rm var}(\mu_{b,c})$ and ${\rm var}(\sigma_{b,c})$ become smaller after being perturbed. Thus, to recover some domain characteristics (i.e., distinctions of the examples) for classification, we need to enlarge ${\rm var}(\mu_{b,c})$ and ${\rm var}(\sigma_{b,c})$ for the perturbed examples. 

Let $\hat{\mu}_{b,c}$ and $\hat{\sigma}_{b,c}$ be the new $\mu_{b,c}$ and new $\sigma_{b,c}$. To achieve the above goal, the numerical values of $\hat{\mu}_{b,c}$ and $\hat{\sigma}_{b,c}$ can be obtained by biased resampling from $\mathcal{N}(\Phi_{\mu_c},\Sigma_{\mu_c}^2)$ and $\mathcal{N}(\Phi_{\sigma_c},\Sigma_{\sigma_c}^2)$, respectively. Since the sampling process cannot be incorporated into gradient-based back-prorogation training, re-parameterization trick is applied here. We design $\varepsilon_{b,c}^{\mu}$ and $\varepsilon_{b,c}^{\sigma}$ as the {\it\textbf{uncertain correction terms}} for $\mu_{b,c}$ and $\sigma_{b,c}$, which are randomly and independently drawn from standard Gaussian distribution $\mathcal{N}(0, 1)$, i.e.,
\begin{equation}\label{eq.noise}
\left\{
\begin{array}{lll}
    \varepsilon^{\mu}_{b,c} \sim \mathcal{N}(0, 1)\\
    \varepsilon^{\sigma}_{b,c} \sim \mathcal{N}(0, 1)
\end{array}\right.,\ b=1,\ldots,B,\ c=1,\ldots,C.
\end{equation}
Then, the new feature mean $\hat{\mu}_{b,c}$ and new feature standard deviation $\hat{\sigma}_{b,c}$ are derived as
\begin{equation}\label{eq.reConstuctMS}
\left\{
\begin{array}{lll}
    \hat{\mu}_{b,c}=\mu_{b,c}+\alpha \cdot \varepsilon_{b,c}^{\mu} \cdot \Sigma_{\mu_c}\\
    \hat{\sigma}_{b,c}=\sigma_{b,c}+\beta \cdot \varepsilon_{b,c}^{\sigma} \cdot \Sigma_{\sigma_c}
\end{array}\right.,
\end{equation}
where $\alpha$ and $\beta$ are the noise intensities for the mean and standard deviation, respectively. Having the resampled $\hat{\mu}_{b,c}$ and $\hat{\sigma}_{b,c}$ , $\mathbf{z}_{b,c}$ is finally reconstructed as
\begin{equation}\label{eq.reConstuctMap}
    \hat{\mathbf{z}}_{b,c}=\hat{\sigma}_{b,c}\times\frac{\mathbf{z}_{b,c}-\mu_{b,c}}{\sigma_{b,c}}+\hat{\mu}_{b,c},
\end{equation}
where the operations $-$, $\times$, and $+$ indicate element-wise subtraction, multiplication, and addition, respectively, $b=1,\ldots,B$ and $c=1,\ldots,C$.

Note that the reconstruction depends only on the statistical information of the current batch, thus the FSU module can be independently employed in {\it\textbf{training}}, {\it\textbf{attacking}}, and {\it\textbf{predicting}} stages. In training, it encourages the network to learn some uncertainty via Gaussian noise; in attacking and predicting, it reconstructs features in each epoch to recover some domain characteristics for classification.

\begin{figure}[t] 
\centering
\includegraphics[width=1.0\linewidth]{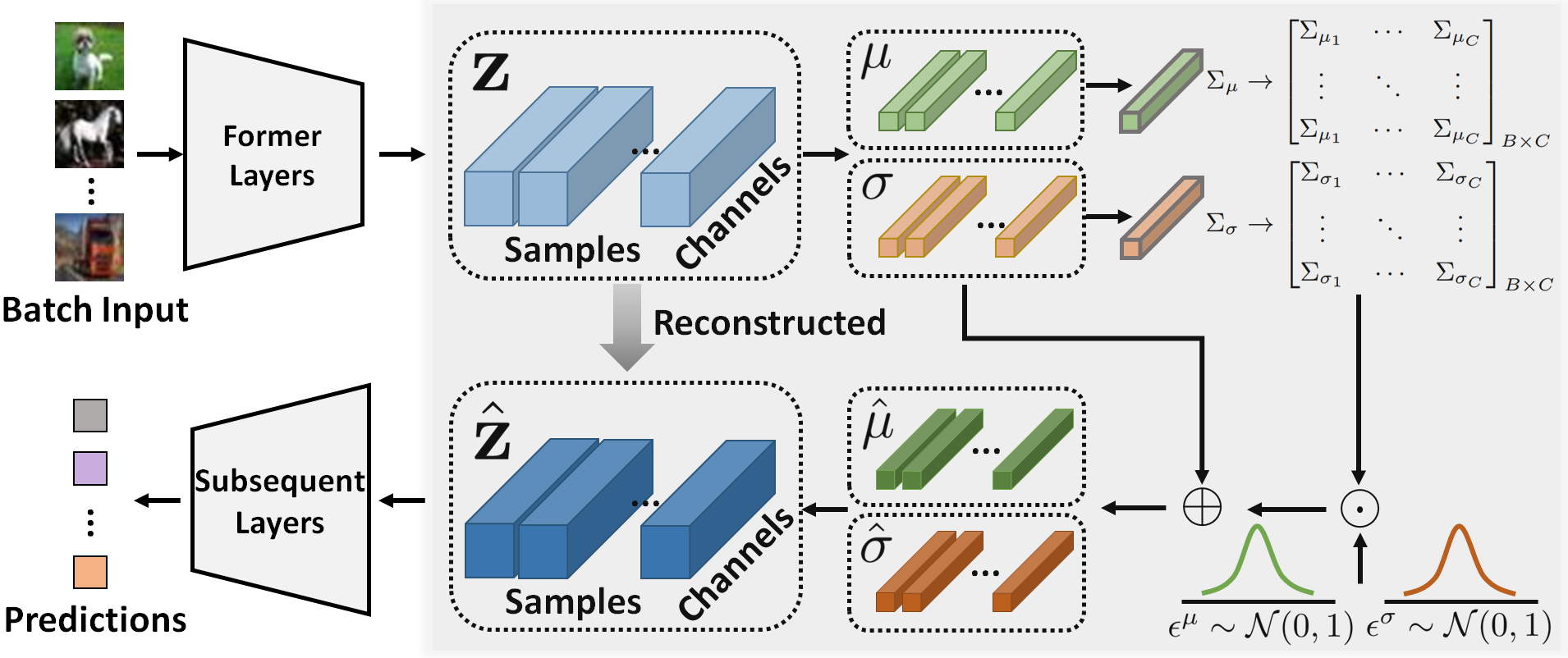}
\caption{A schematic view of the proposed FSU module.}
\label{fig.overview}
\end{figure}

\subsection{Intuitive Explanations}

\begin{theorem}\label{thm1}
Let $\hat{\mu}_{b,c}$ and $\hat{\sigma}_{b,c}$ be the reconstructed $\mu_{b,c}$ and $\sigma_{b,c}$ based on Eq.~(\ref{eq.reConstuctMS}), then there are ${\rm var}[\hat{\mu}_{b,c}]\geq{\rm var}[{\mu}_{b,c}]$ and ${\rm var}[\hat{\sigma}_{b,c}]\geq{\rm var}[{\sigma}_{b,c}]$.
\end{theorem}
\begin{proof}
Please refer to Appendix~\ref{append.thm1}.
\end{proof}

\paragraph{Perspective 1} As presented in {\bf Theorem}~\ref{thm1}, the reconstruction can enlarge the variances of both the channel-wise mean and standard deviation of feature maps, thereby restoring some distinctions among examples and recovering some discriminative domain characteristics for classification. 

\paragraph{Perspective 2} We derive an approximate distribution of the reconstructed features in Eq.~(\ref{eq.reConstuctMap}). Unfortunately, when ${\mu}_{b,c}$, ${\sigma}_{b,c}$, $\hat{\mu}_{b,c}$, $\hat{\sigma}_{b,c}$ are all treated as random variables, the inference on this distribution will be too complex, since it includes the division of two Gaussian distributions, as well as the multiplication of a Gaussian distribution and a ratio distribution. Given a specific example, we can treat ${\mu}_{b,c}$ and ${\sigma}_{b,c}$ as fixed values, and according to the proof of {\bf Theorem}~\ref{thm1} in Appendix~\ref{append.thm1}, we approximately let $\hat{\mu}_{b,c}\sim\mathcal{N}(\mu_{b,c},(1+\alpha^2)\Sigma_{\mu_c}^2)$ and $\hat{\sigma}_{b,c}\sim\mathcal{N}(\sigma_{b,c},(1+\beta^2)\Sigma_{\sigma_c}^2)$. Thus, according to Eq.~(\ref{eq.reConstuctMap}), each feature $\hat{z}_{b,c}$ in feature map $\hat{\mathbf{z}}_{b,c}$ also follows a Gaussian distribution, whose expectation is estimated as
\begin{equation}
\begin{aligned}
        \mathbb{E}[\hat{z}_{b,c}] & = \mathbb{E}[\hat{\sigma}_{b,c}]\times\frac{z_{b,c}-\mu_{b,c}}{\sigma_{b,c}}+\mathbb{E}[\hat{\mu}_{b,c}]\\
        & = \sigma_{b,c}\times\frac{z_{b,c}-\mu_{b,c}}{\sigma_{b,c}}+\mu_{b,c}\\
        & = z_{b,c}
\end{aligned},
\end{equation}
and variance is estimated as 
\begin{equation}
\label{variance_of_rx}
\begin{aligned}
{\rm var}[\hat{z}_{b,c}] &= {\rm var}[\hat{\sigma}_{b,c}]  \left(\frac{z_{b,c}-\mu_{b,c}}{\sigma_{b,c}}\right)^2 + {\rm var}[\hat{\mu}_{b,c}]\\
        &= (1+\alpha^2)\Sigma_{\sigma_c}^2  \left(\frac{z_{b,c}-\mu_{b,c}}{\sigma_{b,c}}\right)^2 + (1+\beta^2)\Sigma_{\mu_c}^2
\end{aligned}.
\end{equation}
Finally, we approximately have 
\begin{equation}
\label{distribution_of_rx}
\small
\begin{aligned}
&\hat{z}_{b,c}\thicksim \\
&\mathcal{N}\left(z_{b,c},\ \  (1+\alpha^2)\Sigma_{\sigma_c}^2  \left(\frac{z_{b,c}-\mu_{b,c}}{\sigma_{b,c}}\right)^2 + (1+\beta^2)\Sigma_{\mu_c}^2\right)
\end{aligned}.
\end{equation}

\begin{figure}[t]
    \centering
    \vspace{-0.3cm}
    \subfloat{\includegraphics[width=1.3cm]{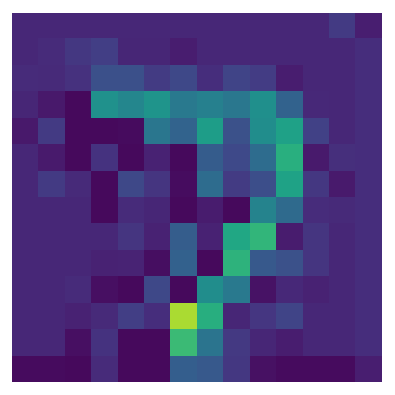}}\quad
    \subfloat{\includegraphics[width=1.3cm]{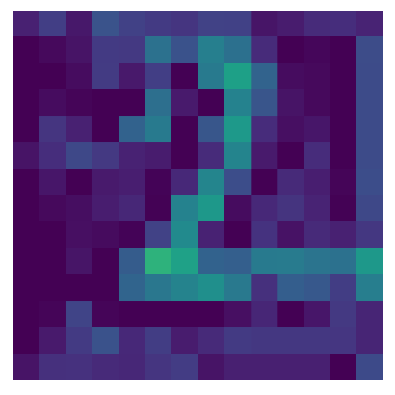}}\quad
    \subfloat{\includegraphics[width=1.3cm]{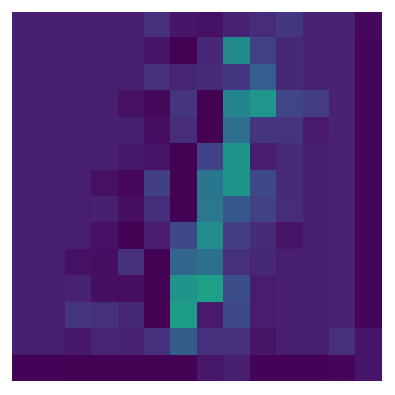}}\quad
    \subfloat{\includegraphics[width=1.3cm]{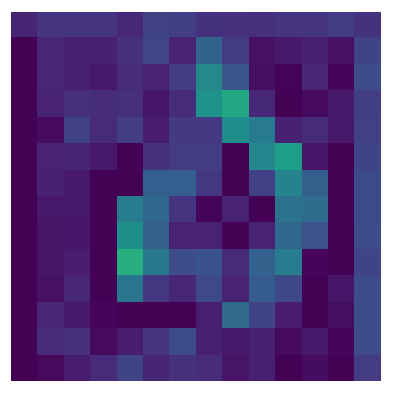}}\quad
    \subfloat{\includegraphics[width=1.3cm]{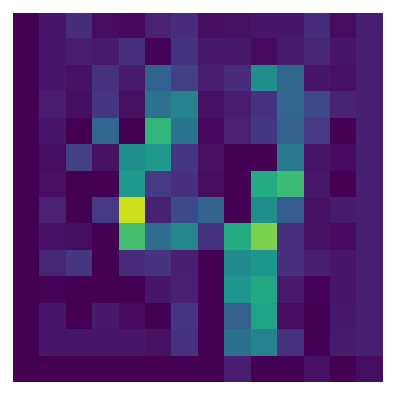}}\quad
    \subfloat{\includegraphics[width=0.45cm]{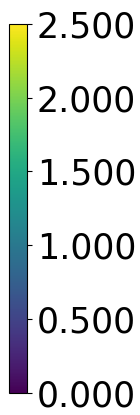}}\\
    \subfloat{\includegraphics[width=1.3cm]{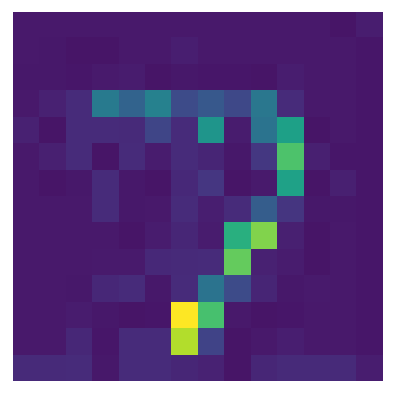}}\quad
    \subfloat{\includegraphics[width=1.3cm]{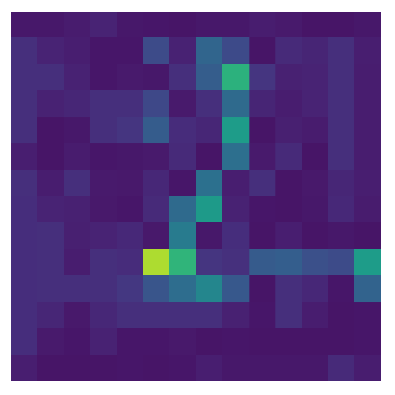}}\quad
    \subfloat{\includegraphics[width=1.3cm]{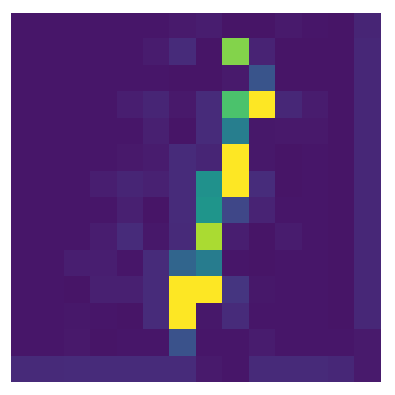}}\quad
    \subfloat{\includegraphics[width=1.3cm]{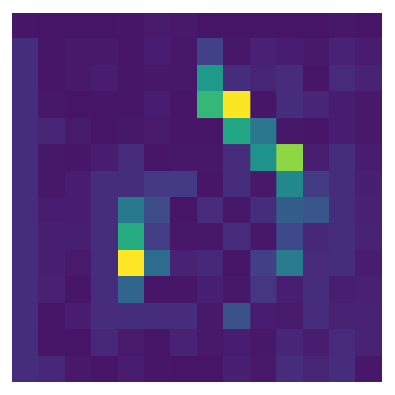}}\quad
    \subfloat{\includegraphics[width=1.3cm]{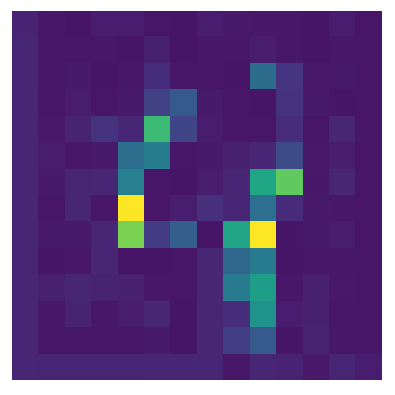}}\quad
    \subfloat{\includegraphics[width=0.45cm]{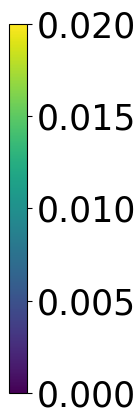}}
    \vspace{-0.1cm}
    \caption{Feature maps (first line) and variance graphs calculated by Eq.~(\ref{variance_of_rx}) (second line) of FGSM-perturbed examples in MNIST.}
    \vspace{-0.1cm}
    \label{variance_map}
\end{figure}

Based on this approximate inference, $\hat{\mathbf{z}}_{b,c}$ is essentially a resampled feature map from a multivariate Gaussian distribution, whose mean is $\mathbf{z}_{b,c}$ and variance is determined by normalized $\mathbf{z}_{b,c}$ and batch-level statistical values $(\Sigma_{\mu_c}^2,\Sigma_{\sigma_c}^2)$. At the pixel level, a larger deviation of the feature value from the mean corresponds to a higher variance, resulting in greater amplitudes of resampling. Fig.~\ref{variance_map} shows several FGSM-perturbed feature maps in dataset MNIST through the naturally trained LeNet-5 model, i.e., ${\mathbb{E}}[\hat{\mathbf{z}}_{b,c}]$, and the variance graphs calculated by Eq.~(\ref{variance_of_rx}), i.e., ${\rm var}[\hat{\mathbf{z}}_{b,c}]$. Obviously, points with large variances correspond to significant features with important semantic information. Adversarial attacks make such significant feature points be less discriminative from the background, thus misleading the established model into making false predictions. The FSU module reconstructs such points with larger variances. This process, to some extent, mitigates the influence of adversarial perturbations.

\subsection{FSU for Testing or Adversarial Training}\label{subsec.twoSolutions}

Many works have demonstrated that models learned by natural training may not be robust enough to defend against powerful attacks. Especially, introducing randomness to the network may endow the model with robustness against some mild attacks such as FGSM, BIM, MIM, etc. However, when confronted with strong attacks such as large-step PGD, CW, AA, etc., the robustness may drop seriously. In the proposed FSU, the feature statistics $(\mu,\sigma)$ can be seen as low-dimensional representations of examples. Although they carry some domain characteristics of data, the model is still difficult to defend against powerful attacks if adversarial information is not introduced during training. To deal with this problem, the following two solutions can be considered: 
\begin{enumerate}
\item Selecting a robust model established by adversarial training and incorporating the FSU module into the testing phase (including attacking and predicting).
\item Incorporating the FSU module into an adversarial training or model fine-tuning process by introducing some directional information at the level of feature statistics.
\end{enumerate}
Note that the first solution requires the examples to be tested in batch to obtain the distributions of feature statistics, which is unsuitable in many practical attacking scenarios. Therefore, we mainly focus on the second solution. As introduced in Section~\ref{subsec.robustness}, classic adversarial training uses PGD attack to generate perturbations, and greater perturbations can better expose the vulnerabilities of the model. However, without directional information, the FSU module tends to restore some features of the perturbed examples, weakening the perturbations and negatively affecting adversarial training. Thus, we introduce some directional information at the level of feature statistics, such that the vulnerabilities of the model can be learned better. Specifically, for the uncertain correction terms in Eq.~(\ref{eq.noise}), i.e., $\varepsilon^{\mu}$ and $\varepsilon^{\sigma}$, we determine their sampling intervals during each attacking epoch channel-wisely as

\begin{equation}\label{eq.muInterval}
\small
{\rm \varepsilon^{\mu}\in}\left\{
\begin{array}{ll}
(-\infty,0), & {\rm if}\ \mu'-\mu< 0,\\
(0,+\infty), & {\rm if}\ \mu'-\mu> 0,\\
(-\infty,+\infty), & {\rm otherwise},\\
\end{array}\right.
\end{equation}
and
\begin{equation}\label{eq.sigmaInterval}
\small
{\rm \varepsilon^{\sigma}\in}\left\{
\begin{array}{ll}
(-\infty,0), & {\rm if}\ \sigma'-\sigma< 0,\\
(0,+\infty), & {\rm if}\ \sigma'-\sigma> 0,\\
(-\infty,+\infty), & {\rm otherwise},\\
\end{array}\right.
\end{equation}
such that a reduced offset can become smaller and an increased offset can become larger, amplifying perturbations at the level of feature statistics and enhancing adversarial training. 

Finally, the details of the proposed method are described in Algorithm~\ref{alg.FSU}. As a remark, this algorithm can be used to train / fine-tune any preliminary / well-established models. 

\begin{algorithm}[t]
\small
\renewcommand\arraystretch{0.5}
\caption{Training with FSU}\label{alg.FSU}
\KwIn{Initial model $F_{\theta}$, training set $\mathbb{D}$, learning rate $\eta$, attacking method, noise intensity factors $\alpha\geq0$ and $\beta\geq0$.}
\KwOut{Updated model $F_{\theta}$.}
\For{$t=1:{\rm num\_epoch}$}{
\For{$e=1:{\rm num\_batch}$}{
\For{${\rm each\ clean\ example}~\mathbf{x}~{\rm in\ the\ batch}$}{
Apply attacking method to $\mathbf{x}$ and get perturbed $\mathbf{x}'$;
}
Forward propagate the batch of $\mathbf{x}$ and $\mathbf{x}'$;\\
\If{${\rm reconstruction\ layer}$}{
Extract hidden feature maps of $\mathbf{x}$ and $\mathbf{x}'$, i.e, $\mathbf{z}$ and $\mathbf{z}'$;\\
Calculate channel-wise $\mu$ and $\sigma$ for each $\mathbf{z}$;\\
Calculate channel-wise $\mu'$ and $\sigma'$ for each $\mathbf{z}'$;\\
Calculate ($\Sigma_{\mu'},\Sigma_{\sigma'}$) for the batch of $\mathbf{z}'$ via Eq.~(\ref{eq.statistics});\\
Determine the sampling interval channel-wisely for $\varepsilon^{\mu}$ and $\varepsilon^{\sigma}$ based on (\ref{eq.muInterval}) and (\ref{eq.sigmaInterval});\\
Resample $\hat{\mu}'$ and $\hat{\sigma}'$ for each $\mathbf{z}'$ via Eq.~(\ref{eq.reConstuctMS});\\
Reconstruct each $\hat{\mathbf{z}}'$ via Eq.~(\ref{eq.reConstuctMap});\\
Propagate the feature maps to the subsequent layer;\\
}
Back-propagate to update $\theta$ via gradient descent based on the batch of $\mathbf{x}$, $\mathbf{x}'$ and $\hat{\mathbf{x}}'$;\\
}
}
\textbf{return} Updated model $F_{\theta}$.\\
\end{algorithm}

\section{Experiments}\label{sec.exp}
In this section, we will conduct extensive experiments to validate the applicability and effectiveness of the FSU module in robustness enhancement.

\subsection{Experimental Setup}\label{subsec.setup}

\paragraph{Datasets} We evaluate model robustness on widely used benchmark datasets, including MNIST~\cite{lecun1998gradient}, CIFAR10~\cite{krizhevsky2009learning}, SVHN~\cite{goodfellowmulti} and CIFAR100~\cite{krizhevsky2009learning}. The details of these datasets are provided in Table~\ref{tab.datasets}.

\paragraph{Evaluation Protocols} Various benchmark attacks are performed, including FGSM~\cite{fgsm}, BIM~\cite{bim}, MIM~\cite{mim}, PGD~\cite{pgd}, CW~\cite{cw}, FAB~\cite{Croce2020Minimally} and AA~\cite{Croce2020Reliable}. Among them, FGSM, BIM, MIM and PGD are typical gradient-based methods; CW is an optimization-based method; FAB is an adaptive boundary arrack and AA is an ensemble method that integrates both white-box and black-box attacks. For the perturbation constraint, FGSM, BIM, MIM, PGD, FAB and AA are $l_{\infty}$-bounded; CW is both $l_{2}$ and $l_\infty$-bounded. The attack strength is set as $\epsilon=0.3$ for MNIST and is set as $\epsilon=8/255$ for CIFAR10, SVHN and CIFAR100. 

\begin{remark}\label{remark2}
For CW, it is necessary to point out that there exist three different implementation versions: 
\begin{itemize}
\item {\bf AdverTorch}\footnote{\url{https://github.com/BorealisAI/advertorch}}: the strongest but slowest one that implements binary search for finding the best parameter $c$.
\item {\bf TorchAttack}\footnote{\url{https://github.com/Harry24k/adversarial-attacks-pytorch}}: a faster but weaker version of AdverTorch that omits the binary search process.
\item {\bf PGD-100}\footnote{\url{https://github.com/TLMichael/Hypocritical-Perturbation/}}: a PGD-based version that performs 100-step PGD by replacing the loss function with the CW loss.
\end{itemize}
It is noteworthy that many existing works adopt the {\bf PGD-100}-version for CW evaluation. Based on our empirical studies, the robust accuracy of most models will drop significantly when confronted with real CW attacks. Since the time cost of $l_\infty$-CW evaluation by {\bf AdverTorch} and {\bf TorchAttack} is too high to afford, the $l_\infty$-CW attack usually adopts the {\bf PGD-100}-version, and $l_2$-CW attack can adopt the {\bf AdverTorch}- and {\bf TorchAttack}-versions.
\end{remark}

\paragraph{Platform} The empirical studies are realized using python language with PyTorch framework. Both {\bf AdverTorch} and {\bf TorchAttack} libraries are used for evaluation. All experiments are conducted on a GeForce RTX 2080ti GPU with CUDA 10.1 and Ubuntu 16.04.

\begin{table}[t]
\caption{Details of the selected datasets.}\label{tab.datasets}
\centering
\vspace{-0.2cm}
\scalebox{0.8}{
\begin{tabular}{lrrrrr}
\toprule
Dataset  & \# Train Samples & \# Test Samples & \# Classes& \# Channels & Image Size\\
\midrule  
MNIST & 60,000 & 10,000 & 10  & 1 & 28$\times$28 \\
CIFAR10 & 60,000 & 10,000 & 10 & 3 & 32$\times$32 \\
SVHN & 73,257 & 26,032 & 10 & 3 & 32$\times$32 \\
CIFAR100 & 60,000 & 10,000 & 100 & 3 & 32$\times$32 \\
\bottomrule
\end{tabular}}
\end{table}

\begin{table*}[t]
    \centering
    \caption{Robustness evaluation (\%) (left part) and proportion of decision changes by the FSU module (right part) on datasets MNIST, CIFAR10, SVHN and CIFAR100.}
    \vspace{-0.2cm}
    \label{tab:alltab}
    \renewcommand\arraystretch{0.9}
    \scalebox{0.75}{
    \begin{tabular}{lccccccccccccl}
    \toprule
    Method & Clean & FGSM$_\infty$ & BIM$_\infty$ & MIM$_\infty$ & PGD$_\infty$ & \#Epochs / Time & $P$ & Clean & FGSM$_\infty$ & BIM$_\infty$ & MIM$_\infty$ & PGD$_\infty$ & ($\alpha$,$\beta$)\\
    &&{\tiny (TorchAttack)}&{\tiny (TorchAttack)}&{\tiny (TorchAttack)}&{\tiny (TorchAttack)}&&&&{\tiny (TorchAttack)}&{\tiny (TorchAttack)}&{\tiny (TorchAttack)}&{\tiny (TorchAttack)}&\\
    &&&($2/255$)&($2/255$)&($2/255$)&&&&&($2/255$)&($2/255$)&($2/255$)&\\
    \midrule
    \multicolumn{14}{c}{\pmb{MNIST} ($\epsilon=0.3$)\quad 40 steps for BIM, MIM and PGD} \\ 
    \cdashline{1-14}[1pt/1pt]\noalign{\smallskip}
    Baseline-NT & 99.19 & 17.34 & 00.21 & 00.58 & 00.05 & 100 / 944s & $P_1$ (\%)  & 00.13$\pm$0.03 & 75.61$\pm$2.11 & 79.93$\pm$2.08 & 75.23$\pm$2.34 & 78.87$\pm$2.42 & (/, /)\\
    FSU-NT & 99.16$\pm$0.04 & 90.59$\pm$1.02 & 91.63$\pm$0.52 & 85.86$\pm$1.53 & 85.42$\pm$1.18 & 100 / 1306s & $P_2$ (\%)  & 43.15$\pm$6.79 & 97.10$\pm$0.90 & 92.49$\pm$2.16 & 80.62$\pm$5.57 & 90.61$\pm$1.89 & (1, 1)\\
    \midrule
    \multicolumn{14}{c}{\pmb{CIFAR10} ($\epsilon=8/255$)\quad 10 steps for BIM, MIM and PGD} \\
    \cdashline{1-14}[1pt/1pt]\noalign{\smallskip}
    Baseline-NT & 92.68 & 10.84 & 00.00 & 00.04 & 00.01 & 200 / 6826s & $P_1$ (\%) & 01.51$\pm$0.07 & 66.39$\pm$0.79 & 65.46$\pm$0.61 & 65.06$\pm$0.55 & 66.79$\pm$0.63 & (/, /)\\
    FSU-NT  & 92.36$\pm$0.21 & 73.24$\pm$0.44 & 59.11$\pm$0.78 & 57.06$\pm$0.80 & 59.68$\pm$0.59 & 200 / 7152s & $P_2$ (\%) & 35.43$\pm$4.99 & 94.36$\pm$0.40 & 90.44$\pm$0.74 & 87.54$\pm$1.08 & 89.49$\pm$0.47 & (1, 1)\\
    \midrule
    \multicolumn{14}{c}{\pmb{SVHN} ($\epsilon=8/255$)\quad 10 steps for BIM, MIM and PGD} \\
    \cdashline{1-14}[1pt/1pt]\noalign{\smallskip}
    Baseline-NT & 96.40 & 32.14 & 01.16 & 01.59 & 01.35 & 100 / 4715s & $P_1$ (\%) & 00.49$\pm$0.04 & 54.07$\pm$0.99 & 74.20$\pm$0.65 & 73.81$\pm$0.52 & 75.73$\pm$0.55 & (/, /)\\
    FSU-NT  & 96.25$\pm$0.07& 80.06$\pm$0.33 & 64.98$\pm$0.80 & 58.92$\pm$0.85 & 67.83$\pm$0.93 & 100 / 4985s & $P_2$ (\%) & 35.80$\pm$3.93 & 88.47$\pm$0.52 & 86.32$\pm$0.49 & 78.04$\pm$0.61 & 88.05$\pm$0.60 & (1, 1)\\
    \midrule
    \multicolumn{14}{c}{\pmb{CIFAR100} ($\epsilon=8/255$)\quad 10 steps for BIM, MIM and PGD} \\
    \cdashline{1-14}[1pt/1pt]\noalign{\smallskip}
    Baseline-NT & 73.64 & 09.41 & 00.18 & 00.29 & 00.13 &300 / 10790s & $P_1$ (\%) & 17.13$\pm$0.15 & 56.67$\pm$0.31 & 63.48$\pm$0.21 & 64.21$\pm$0.16 & 64.42$\pm$0.33 & (/, /)\\
    FSU-NT & 70.03$\pm$0.15 & 49.39$\pm$0.25 & 36.06$\pm$0.35 & 42.20$\pm$0.21 & 34.76±0.25 & 300 / 11026s & $P_2$ (\%) & 10.58$\pm$0.84 & 61.23$\pm$0.32 & 62.43$\pm$0.37 & 67.75$\pm$0.33 & 59.60$\pm$0.24 & (4, 1)\\
    \bottomrule
    \end{tabular}
    }
\vspace{-0.2cm}
\end{table*}

\subsection{Natural Training for Mild Attacks}\label{subsec.mild}

\subsubsection{Settings} 
In this section, a pure backbone DNN model is learned by natural training for each dataset, with the classic softmax cross-entropy loss. 

\textbf{Backbones and Parameters:} LeNet-5 is used as the backbone for MNIST; while ResNet-18 is used for CIFAR10, SVHN and CIFAR100. The FSU module is employed in both training and testing (including attacking and predicting) phases. In training, adaptive moment estimation (Adam) is used for optimization, with an initial learning rate of $0.001$, a cosine decay period of $100$, a minimum learning rate of $10^{-6}$, and a batch size of $128$. The training epochs for the four datasets are $100$, $200$, $100$ and $300$, respectively. 

\textbf{Training and Evaluation:} The whole training procedure does not introduce any additional information, i.e., adversarial perturbation, extra or augmented data, auxiliary architecture, etc. The only modification is the insertion of the FSU module between two layers of the backbone, which performs a feature reconstruction for each example. The FSU module is also employed for testing (includes attacking and predicting), where the examples are tested batch-by-batch with batch size $1000$. Since the FSU module only relies on the feature statistics of the current batch, the reconstructions during training and testing do not affect each other. Typical gradient-based FGSM, BIM, MIM, and PGD attacks ({\bf TorchAttack} library) are performed. The number of attacking iterations for BIM, MIM, and PGD is 40 (with step size $2/255$) for MNIST, and 10 (with step size $2/255$) for CIFAR10, SVHN, and CIFAR100. The noise intensity parameter $\beta$ for feature Std is fixed as $1$, and the noise intensity parameter $\alpha$ for feature mean is searched in $[0.2,6]$ with step size $0.2$. The best $\alpha$ is selected by PGD-10. Both $\alpha$ and $\beta$ keep consistent in training and testing for a dataset. The random seed for $\mathcal{N}(0, 1)$ in Eq.~(\ref{eq.noise}) is consistently set as $6666$. Each experiment is conducted for $10$ trials, where the mean and standard deviation are recorded. 

\textbf{Reconstruction Positions:} The FSU module can be inserted after any hidden layer of the network. Different reconstruction positions may result in different performances. We choose the reconstruction position based on two-fold cross-validation on the training set under the FGSM attack.
\begin{itemize}
\item For MNIST, we employ it after the first and second pooling layers during training, and after the first pooling layer during attacking and predicting. 
\item For CIFAR10, we employ it after the first convolution block during training, and after the fourth convolution block during attacking and predicting. 
\item For SVHN, we employ it after the third convolution block during training, and after the fourth convolution block during attacking and predicting. 
\item For CIFAR100, we employ it after the fourth convolution block during training, attacking, and predicting. 
\end{itemize}

\subsubsection{Empirical Studies} 
Basically, the following analyses can be made for this part of experiment.

\textbf{Overall Performance:} The left part of Table~\ref{tab:alltab} reports the clean accuracy, robust accuracy, and time cost of natural training without FSU (Baseline-NT) and with FSU (FSU-NT). Obviously, the robustness has been significantly improved by FSU under all attacks. More importantly, the time cost of FSU is substantially the same as the baseline, which is much more efficient than existing mainstream defense methods. Furthermore, the right part of Table~\ref{tab:alltab} reports the proportion of examples with decision change caused by FSU ($P_1$) and the proportion of examples in $P_1$ with decision change from wrong to correct ($P_2$). For clean examples, $P_1$ is very small, indicating that the FSU module has trivial impacts on clean examples. However, for perturbed examples, the value of $P_1$ under all attacks is large. Moreover, the proportion of decisions changed from wrong to correct (i.e., $P_2$) is high, i.e., batch-level feature statistics successfully helped many examples resist the attacks. 

\textbf{Ablation Study:} Fig.~\ref{fig.ablation} shows the robust accuracy against FGSM attack on MNIST in the cases of employing the FSU module in different stages. It can be observed that employing FSU only in training has similar accuracy to the baseline. This is a rational phenomenon, since no adversarial information is introduced during training, the reconstruction of clean training data has trivial impacts on test robustness. Despite this, employing FSU in training, attacking and predicting phases shows a significant improvement. As aforementioned, in training, FSU encourages the network to learn some uncertainty via Gaussian noise; while in testing, FSU reconstructs features in each epoch of attacking to recover some domain characteristics for classification.

\textbf{Comparison with State-of-the-Art Methods:} We further compare the FSU method with recent state-of-the-art adversarial training methods, including PCL (ICCV'19)~\cite{pcl}, TRADES (ICML'19)~\cite{trades}, AT-AWP (NeurIPS'20)~\cite{wu2020Adversarial}, MAIL-TRADES (NeurIPS'21)~\cite{Wang2021Probabilistic}, MLCAT$_{\rm WP}$ (ICML'22)~\cite{yu2022Understanding}, AT+RiFT (ICCV'23)~\cite{Zhu2023Improving}, and DKL (NeurIPS'24)~\cite{cui2024decoupled}. Besides, an up-to-date non-adversarial training method is also compared, i.e., SAM  (ICML'24)~\cite{zhang2024duality}. The backbones, training parameters, and training strategies all adopt the recommendations in the original papers. However, the models trained by AT-AWP and MLCAT$_{\rm wp}$ for MNIST under the default settings fail to show any robustness against the chosen attacks. Thus, AT-AWP is re-trained on MNIST for $200$ epochs by PGD-40 with step size $0.03$, and MLCAT$_{\rm wp}$ is re-trained on MNIST for $50$ epochs by PGD-40 with step size $0.02$. For each adversarial training method, the best model is selected via PGD-10 with step size $2/255$. The comparison results are reported in Table~\ref{tab:mild} (Please refer to Appendix~\ref{append.mild}), where the results of FSU are the average of 10 trials. It can be seen that FSU did not sacrifice much clean accuracy; at the same time, it shows state-of-the-art performance under many attacks, especially on datasets CIFAR10 and CIFAR100. In addition, the training cost of FSU is only a few tenths of that of other methods, since it utilizes natural training rather than adversarial training. From this perspective, FSU achieves significant robustness at a very low time cost. 

\begin{remark}\label{remark3}
It is noteworthy that applying the FSU module for testing requires the examples to be perturbed in batch, which is unsuitable in many practical attacking scenarios. Besides, many works have demonstrated that models learned by natural training may not be robust enough when confronted with powerful attacks. In the next subsection, we will try to overcome these limitations.
\end{remark}

\begin{table*}[t]
    \centering
    \caption{Robust accuracy (\%) under various attacks.}
    \vspace{-0.2cm}
    \label{tab:powerful}
    \renewcommand\arraystretch{0.9}
    \scalebox{0.75}{
    \begin{tabular}{lp{1.4cm}lllllllllll}
    \toprule
    &\underline{Clean}&\multicolumn{7}{c}{\underline{Gradient}} &\underline{Adaptive} &\underline{Ensemble} &\underline{Optimization} &($\alpha$,$\beta$)\\
    Attack& Cln. & FGSM$_\infty$ & BIM$_\infty$-10 & MIM$_\infty$-40 & PGD$_\infty$-20 &PGD$_\infty$-40 &PGD$_\infty$-100 & CW$_\infty$ & FAB$_\infty$ & AA$_\infty$ & CW$_2$ &\\
    &&{\scriptsize(AdverTorch)}&{\scriptsize(AdverTorch)}&{\scriptsize(AdverTorch)}&{\scriptsize(AdverTorch)}&{\scriptsize(AdverTorch)}&{\scriptsize(AdverTorch)}&{\scriptsize(PGD-100)}&{\scriptsize(AutoAttack)}&{\scriptsize(AutoAttack)}&{\scriptsize(TorchAttack)}&\\
    Method&&&{\scriptsize(10, 2/255)}&{\scriptsize(40, 2/255)}&{\scriptsize(20, 2/255)}&{\scriptsize(40, 2/255)}&{\scriptsize(100, 2/255)}&{\scriptsize(100, 2/255)}&{\scriptsize(5, 100)}&&{\scriptsize(50, 0.01)}\\
    \midrule																											
    \multicolumn{13}{c}{\pmb{MNIST} ($\epsilon=0.3$)}\\																											
    \cdashline{1-13}[1pt/1pt]\noalign{\smallskip}																											
TRADES	&	99.53 	&	97.55 	&	99.10 	&	96.33 	&	93.92 	&	98.45 	&	97.26 	&	95.72 	&	98.16 	&	92.98 	&	95.77 	& (/,	/)\\
TRADES+FSU	&	99.50 	&	97.68 	$\bullet$ &	99.07 	$\circ$ &	96.46 	$\bullet$ &	94.16 	$\bullet$ &	98.48 	$\bullet$ &	97.23 	$\circ$ &	95.83 	$\bullet$ &	98.07 	$\circ$ &	93.03 	$\bullet$ &	96.03 	$\bullet$ & (0.3, 1)\\
\cdashline{1-13}[1pt/1pt]\noalign{\smallskip}																											
AT-AWP	&	98.73 	&	95.47 	&	98.26 	&	94.82 	&	97.53 	&	96.20 	&	94.54 	&	94.48 	&	93.65 	&	93.18 	&	98.34 	& (/,	/)\\
AT-AWP+FSU	&	98.85 	&	95.86 	$\bullet$ &	98.26 	$\bullet$ &	95.09 	$\bullet$ &	97.61 	$\bullet$ &	96.35 	$\bullet$ &	94.94 	$\bullet$ &	94.77 	$\bullet$ &	93.98 	$\bullet$ &	93.36 	$\bullet$ &	98.52 	$\bullet$ & (0.5, 1)\\
\cdashline{1-13}[1pt/1pt]\noalign{\smallskip}																											
MAIL-TRADES	&	99.14 	&	97.27 	&	98.72 	&	96.42 	&	98.21 	&	97.09 	&	95.60 	&	95.69 	&	94.78 	&	94.03 	&	98.09 	& (/,	/)\\
MAIL-TRADES+FSU	&	99.21 	&	97.83 	$\bullet$ &	98.89 	$\bullet$ &	97.15 	$\bullet$ &	98.51 	$\bullet$ &	97.66 	$\bullet$ &	96.45 	$\bullet$ &	96.57 	$\bullet$ &	95.86 	$\bullet$ &	95.23 	$\bullet$ &	98.37 	$\bullet$ & (0.1, 1)\\
\cdashline{1-13}[1pt/1pt]\noalign{\smallskip}																											
MLCATWP	&	98.88 	&	97.24 	&	98.42 	&	96.16 	&	97.76 	&	97.02 	&	96.15 	&	96.17 	&	93.95 	&	93.64 	&	97.68 	& (/,	/)\\
MLCAT$_{\rm WP}$+FSU	&	98.82 	&	96.90 	$\circ$ &	98.31 	$\circ$ &	95.87 	$\circ$ &	97.66 	$\circ$ &	96.55 	$\circ$ &	95.44 	$\circ$ &	95.51 	$\circ$ &	94.18 	$\circ$ &	93.77 	$\bullet$ &	97.69 	$\bullet$ & (0.5, 1)\\
\cdashline{1-13}[1pt/1pt]\noalign{\smallskip}																											
AT+RiFT	&	99.44 	&	95.20 	&	99.01 	&	94.25 	&	98.08 	&	95.06 	&	60.72 	&	63.44 	&	62.13 	&	17.46 	&	97.85 	& (/,	/)\\
AT+RiFT+FSU	&	99.41 	&	97.93 	$\bullet$ &	98.92 	$\circ$ &	97.06 	$\bullet$ &	98.43 	$\bullet$ &	97.48 	$\bullet$ &	96.16 	$\bullet$ &	96.22 	$\bullet$ &	94.56 	$\bullet$ &	90.32 	$\bullet$ &	98.18 	$\bullet$ & (0.9, 1)\\
\cdashline{1-13}[1pt/1pt]\noalign{\smallskip}																											
DKL	&	98.80 	&	97.07 	&	98.26 	&	96.73 	&	97.48 	&	96.96 	&	96.49 	&	96.53 	&	94.34 	&	85.00 	&	98.11 	& (/,	/)\\
DKL+FSU	&	98.81 	&	97.87 	$\bullet$ &	98.28 	$\bullet$ &	97.14 	$\bullet$ &	90.83 	$\circ$ &	88.68 	$\circ$ &	79.57 	$\circ$ &	79.00 	$\circ$ &	81.79 	$\circ$ &	04.56 	$\circ$ &	98.22 	$\bullet$ & (0.7, 1)\\

    \midrule
    \multicolumn{13}{c}{\pmb{CIFAR10} ($\epsilon=8/255$)}\\
    \cdashline{1-13}[1pt/1pt]\noalign{\smallskip}
    TRADES	&	85.76 	&	61.61 	&	53.67 	&	53.15 	&	52.64 	&	52.21 	&	52.13 	&	52.59 	&	52.15 	&	50.44 	&	21.47 	& (/,	/)\\
TRADES+FSU	&	86.21 	&	61.51 	$\circ$ &	53.50 	$\circ$ &	52.95 	$\circ$ &	52.36 	$\circ$ &	52.02 	$\circ$ &	51.88 	$\circ$ &	52.65 	$\bullet$ &	52.52 	$\bullet$ &	50.43 	$\circ$ &	22.37 	$\bullet$ & (0.1, 1)\\
\cdashline{1-13}[1pt/1pt]\noalign{\smallskip}																											
AT-AWP	&	84.82 	&	58.90 	&	52.28 	&	51.92 	&	51.38 	&	51.09 	&	50.98 	&	49.66 	&	47.47 	&	46.83 	&	26.00 	& (/,	/)\\
AT-AWP+FSU	&	82.04 	&	60.09 	$\bullet$ &	55.75 	$\bullet$ &	55.58 	$\bullet$ &	55.39 	$\bullet$ &	55.24 	$\bullet$ &	55.25 	$\bullet$ &	51.26 	$\bullet$ &	49.50 	$\bullet$ &	48.88 	$\bullet$ &	45.82 	$\bullet$ & (0.5, 1)\\
\cdashline{1-13}[1pt/1pt]\noalign{\smallskip}																											
MAIL-TRADES	&	84.03 	&	55.96 	&	50.89 	&	50.78 	&	50.60 	&	50.36 	&	50.40 	&	47.91 	&	47.47 	&	46.82 	&	24.14 	& (/,	/)\\
MAIL-TRADES+FSU	&	84.12 	&	57.05 	$\bullet$ &	51.92 	$\bullet$ &	51.76 	$\bullet$ &	51.46 	$\bullet$ &	51.33 	$\bullet$ &	51.28 	$\bullet$ &	48.16 	$\bullet$ &	47.58 	$\bullet$ &	46.97 	$\bullet$ &	35.84 	$\bullet$ & (0.5, 1)\\
\cdashline{1-13}[1pt/1pt]\noalign{\smallskip}																											
MLCATWP	&	85.11 	&	58.87 	&	55.75 	&	55.37 	&	55.56 	&	55.32 	&	55.35 	&	50.27 	&	48.97 	&	47.43 	&	55.32 	& (/,	/)\\
MLCAT$_{\rm WP}$+FSU	&	83.67 	&	60.38 	$\bullet$ &	58.43 	$\bullet$ &	58.02 	$\bullet$ &	58.37 	$\bullet$ &	58.33 	$\bullet$ &	58.36 	$\bullet$ &	51.22 	$\bullet$ &	50.18 	$\bullet$ &	48.66 	$\bullet$ &	62.01 	$\bullet$ & (0.7, 1)\\
\cdashline{1-13}[1pt/1pt]\noalign{\smallskip}																											
AT+RiFT	&	76.01 	&	40.79 	&	34.28 	&	33.93 	&	33.53 	&	33.09 	&	33.00 	&	33.24 	&	30.99 	&	30.03 	&	20.52 	& (/,	/)\\
AT+RiFT+FSU	&	72.00 	&	48.74 	$\bullet$ &	45.64 	$\bullet$ &	45.46 	$\bullet$ &	45.28 	$\bullet$ &	45.22 	$\bullet$ &	45.21 	$\bullet$ &	40.97 	$\bullet$ &	39.32 	$\bullet$ &	38.64 	$\bullet$ &	54.80 	$\bullet$ & (0.7, 1)\\
\cdashline{1-13}[1pt/1pt]\noalign{\smallskip}																											
DKL	&	79.84 	&	59.06 	&	55.70 	&	55.64 	&	55.37 	&	55.27 	&	55.24 	&	51.60 	&	51.35 	&	50.97 	&	50.42 	& (/,	/)\\
DKL+FSU	&	82.13 	&	60.29 	$\bullet$ &	56.35 	$\bullet$ &	56.20 	$\bullet$ &	56.08 	$\bullet$ &	55.80 	$\bullet$ &	55.80 	$\bullet$ &	52.10 	$\bullet$ &	51.66 	$\bullet$ &	51.07 	$\bullet$ &	57.64 	$\bullet$ & (0.3, 1)\\
    \midrule
    \multicolumn{13}{c}{\pmb{SVHN} ($\epsilon=8/255$)}\\
    \cdashline{1-13}[1pt/1pt]\noalign{\smallskip}
    TRADES	&	91.73 	&	69.10 	&	58.63 	&	57.84 	&	57.56 	&	57.19 	&	57.03 	&	53.77 	&	52.84 	&	50.90 	&	12.42 	& (/,	/)\\
TRADES+FSU	&	91.56	&	69.69	$\bullet$ &	61.51	$\bullet$ &	60.80	$\bullet$ &	60.59	$\bullet$ &	60.31	$\bullet$ &	60.18	$\bullet$ &	56.68	$\bullet$ &	56.01	$\bullet$ &	54.66	$\bullet$ &	15.88   $\bullet$ & (0.9, 1)\\
\cdashline{1-13}[1pt/1pt]\noalign{\smallskip}																											
AT-AWP	&	93.39	&	66.62	&	55.78	&	54.40	&	53.42	&	52.23	&	51.74	&	50.77	&	49.31	&	46.88	&	03.92	& (/,	/)\\
AT-AWP+FSU	&	91.69	&	69.16	$\bullet$ &	61.81	$\bullet$ &	61.14	$\bullet$ &	60.68	$\bullet$ &	60.24	$\bullet$ &	59.97	$\bullet$ &	55.31	$\bullet$ &	53.85	$\bullet$ &	51.97	$\bullet$ &	06.98	$\bullet$ & (0.3, 1)\\
\cdashline{1-13}[1pt/1pt]\noalign{\smallskip}																											
MAIL-TRADES	&	92.40 	&	70.74 	&	60.39 	&	59.48 	&	59.57 	&	59.25 	&	59.08 	&	55.54 	&	55.08 	&	52.53 	&	05.68 	& (/,	/)\\
MAIL-TRADES+FSU	&	91.24 	&	71.87 	$\bullet$ &	61.70 	$\bullet$ &	60.75 	$\bullet$ &	60.96 	$\bullet$ &	60.61 	$\bullet$ &	60.51 	$\bullet$ &	56.63 	$\bullet$ &	56.07 	$\bullet$ &	53.69 	$\bullet$ &	09.88 	$\bullet$ & (0.3, 1)\\
\cdashline{1-13}[1pt/1pt]\noalign{\smallskip}																											
AT+RiFT	&	91.33	&	53.80	&	41.57	&	40.83	&	39.79	&	39.15	&	38.75	&	39.70	&	37.58	&	36.04	&	04.49	& (/,	/)\\
AT+RiFT+FSU	&	88.66	&	60.39	$\bullet$ &	53.41	$\bullet$ &	53.07	$\bullet$ &	52.36	$\bullet$ &	51.82	$\bullet$ &	51.55	$\bullet$ &	48.05	$\bullet$ &	46.67	$\bullet$ &	45.38	$\bullet$ &	10.55	$\bullet$ & (0.5, 1)\\
\cdashline{1-13}[1pt/1pt]\noalign{\smallskip}																											
DKL	&	95.13 	&	90.88 	&	65.36 	&	60.33 	&	59.37 	&	52.98 	&	48.08 	&	47.15 	&	54.59 	&	24.23 	&	08.06 	& (/,	/)\\
DKL+FSU	&	91.18 	&	71.02 	$\circ$ &	61.46 	$\circ$ &	52.66 	$\circ$ &	58.37 	$\circ$ &	56.46 	$\bullet$ &	55.50 	$\bullet$ &	51.06 	$\bullet$ &	61.37 	$\bullet$ &	41.36 	$\bullet$ &	12.52 	$\bullet$ & (0.1, 1)\\
    \midrule
    \multicolumn{13}{c}{\pmb{CIFAR100} ($\epsilon=8/255$)}\\
    \cdashline{1-13}[1pt/1pt]\noalign{\smallskip}
    TRADES	&	58.93 	&	33.61 	&	31.15 	&	31.08 	&	31.06 	&	30.90 	&	30.93 	&	27.94 	&	27.26 	&	26.99 	&	15.26 	& (/,	/)\\
TRADES+FSU	&	60.17 	&	34.35 	$\bullet$ &	32.16 	$\bullet$ &	32.11 	$\bullet$ &	32.07 	$\bullet$ &	32.00 	$\bullet$ &	31.95 	$\bullet$ &	28.41 	$\bullet$ &	27.68 	$\bullet$ &	27.36 	$\bullet$ &	15.85 	$\bullet$ & (0.3, 1)\\
\cdashline{1-13}[1pt/1pt]\noalign{\smallskip}																											
AT-AWP	&	53.61	&	32.72	&	30.91	&	30.89	&	30.85	&	30.84	&	30.83	&	27.20	&	25.44	&	25.12	&	23.63	& (/,	/)\\
AT-AWP+FSU	&	50.66	&	33.12	$\bullet$ &	31.90	$\bullet$ &	31.81	$\bullet$ &	31.77	$\bullet$ &	31.76	$\bullet$ &	31.68	$\bullet$ &	27.02	$\circ$ &	25.28	$\circ$ &	25.00	$\circ$ &	25.51	$\bullet$ & (0.7, 1)\\
\cdashline{1-13}[1pt/1pt]\noalign{\smallskip}																											
MAIL-TRADES	&	61.33 	&	31.29 	&	27.77 	&	27.58 	&	27.46 	&	27.30 	&	27.37 	&	24.07 	&	22.96 	&	22.64 	&	12.58 	& (/,	/)\\
MAIL-TRADES+FSU	&	60.33 	&	31.29  	$\bullet$	&	28.48 	$\bullet$ &	28.30 	$\bullet$ &	28.26 	$\bullet$ &	28.14 	$\bullet$ &	28.13 	$\bullet$ &	23.96 	$\circ$ &	23.14 	$\bullet$ &	22.80 	$\bullet$ &	15.25 	$\bullet$ & (0.5, 1)\\
\cdashline{1-13}[1pt/1pt]\noalign{\smallskip}																											
MLCATWP	&	63.51 	&	33.68 	&	29.48 	&	29.25 	&	28.79 	&	28.66 	&	28.51 	&	28.78 	&	26.16 	&	25.62 	&	17.04 	& (/,	/)\\
MLCAT$_{\rm WP}$+FSU	&	60.21 	&	35.88 	$\bullet$ &	33.46 	$\bullet$ &	33.29 	$\bullet$ &	33.09 	$\bullet$ &	33.05 	$\bullet$ &	33.00 	$\bullet$ &	29.63 	$\bullet$ &	26.99 	$\bullet$ &	26.53 	$\bullet$ &	25.52 	$\bullet$ & (0.9, 1)\\
\cdashline{1-13}[1pt/1pt]\noalign{\smallskip}																											
AT+RiFT	&	44.35	&	20.54	&	18.04	&	17.93	&	17.83	&	17.78	&	17.68	&	15.96	&	14.56	&	14.25	&	16.28	& (/,	/)\\
AT+RiFT+FSU	&	40.54	&	24.20	$\bullet$ &	23.30	$\bullet$ &	23.32	$\bullet$ &	23.31	$\bullet$ &	23.24	$\bullet$ &	23.23	$\bullet$ &	18.89	$\bullet$ &	17.61	$\bullet$ &	17.44	$\bullet$ &	25.63	$\bullet$ & (0.5, 1)\\
\cdashline{1-13}[1pt/1pt]\noalign{\smallskip}																											
DKL	&	60.44 	&	35.21 	&	32.72 	&	32.60 	&	32.49 	&	32.41 	&	32.37 	&	28.61 	&	27.85 	&	27.52 	&	18.71 	& (/,	/)\\
DKL+FSU	&	61.02 	&	35.68 	$\bullet$ &	33.07 	$\bullet$ &	32.96 	$\bullet$ &	32.90 	$\bullet$ &	32.79 	$\bullet$ &	32.74 	$\bullet$ &	28.84 	$\bullet$ &	27.97 	$\bullet$ &	27.63 	$\bullet$ &	20.10 	$\bullet$ & (0.1, 1)\\
    \midrule
    \multicolumn{13}{c}{\pmb{Count for Improvement}}\\
    \cdashline{1-13}[1pt/1pt]\noalign{\smallskip}
Count &------&20 / 23&18 / 23&20 / 23&19 / 23&20 / 23&19 / 23&19 / 23&19 / 23&20 / 23&23 / 23&{\bf 197 / 230}\\    
    \bottomrule
    \end{tabular}
    }
\begin{minipage}{17.5cm}
\vspace{0.1cm}
\footnotesize{Note 1: $\bullet$ / $\circ$ means that the robust accuracy is improved / not improved by incorporating the FSU module.\\
Note 2: For FAB, the two parameters are respectively the number of restarts and the number of iterations in each restart.\\
Note 3: For CW$_2$, the two parameters are respectively the number of optimization iterations and learning rate.\\
Note 4: For MLCAT$_{\rm wp}$ on SVHN, all attacks lose effectiveness, i.e., both clean and robust accuracy equal to 19.59\%, thus is not used.}
\end{minipage}
\vspace{-0.2cm}
\end{table*}

\begin{figure}[t] 
\centering    
\includegraphics[width=0.9\linewidth]{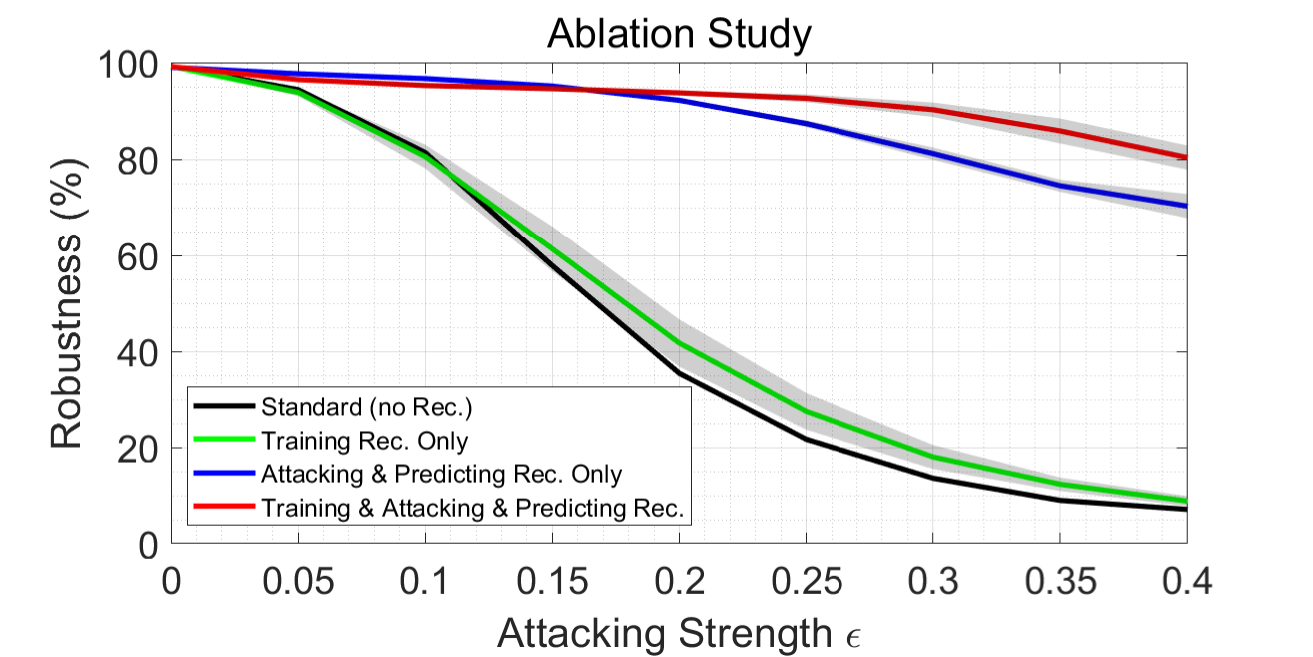}
\caption{Robust accuracy on MNIST under FGSM attack when employing FSU in different stages (with LeNet-5).}
\label{fig.ablation}
\end{figure}

\begin{table*}[t]
    \centering
    \caption{Statistics of robustness improvement.}
    \vspace{-0.2cm}
    \label{tab:statistics}
    \renewcommand\arraystretch{0.9}
    \scalebox{0.75}{
    \begin{tabular}{l|rrrr|rrrr|rrrr|rrrr}
    \toprule
    Method &\multicolumn{4}{c|}{\underline{MNIST}} &\multicolumn{4}{c|}{\underline{CIFAR-10}} &\multicolumn{4}{c|}{\underline{SVHN}} &\multicolumn{4}{c}{\underline{CIFAR-100}}\\
    Improvement &Cln. &Max Rb. &Avg Rb. &Count &Cln. &Max Rb. &Avg Rb. &Count &Cln. &Max Rb. &Avg Rb. &Count &Cln. &Max Rb. &Avg Rb. &Count \\
    &(\%) &(\%) &(\%) &\# / \# &(\%) &(\%) &(\%) &\# / \# &(\%) &(\%) &(\%) &\# / \# &(\%) &(\%) &(\%) &\# / \#\\
    \midrule
TRADES+FSU	&	-0.03 	&	+0.26 	&	+0.08 	&	7 / 10	&	+0.45 	&	+0.90 	&	+0.01 	&	3 / 10	&	-1.16 	&	+4.20 	&	+1.53 	&	10 / 10	&	+1.24 	&	+1.10 	&	+0.78 	&	10 / 10	\\
AT-AWP+FSU	&	+0.12 	&	+0.40 	&	+0.23 	&	10 / 10	&	-2.78 	&	+19.82 	&	+4.63 	&	10 / 10	&	-1.70 	&	+8.23 	&	+5.60 	&	10 / 10	&	-2.95 	&	+1.88 	&	+0.64 	&	7 / 10	\\
MAIL-TRADES+FSU	&	+0.07 	&	+1.20 	&	+0.66 	&	10 / 10	&	+0.09 	&	+11.70 	&	+1.80 	&	10 / 10	&	-0.65 	&	+6.65 	&	+4.78 	&	10 / 10	&	-1.00 	&	+2.67 	&	+0.67 	&	8 / 10	\\
MLCAT$_{\rm WP}$+FSU	&	-0.06 	&	+0.23 	&	-0.23 	&	3 / 10	&	-1.44 	&	+6.69 	&	+2.58 	&	10 / 10	&	/	&	/	&	/	&	/	&	-3.30 	&	+8.48 	&	+3.45 	&	10 / 10	\\
AT+RiFT+FSU	&	-0.03 	&	+72.86 	&	+18.21 	&	9 / 10	&	-4.01 	&	+34.28 	&	+12.59 	&	10 / 10	&	-2.67 	&	+12.80 	&	+10.16 	&	10 / 10	&	-3.81 	&	+9.35 	&	+4.93 	&	10 / 10	\\
DKL+FSU	&	+0.01 	&	+0.80 	&	-14.10 	&	4 / 10	&	+2.29 	&	+7.22 	&	+1.24 	&	10 / 10	&	-3.95 	&	+17.13 	&	+1.08 	&	6 / 10	&	+0.58 	&	+1.39 	&	+0.42 	&	10 / 10	\\
    \cdashline{1-17}[1pt/1pt]\noalign{\smallskip}	
Avg	&	+0.01 	&	+12.63 	&	+0.81 	&	7.2 / 10	&	-0.90 	&	+13.44 	&	+3.81 	&	8.8 / 10	&	-2.03 	&	+9.80 	&	+4.63 	&	9.2 / 10	&	-1.54 	&	+4.15 	&	+1.81 	&	9.2 / 10	\\
    \bottomrule
    \end{tabular}
    }
\begin{minipage}{17.5cm}
\vspace{0.1cm}
\footnotesize{Note 1: ``Max Rb." and ``Avg Rb." mean the maximum improvement and average improvement of robust accuracy under the ten attacks for a dataset.\\
Note 2: ``Count" means the number of improved robust evaluation of a model on a dataset.}
\end{minipage}
\vspace{-0.2cm}
\end{table*}

\begin{figure*}[t]
\centering
\subfloat[\scriptsize CIFAR10 - FGSM$_\infty$ Attack ($\epsilon=8/255$)]{\includegraphics[width=0.49\linewidth]{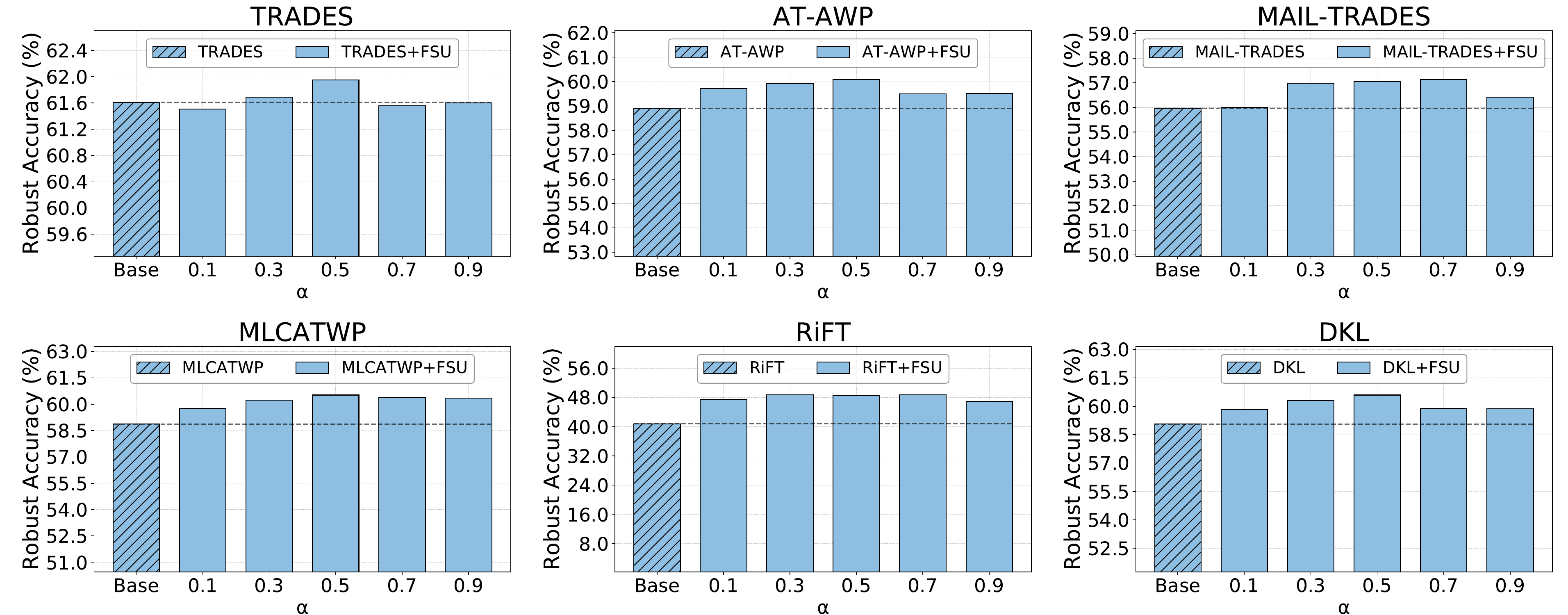}}\quad
\subfloat[\scriptsize CIFAR10 - PGD$_\infty$-100 Attack ($\epsilon=8/255$, step size $2/255$)]{\includegraphics[width=0.49\linewidth]{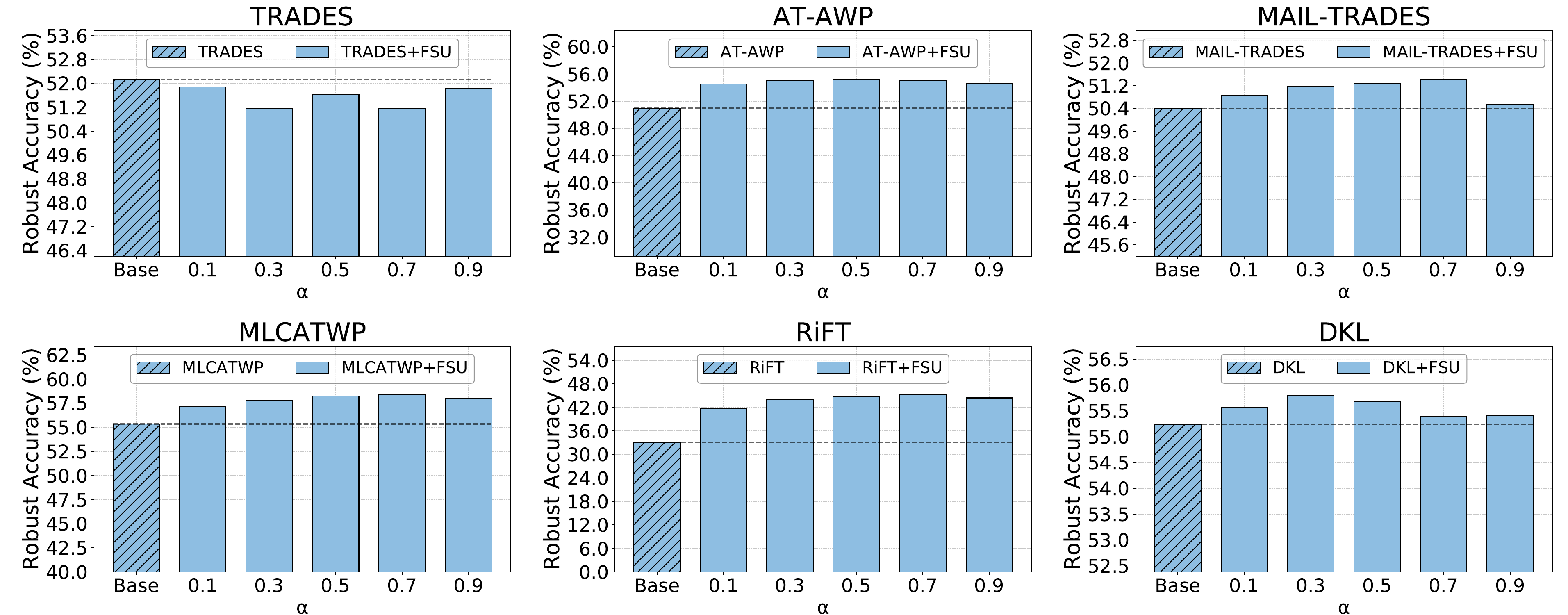}}\\
\subfloat[\scriptsize CIFAR10 - AA$_\infty$ Attack ($\epsilon=8/255$)]{\includegraphics[width=0.49\linewidth]{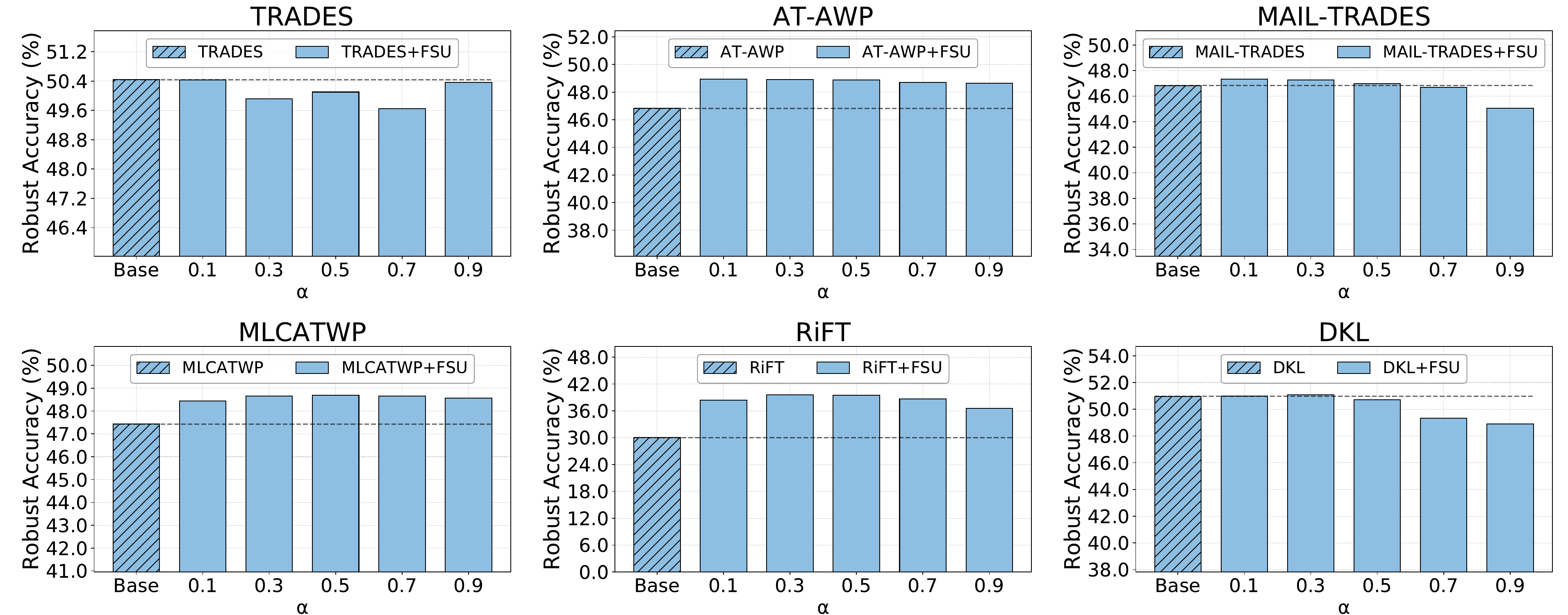}}\quad
\subfloat[\scriptsize CIFAR10 - CW$_2$ Attack ($\epsilon=8/255$, iterations $50$, learning rate $0.01$)]{\includegraphics[width=0.49\linewidth]{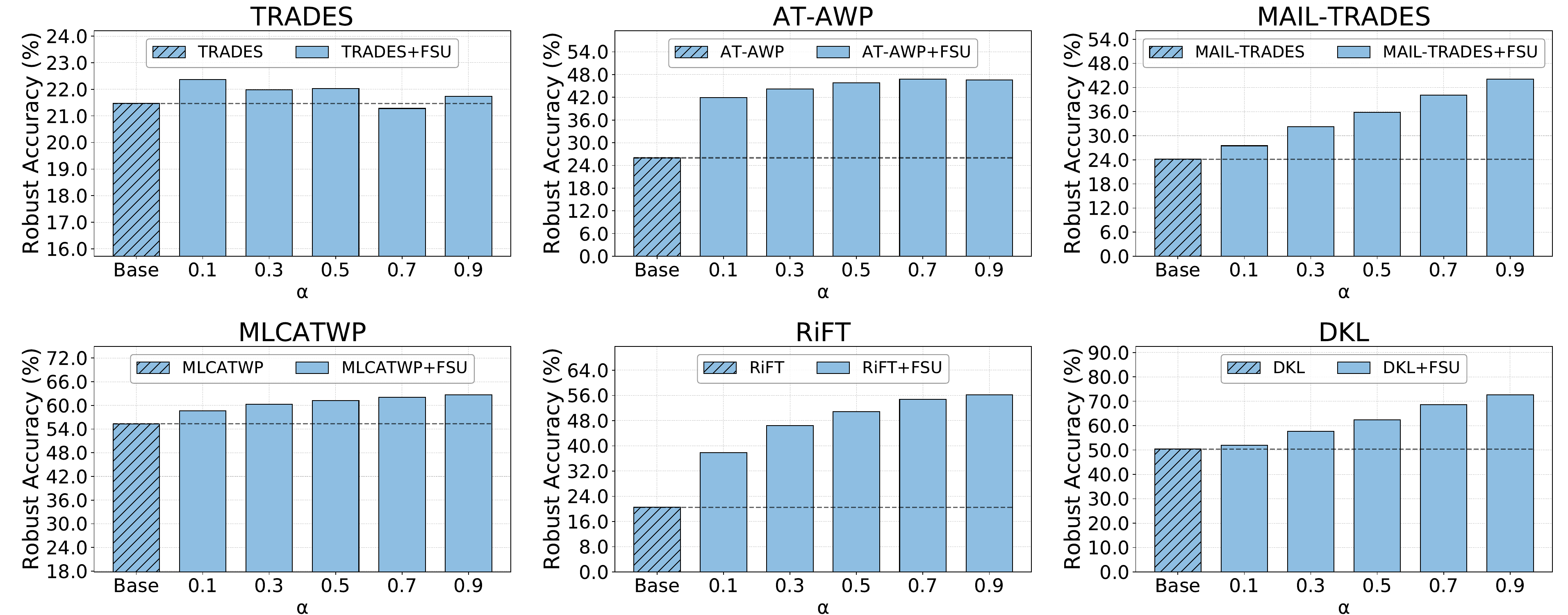}}
\caption{Hyper-parameter sensitivity analysis on dataset CIFAR10.}
\label{fig.sensitivityCifar10}
\end{figure*}

\begin{figure*}[t]
\centering
\subfloat[\scriptsize MNIST]{\includegraphics[width=0.49\linewidth]{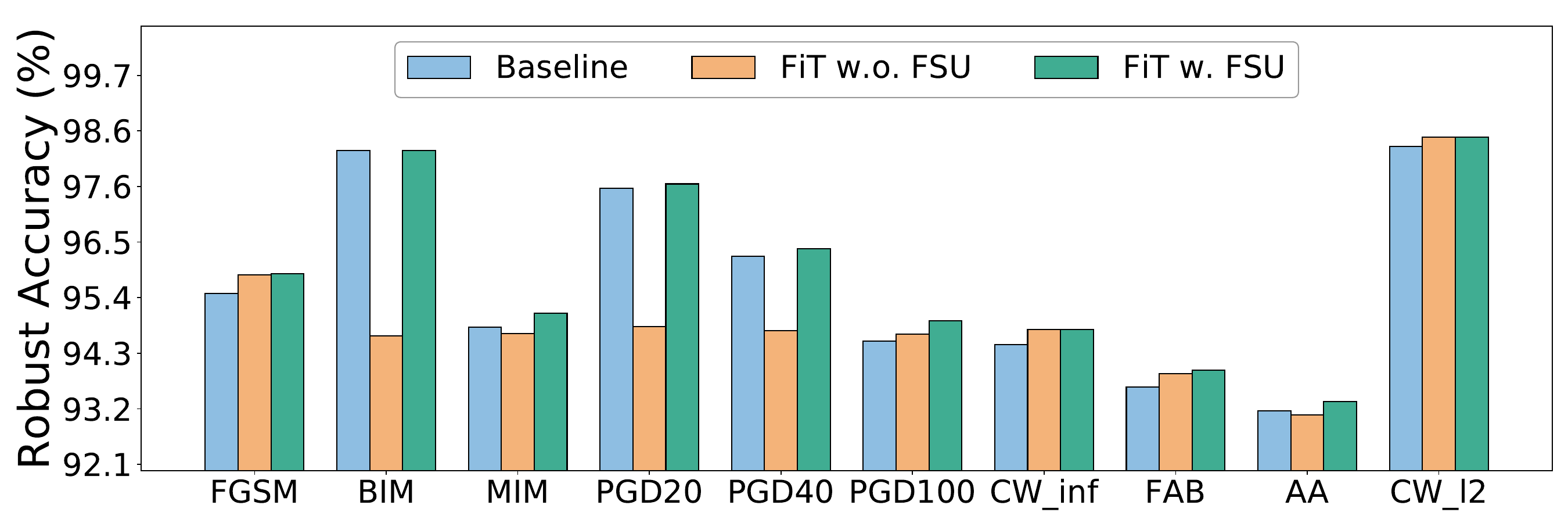}}\quad
\subfloat[\scriptsize CIFAR10]{\includegraphics[width=0.49\linewidth]{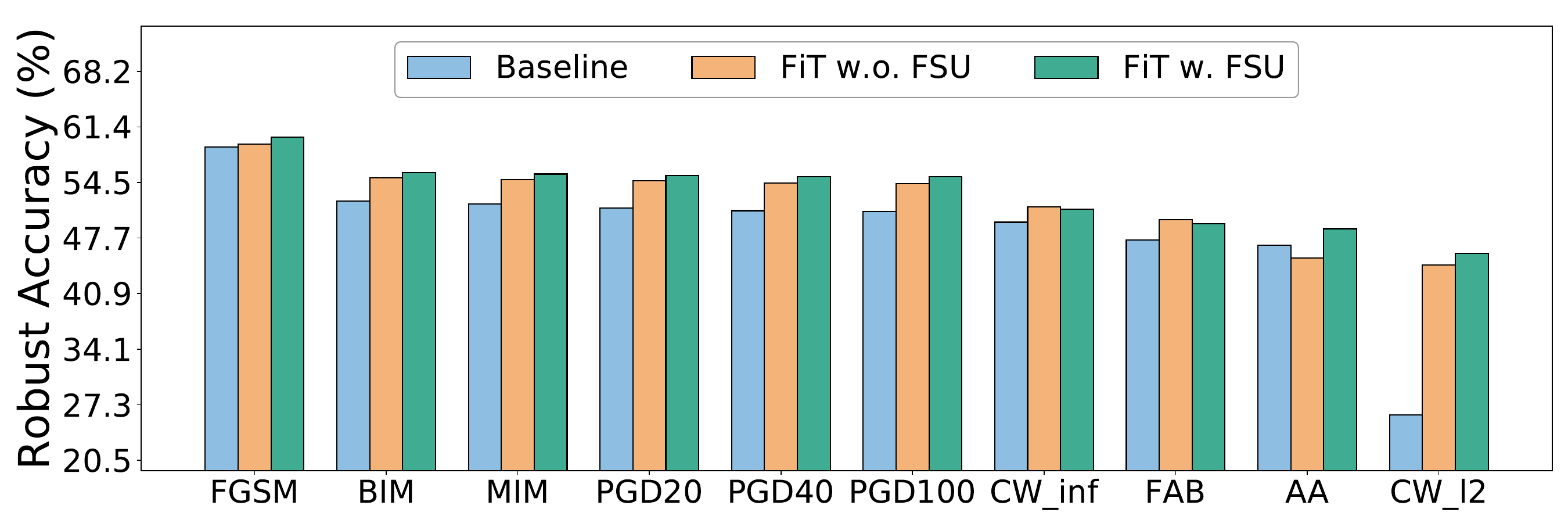}}\\
\vspace{-0.2cm}
\subfloat[\scriptsize SVHN]{\includegraphics[width=0.49\linewidth]{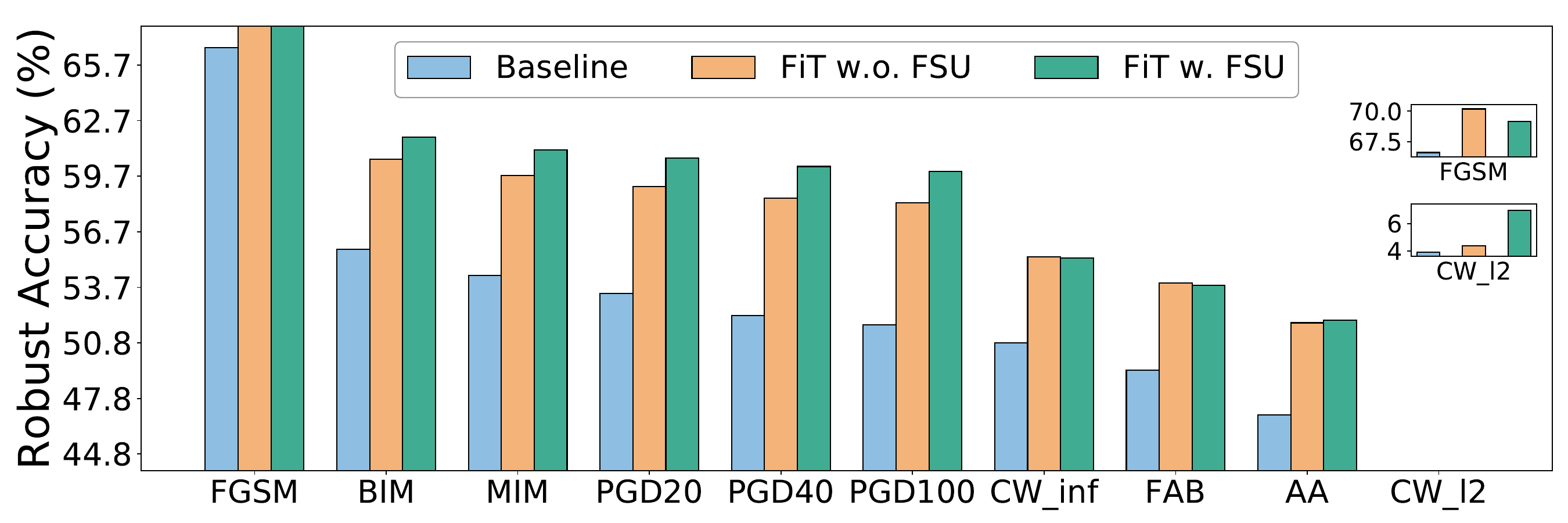}}\quad
\subfloat[\scriptsize CIFAR100]{\includegraphics[width=0.49\linewidth]{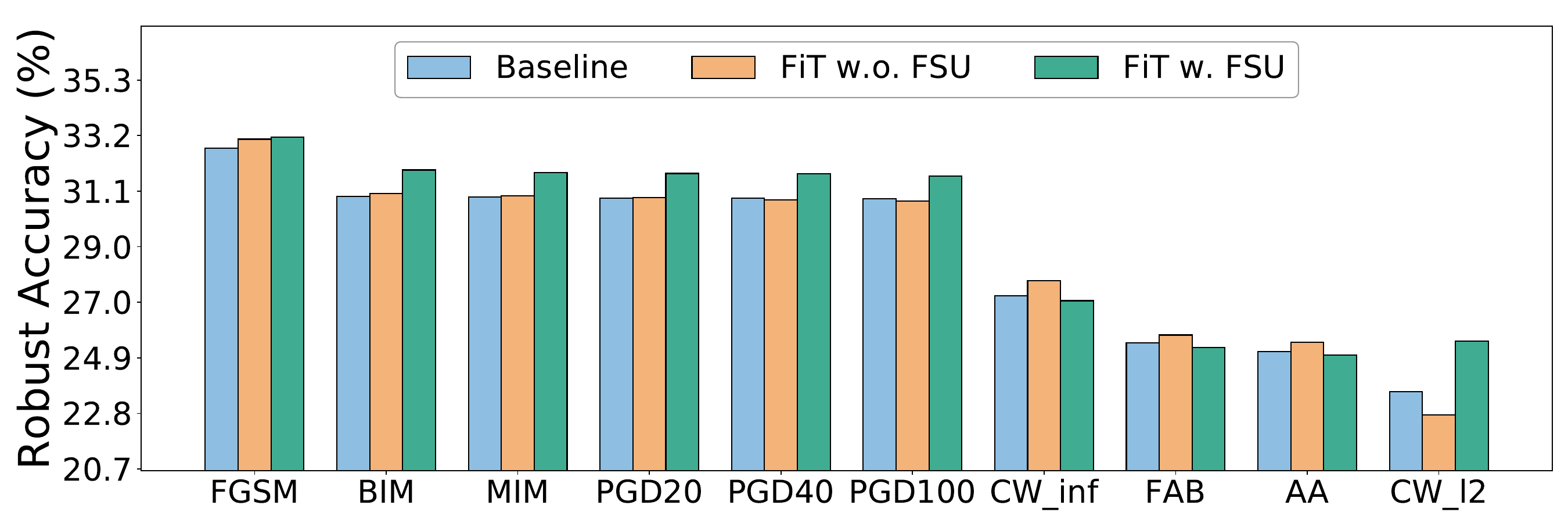}}
\caption{Ablation studies on the FSU module. (Baseline model: AT-AWP; ``Fit w.o. FSU" means pure model fine-tuning without the FSU module; ``Fit w. FSU" means the proposed Algorithm~\ref{alg.FSU}.)}
\vspace{-0.1cm}
\label{fig.ablationAT}
\end{figure*}

\subsection{Adversarial Training for Powerful Attacks}\label{subsec.powerful}

\subsubsection{Settings} In this section, we select several robust models established by adversarial training in the last subsection (i.e., TRADES~\cite{trades}, AT-AWP~\cite{wu2020Adversarial}, MAIL-TRADES~\cite{Wang2021Probabilistic}, MLCAT$_{\rm WP}$~\cite{yu2022Understanding}, AT+RiFT~\cite{Zhu2023Improving} and DKL~\cite{cui2024decoupled} in Table~\ref{tab:mild}), and fine-tune these models by the proposed Algorithm~\ref{alg.FSU}.

\textbf{Training Parameters:}  Similar to Subsection~\ref{subsec.mild}, the noise intensity parameter $\beta$ for feature Std is fixed as $1$, and the noise intensity parameter $\alpha$ for feature mean is searched in $\{0.1,0.3,0.5,0.7,0.9\}$. The number of fine-tuning epochs is set to $20$ consistently for all the models on all the datasets, the learning rate and decay strategy follow the settings of the original training, the random seed is set to default ($0$ or $1$), and PGD-10 attack with step size $2/255$ ($0.1$ for MNIST) is used for generating perturbations.

\textbf{Reconstruction Positions:} It is too time-consuming to cross-validate the reconstruction position for adversarial training. In this experiment, we consistently apply the feature reconstruction in the latent space for all the models on all the datasets, i.e., the FSU module is inserted between the logits layer and the classification layer. 

\textbf{Evaluation Protocols:} Typical gradient-based FGSM, BIM, MIM and PGD attacks ({\bf AdverTorch} library), optimization-based CW attack, as well as adaptive boundary attack FAB and ensemble attack AA are performed. Specifically, we perform a $10$-step BIM attack, a $40$-step MIM attack, a $20$-step PGD attack, a $40$-step PGD attack, and a $100$-step PGD attack, all with step size $2/255$. For FAB, the number of restarts is set to $5$, the number of iterations for each restart is set to $100$, and all the other attacking parameters are set to default. All of FGSM, BIM, MIM, PGD, FAB, and AA are $l_{\infty}$-bounded. As for CW, we test both $l_{\infty}$-bounded {\bf PGD-100}-version and $l_2$-bounded {\bf TorchAttack}-version. For {\bf PGD-100}-version, the step size is set to $2/255$; for  {\bf TorchAttack}-version, the attacking parameters are set to default, i.e., $50$ optimization iterations with a learning rate of $0.01$. 

\subsubsection{Empirical Studies} Basically, the following analyses can be made for this part of experiment.

\textbf{Overall Performance:} Table~\ref{tab:powerful} reports the robust accuracies of the several well-established models under various attacks, as well as the robust accuracies of the models fine-tuned by Algorithm~\ref{alg.FSU}. Notably, we have the following observations.

\begin{itemize}
\item It is obvious from Table~\ref{tab:powerful} that the FSU module helped these models better defend against almost all of these attacks. It is comprehensively effective for the chosen attacks, especially on datasets CIFAR10, SVHN and CIFAR100. As for dataset MNIST, the robust accuracies of the original models are already high, without too much room for improvement. Despite this, it achieved general improvements on MNIST for TRADES, AT-AWP, MAIL-TRADES and AT+RiFT, and similar performance for MLCAT$_{\rm WP}$.
\item The robustness improvement under CW$_2$ is remarkable. It is noteworthy that we have tested a real optimization-based CW attack, i.e., the {\bf TorchAttack}-version. As can be seen from Table~\ref{tab:powerful}, the robustness under CW$_\infty$ is similar to that under PGD-100. However, when confronted with CW$_2$, the accuracies of many models drop seriously. With the help of FSU, the accuracy improves to a great extent. E.g., the robust accuracies of AT-AWP and AT+RiFT are improved respectively by $19.82\%$ and $34.28\%$ under CW attack for dataset CIFAR10.
\item The robustness improvement under AA is also remarkable in some cases. E.g., the robust accuracies of AT+RiFT and DKL are improved respectively by $9.34\%$ and $17.13\%$ under AA for dataset SVHN. It is well-known that AA sequentially performs four different attacks, i.e., Auto-PGD$_{\rm CE}$ (white-box), Auto-PGD$_{\rm DLR}$ (white-box), FAB (white-box) and square attack (black-box), which is a very powerful tool for robustness evaluation. The universal improvement of the models under AA demonstrates the comprehensive effectiveness of the proposed FSU module.
\item The last row of Table~\ref{tab:powerful} counts the number of improvements under each attack. It is obvious that the robustness of most models has been improved by FSU, no matter under which attack. Under the ten attacks on the four benchmark datasets (a total of 230 comparisons), FSU achieved robustness improvement in 197 comparisons.
\end{itemize}

\begin{figure*}[!t]
\centering
\subfloat[\scriptsize Mean Shift (Baseline-Train)]{\includegraphics[width=0.245\linewidth]{non-cifar10-pgd-train-mean.pdf}}\
\subfloat[\scriptsize Mean Shift (Baseline-Test)]{\includegraphics[width=0.245\linewidth]{non-cifar10-pgd-test-mean.pdf}}\
\subfloat[\scriptsize Std Shift (Baseline-Train)]{\includegraphics[width=0.245\linewidth]{non-cifar10-pgd-train-std.pdf}}\
\subfloat[\scriptsize Std Shift (Baseline-Test)]{\includegraphics[width=0.245\linewidth]{non-cifar10-pgd-test-std.pdf}}\\
\subfloat[\scriptsize Mean Shift (FSU-Train)]{\includegraphics[width=0.245\linewidth]{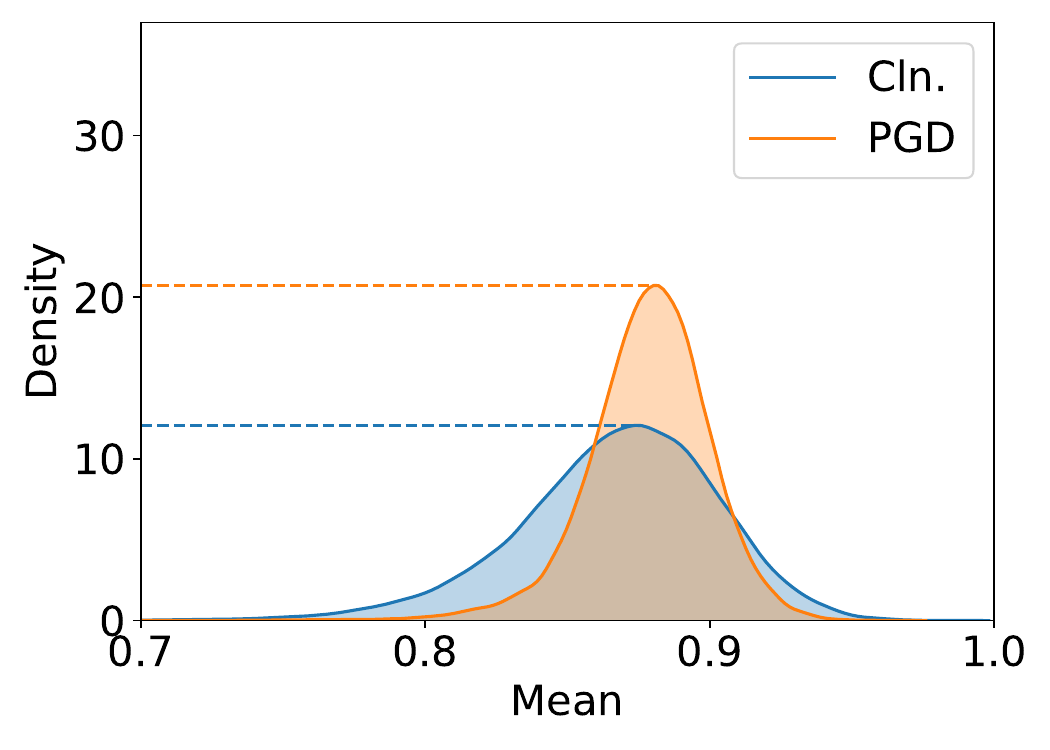}}\
\subfloat[\scriptsize Mean Shift (FSU-Test)]{\includegraphics[width=0.245\linewidth]{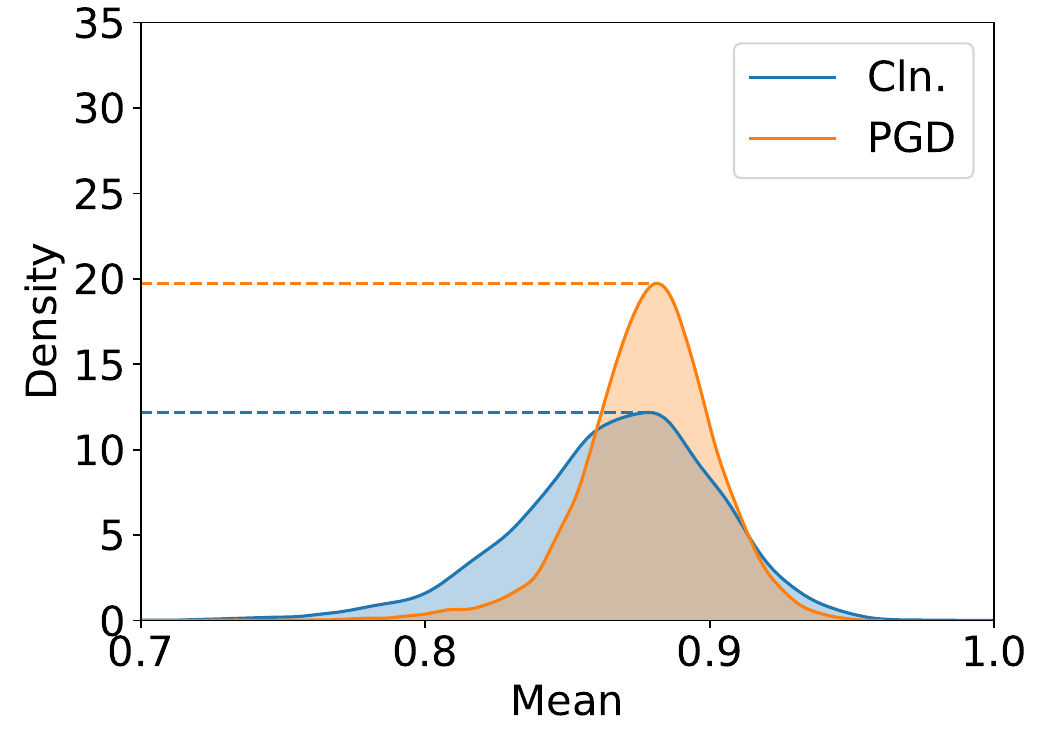}}\
\subfloat[\scriptsize Std Shift (FSU-Train)]{\includegraphics[width=0.245\linewidth]{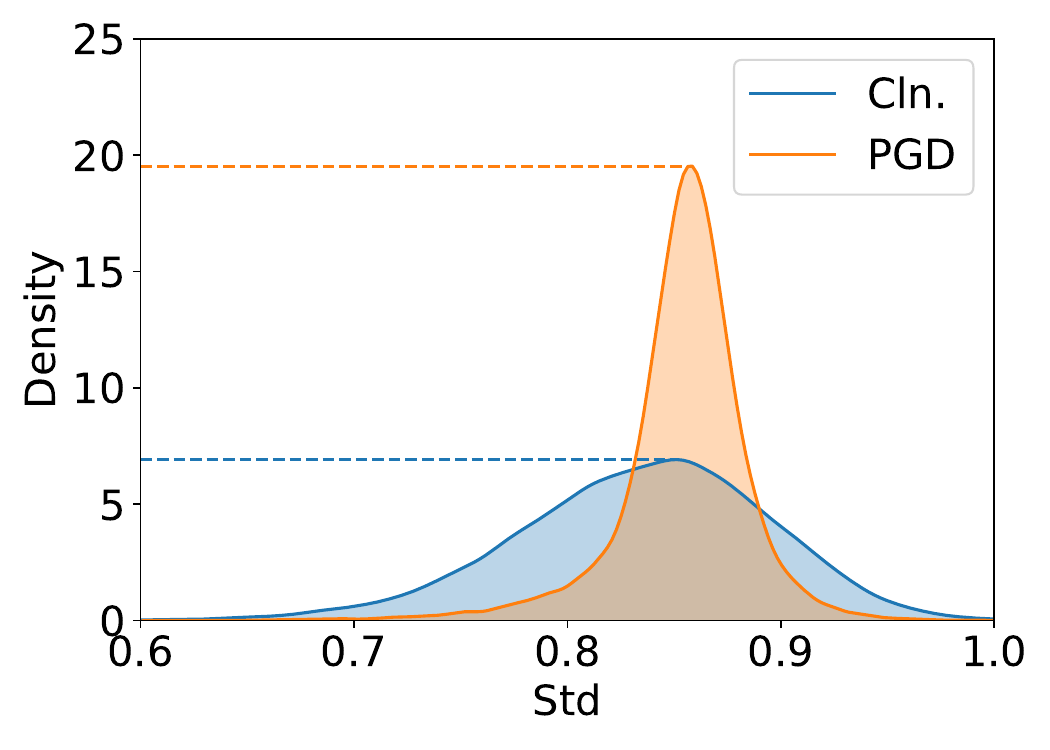}}\
\subfloat[\scriptsize Std Shift (FSU-Test)]{\includegraphics[width=0.245\linewidth]{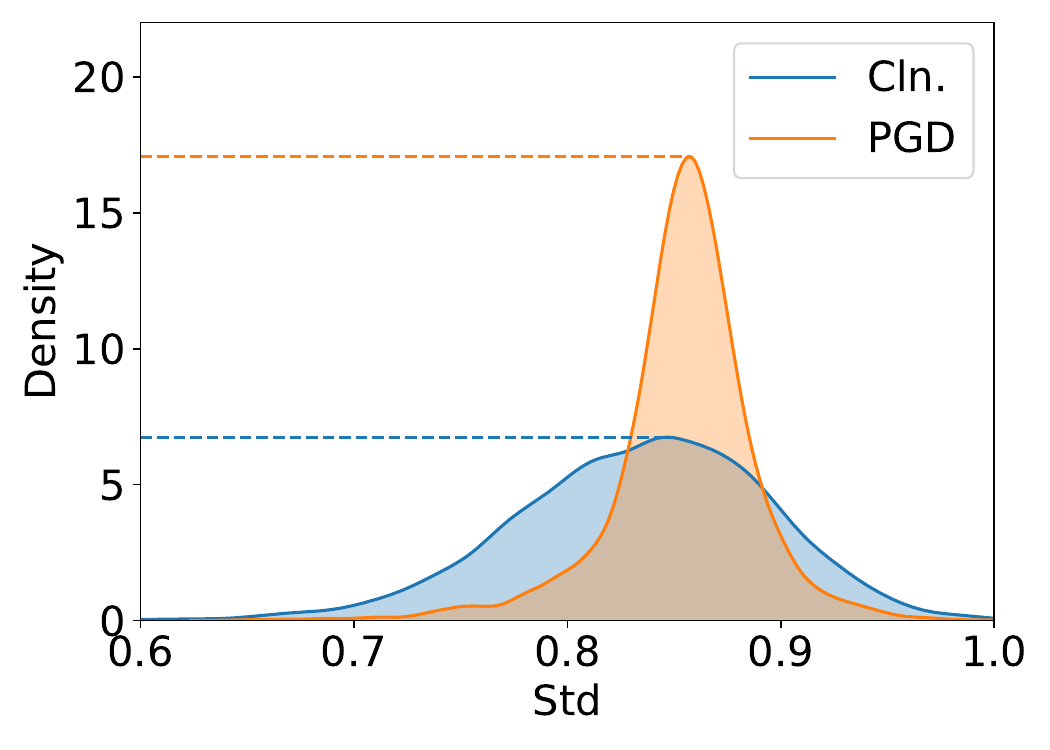}}
\caption{Calibration for distribution shifts of feature statistics in dataset CIFAR10 (average of 512 channels in latent space, ResNet-18 model). ``Cln." and ``PGD" represent clean examples and examples perturbed by PGD-10 attack, respectively. ``Baseline" and ``FSU" mean the naturally trained models without and with the FSU module, respectively. ``Train" and ``Test" represent training data and testing data, respectively.}
\vspace{-0.2cm}
\label{fig.featureDisCifar10}
\end{figure*}

\begin{figure}[t]
\centering
\subfloat[\scriptsize{Example 1: Baseline-NT}]{\includegraphics[width=0.45\linewidth]{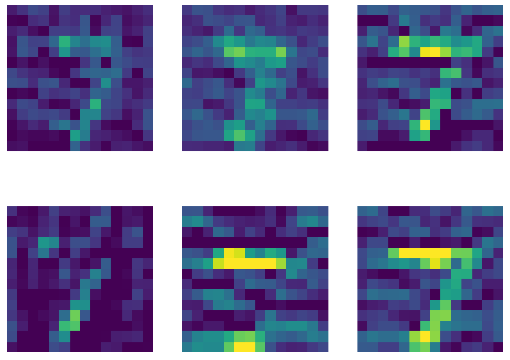}}\quad\quad
\subfloat[\scriptsize{Example 1: FSU-NT}]{\includegraphics[width=0.45\linewidth]{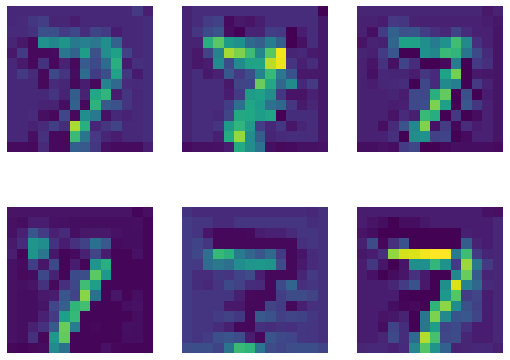}}\\ 
\subfloat[\scriptsize{Example 2: Baseline-NT}]{\includegraphics[width=0.45\linewidth]{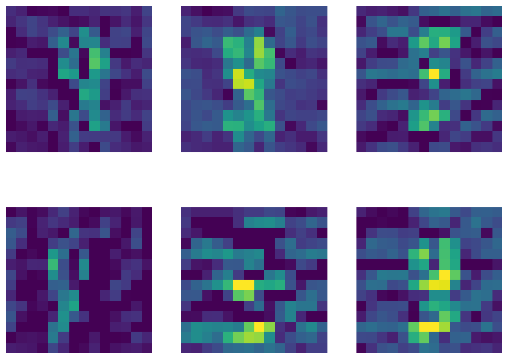}}\quad\quad
\subfloat[\scriptsize{Example 2: FSU-NT}]{\includegraphics[width=0.45\linewidth]{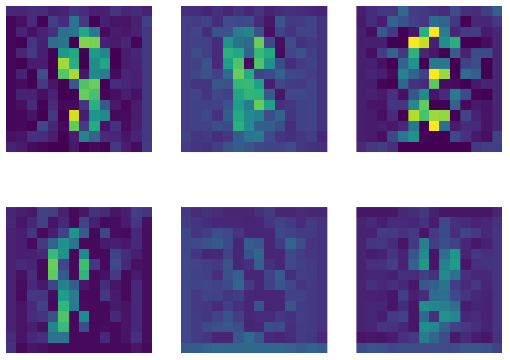}}
\caption{Feature maps of six channels for randomly selected testing examples in MNIST perturbed by FGSM (with LeNet-5).}
\label{fig.featureMapMNISTFGSM0}
\vspace{-0.2cm}
\end{figure}

\begin{figure}[t]
\centering
\subfloat[\scriptsize Cln.]{\includegraphics[width=0.15\linewidth]{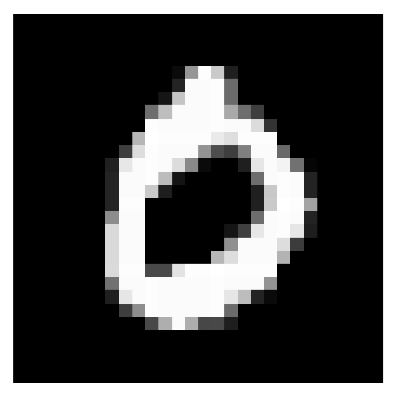}}\quad\ 
\subfloat[\scriptsize Ep 10]{\includegraphics[width=0.15\linewidth]{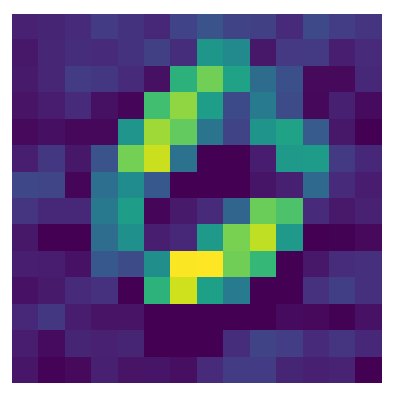}}\quad\ 
\subfloat[\scriptsize Ep 20]{\includegraphics[width=0.15\linewidth]{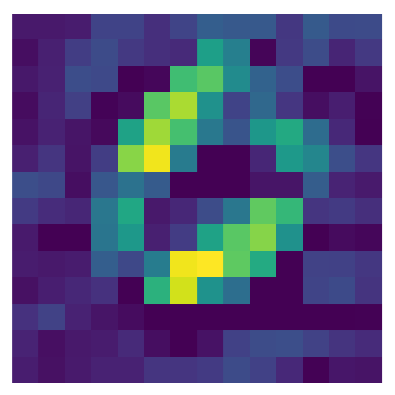}}\quad\ 
\subfloat[\scriptsize Ep 30]{\includegraphics[width=0.15\linewidth]{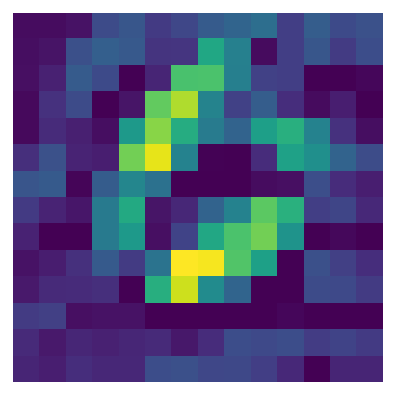}}\quad\ 
\subfloat[\scriptsize Ep 40]{\includegraphics[width=0.15\linewidth]{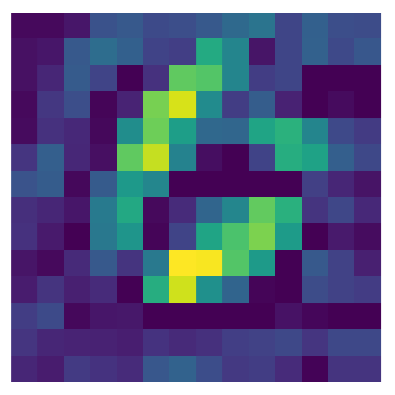}}\\
\vspace{-0.3cm}
\subfloat[\scriptsize Cln.]{\includegraphics[width=0.15\linewidth]{sample0.png}}\quad\ 
\subfloat[\scriptsize Ep 10]{\includegraphics[width=0.15\linewidth]{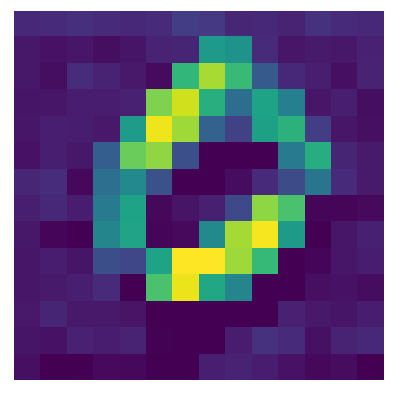}}\quad\ 
\subfloat[\scriptsize Ep 20]{\includegraphics[width=0.15\linewidth]{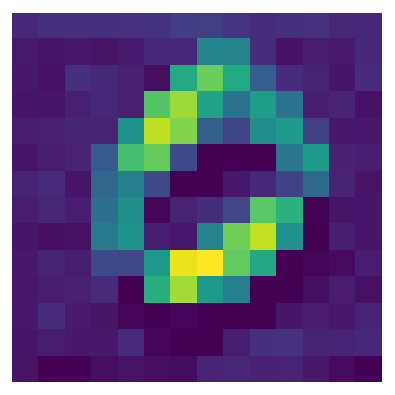}}\quad\ 
\subfloat[\scriptsize Ep 30]{\includegraphics[width=0.15\linewidth]{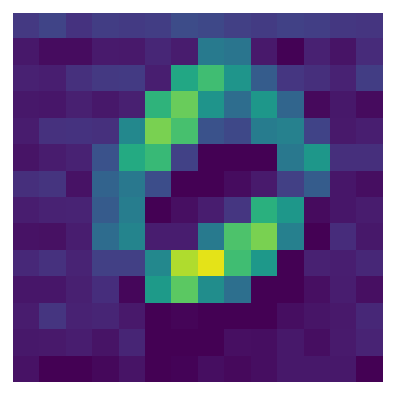}}\quad\ 
\subfloat[\scriptsize Ep 40]{\includegraphics[width=0.15\linewidth]{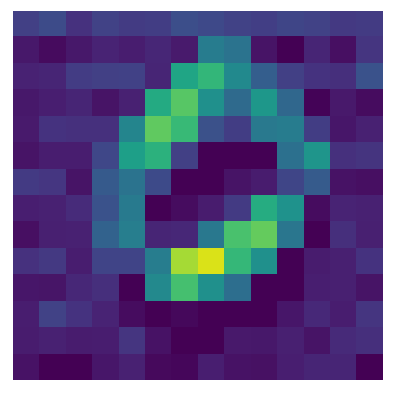}}\\
\vspace{-0.3cm}
\subfloat[\scriptsize Cln.]{\includegraphics[width=0.15\linewidth]{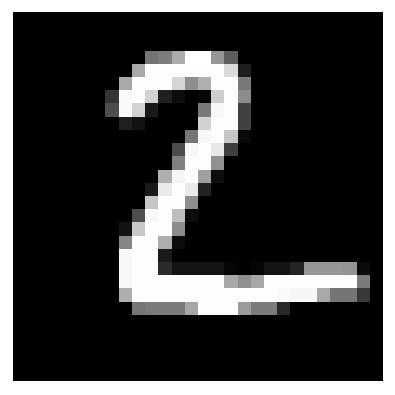}}\quad\ 
\subfloat[\scriptsize Ep 10]{\includegraphics[width=0.15\linewidth]{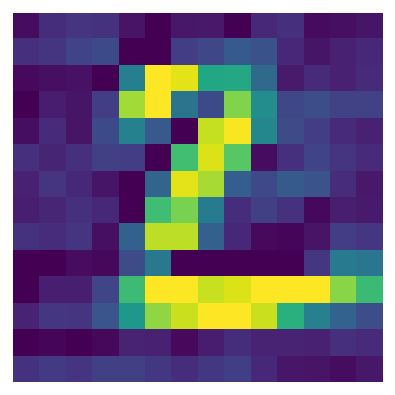}}\quad\ 
\subfloat[\scriptsize Ep 20]{\includegraphics[width=0.15\linewidth]{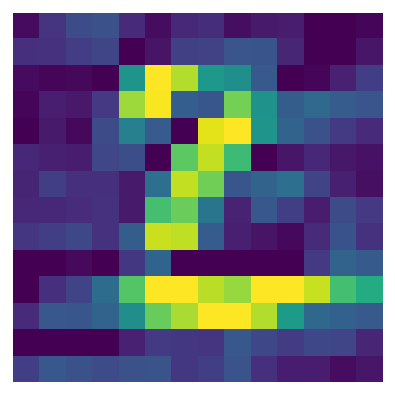}}\quad\ 
\subfloat[\scriptsize Ep 30]{\includegraphics[width=0.15\linewidth]{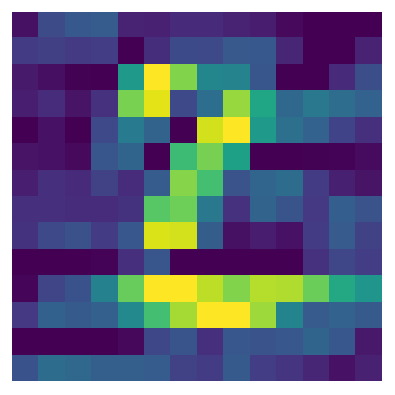}}\quad\ 
\subfloat[\scriptsize Ep 40]{\includegraphics[width=0.15\linewidth]{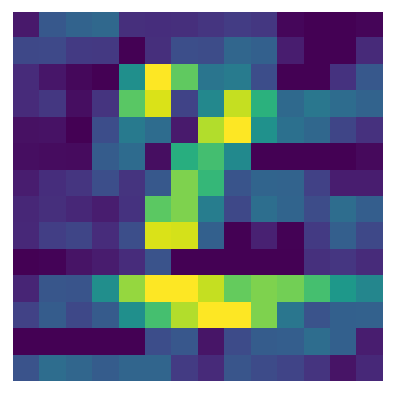}}\\
\vspace{-0.3cm}
\subfloat[\scriptsize Cln.]{\includegraphics[width=0.15\linewidth]{sample2.png}}\quad\ 
\subfloat[\scriptsize Ep 10]{\includegraphics[width=0.15\linewidth]{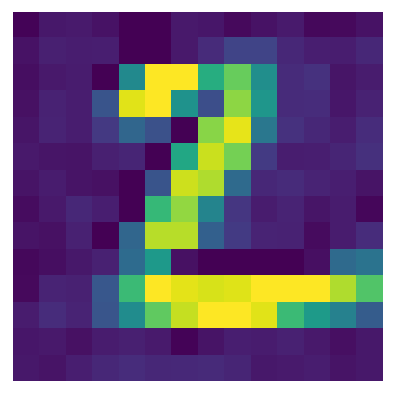}}\quad\ 
\subfloat[\scriptsize Ep 20]{\includegraphics[width=0.15\linewidth]{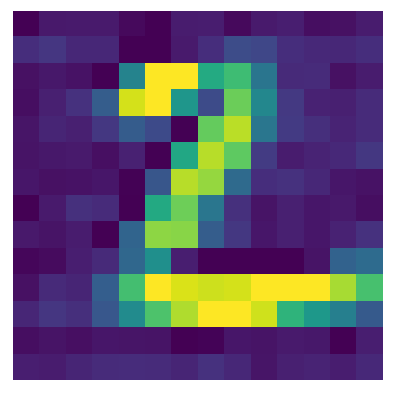}}\quad\ 
\subfloat[\scriptsize Ep 30]{\includegraphics[width=0.15\linewidth]{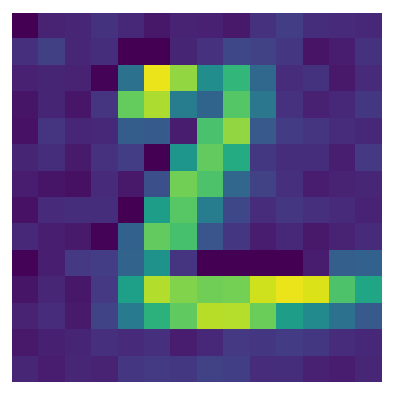}}\quad\ 
\subfloat[\scriptsize Ep 40]{\includegraphics[width=0.15\linewidth]{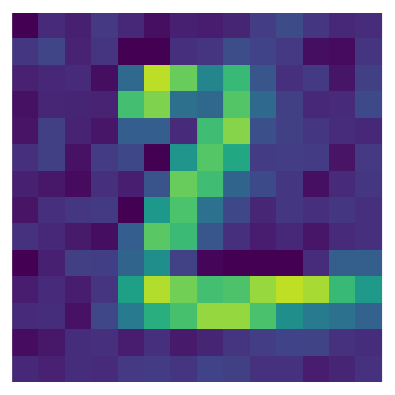}}
\caption{Feature maps of a randomly selected channel for randomly selected testing examples in MNIST during PGD-40 attack (with LeNet-5). (a)$\sim$(e) Example 1: Baseline-NT. (f)$\sim$(j) Example 1: FSU-NT. (k)$\sim$(o) Example 2: Baseline-NT. (p)$\sim$(t) Example 2: FSU-NT. (Ep  means Epoch.)}
\label{fig.featureMapMNISTPGD0}
\end{figure}

\begin{figure}[t]
\centering
\subfloat[\scriptsize{Example 1: Baseline-NT}]{\includegraphics[width=0.45\linewidth]{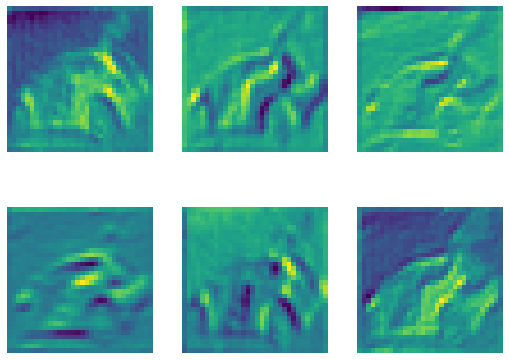}}\quad\quad
\subfloat[\scriptsize{Example 1: FSU-NT}]{\includegraphics[width=0.45\linewidth]{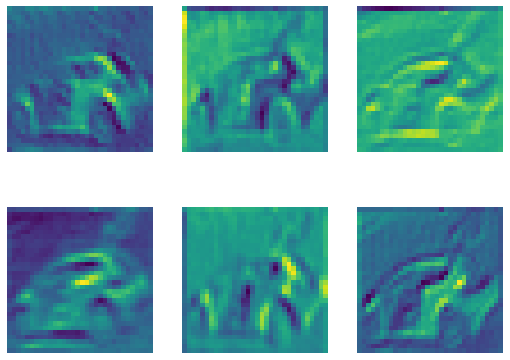}}\\
\subfloat[\scriptsize{Example 2: Baseline-NT}]{\includegraphics[width=0.45\linewidth]{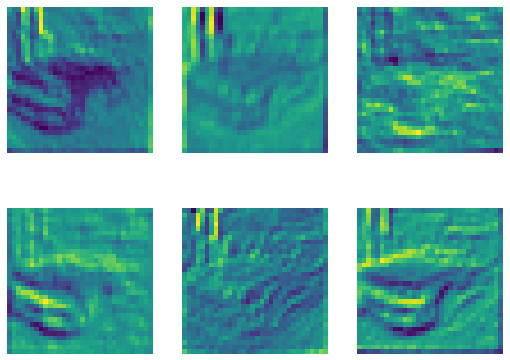}}\quad\quad
\subfloat[\scriptsize{Example 2: FSU-NT}]{\includegraphics[width=0.45\linewidth]{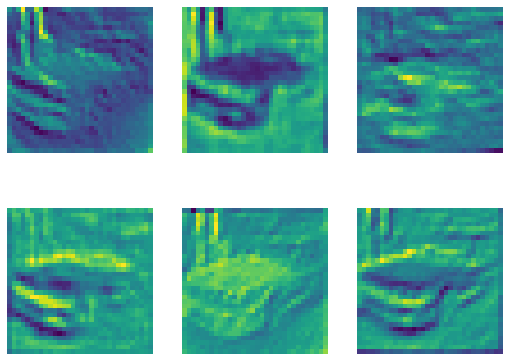}}
\caption{Feature maps of six randomly selected channels for randomly selected testing examples in CIFAR10 perturbed by FGSM (with ResNet-18).}
\label{fig.featureMapCIFAR10FGSM0}
\vspace{-0.2cm}
\end{figure}

\begin{figure}[t]
\centering
\subfloat[\tiny Cln.]{\includegraphics[width=0.15\linewidth]{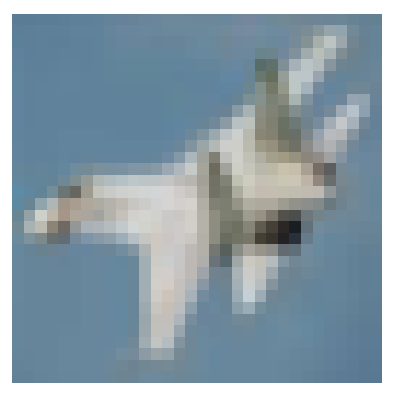}}\quad\
\subfloat[\tiny E10 L1]{\includegraphics[width=0.15\linewidth]{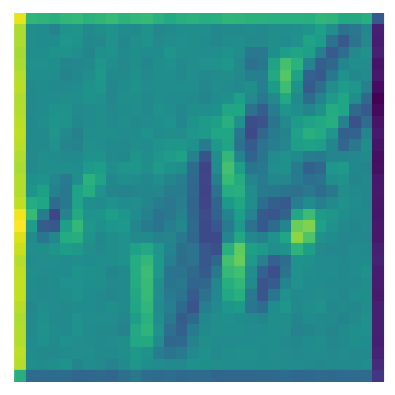}}\quad\ 
\subfloat[\tiny E10 L2]{\includegraphics[width=0.15\linewidth]{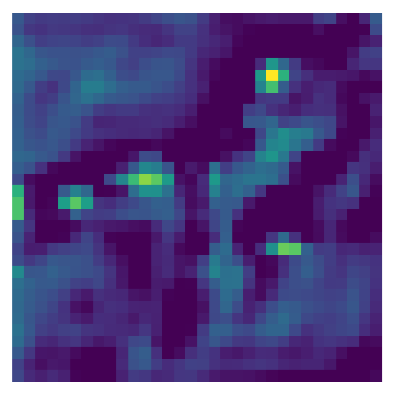}}\quad\ 
\subfloat[\tiny E10 L3]{\includegraphics[width=0.15\linewidth]{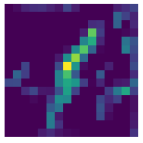}}\quad\ 
\subfloat[\tiny E10 L4]{\includegraphics[width=0.15\linewidth]{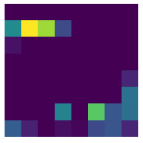}}\\
\vspace{-0.3cm}
\subfloat[\tiny Cln.]{\includegraphics[width=0.15\linewidth]{sample_plane.png}}\quad\ 
\subfloat[\tiny E10 L1]{\includegraphics[width=0.15\linewidth]{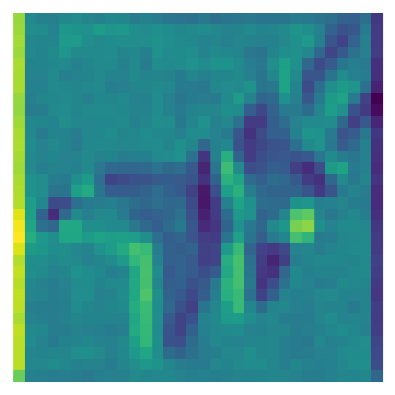}}\quad\ 
\subfloat[\tiny E10 L2]{\includegraphics[width=0.15\linewidth]{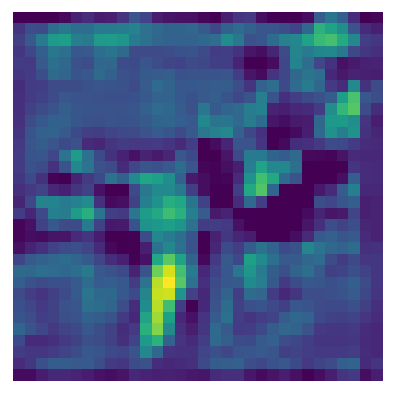}}\quad\ 
\subfloat[\tiny E10 L3]{\includegraphics[width=0.15\linewidth]{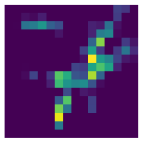}}\quad\ 
\subfloat[\tiny E10 L4]{\includegraphics[width=0.15\linewidth]{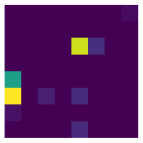}}\\
\vspace{-0.3cm}
\subfloat[\tiny Cln.]{\includegraphics[width=0.15\linewidth]{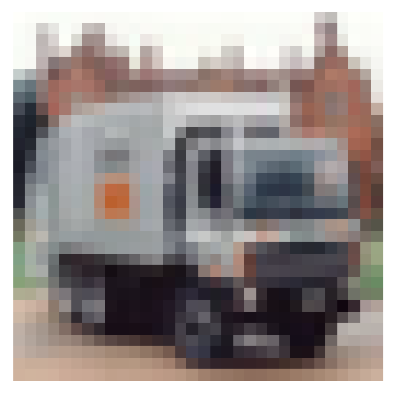}}\quad\ 
\subfloat[\tiny E10 L1]{\includegraphics[width=0.15\linewidth]{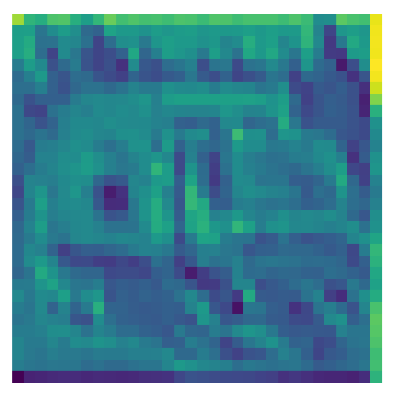}}\quad\ 
\subfloat[\tiny E10 L2]{\includegraphics[width=0.15\linewidth]{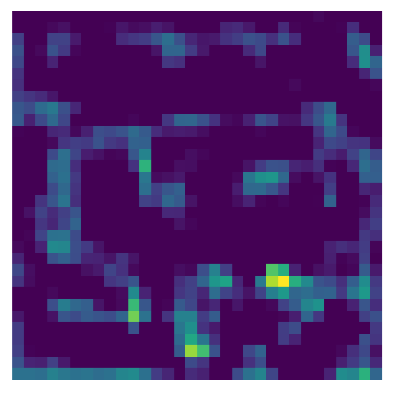}}\quad\ 
\subfloat[\tiny E10 L3]{\includegraphics[width=0.15\linewidth]{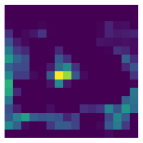}}\quad\ 
\subfloat[\tiny E10 L4]{\includegraphics[width=0.15\linewidth]{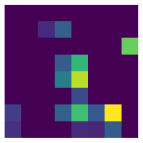}}\\
\vspace{-0.3cm}
\subfloat[\tiny Cln.]{\includegraphics[width=0.15\linewidth]{sample_truck.png}}\quad\ 
\subfloat[\tiny E10 L1]{\includegraphics[width=0.15\linewidth]{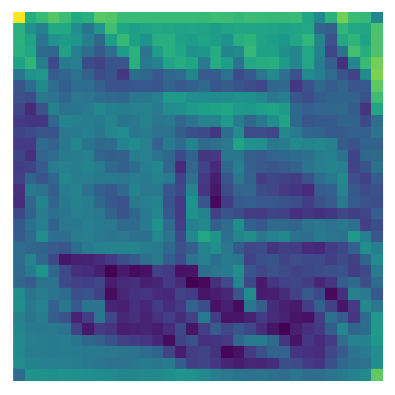}}\quad\ 
\subfloat[\tiny E10 L2]{\includegraphics[width=0.15\linewidth]{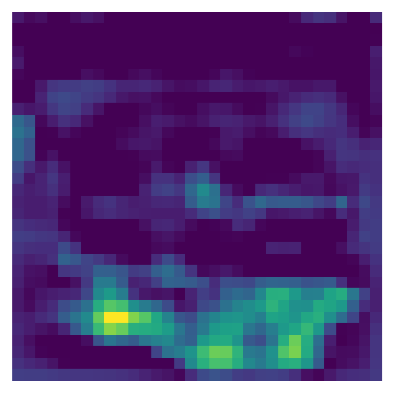}}\quad\ 
\subfloat[\tiny E10 L3]{\includegraphics[width=0.15\linewidth]{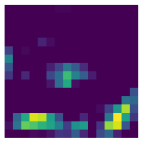}}\quad\ 
\subfloat[\tiny E10 L4]{\includegraphics[width=0.15\linewidth]{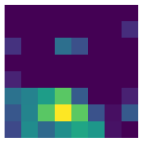}}
\caption{Feature maps of a randomly selected channel for randomly selected testing examples in CIFAR10 by PGD-10 attack (with ResNet-18). (a)$\sim$(e) Example 1: Baseline-NT. (f)$\sim$(j) Example 1: FSU-NT. (k)$\sim$(o) Example 2: Baseline-NT. (p)$\sim$(t) Example 2: FSU-NT. (E and L mean epoch and location. L1: after the first convolution layer; L2: after the first convolution block; L3: after the second convolution block; L4: after the third convolution block.)}
\label{fig.featureMapCIFARPGD0}
\end{figure}

\textbf{Result Statistics:} Table~\ref{tab:statistics} further reports some statistics of the results in Table~\ref{tab:powerful}, including the change of clean accuracy, the maximum / average improvement of robust accuracy, and the count of robustness improvement of the fine-tuned models compared with the original models. It can be seen that on CIFAR10, SVHN and CIFAR100, the clean accuracy has been sacrificed a little bit in some cases, but the robust accuracy has been improved to a greater extent. As for MNIST, it is aforementioned that the accuracies of the original models are already high, without too much room for improvement. Despite this, FSU achieved positive changes on all clean accuracy, maximum and average robust accuracies.

\textbf{Sensitivity Analysis:} Our primary results show that Algorithm~\ref{alg.FSU} is stable with respect to $\beta$ (i.e., the noise intensity of feature Std), while is sensitive to $\alpha$ (i.e., the noise intensity of feature mean). Thus, we fix $\beta$ as $1$ and investigate its performance with $\alpha=\{0.1,0.3,0.5,0.7,0.9\}$. Fig.~\ref{fig.sensitivityCifar10} shows the robust accuracies of the fine-tuned models by Algorithm~\ref{alg.FSU} with different values of $\alpha$ on dataset CIFAR10 under FGSM, PGD-100, AA, and CW$_2$ attacks. It can be seen that different values of $\alpha$ may correspond to different performances. Despite this, the robustness is improved in most cases compared with the original model. In addition, for different attacks, the best accuracy might be obtained with different parameter values, indicating that it is possible to achieve a better defense effect against a certain attack by further tuning the parameters. More sensitivity figures on the other datasets are provided in Appendix~\ref{append.sensitivity}.

\textbf{Ablation Study:}  We also make an ablation study on the FSU module. By deleting steps $7\sim 16$ of Algorithm~\ref{alg.FSU}, we can fine-tune the models only based on $\mathbf{x}$ and $\mathbf{x}'$, denoted as ``Fit w.o. FSU". The ablation results are reported in Fig.~\ref{fig.ablationAT}, where AT-AWP is used as the baseline model. It can be seen that compared with the baseline, pure fine-tuning w.o. FSU could either improve or degrade the robustness. This is due to the fact that the best model is selected via PGD-10, which cannot represent good robustness under other attacks. By incorporating FSU, the robustness is consistently improved under almost all of the attacks, and the improvement is more obvious than that of ``Fit w.o. FSU" in most cases.

\subsection{Calibration for Distribution Shifts}\label{subsec.calibration}

Fig.~\ref{fig.featureDisCifar10} demonstrates the distributions of feature mean and Std, i.e., $(\mu,\sigma)$, of both clean and perturbed examples in dataset CIFAR10 without the FSU module (denoted as Baseline) and with the FSU module (denoted as FSU). Note that to effectively observe the calibration, the naturally trained models in Subsection~\ref{subsec.mild} are used here. Obviously, the distribution shifts of both feature mean and Std caused by the baseline are more obvious than those caused by FSU, indicating that the FSU module helped calibrate the distribution shifts of feature statistics to a great extent. More calibration figures on the other datasets are provided in Appendix~\ref{append.calibration}.

\subsection{Visualization of Reconstructed Features}\label{subsec.visualization}

To visualize the feature reconstruction effects of the FSU module, Figs.~\ref{fig.featureMapMNISTFGSM0}\&\ref{fig.featureMapCIFAR10FGSM0} and Figs.~\ref{fig.featureMapMNISTPGD0}\&\ref{fig.featureMapCIFARPGD0} demonstrate the multi-channel feature maps of randomly selected testing examples perturbed by FGSM and PGD, respectively. Note that the reconstruction effects should be observed during the attacking procedure, thus the models used here are also the naturally trained ones in Subsection~\ref{subsec.mild}. Figs.~\ref{fig.featureMapMNISTFGSM0} and \ref{fig.featureMapMNISTPGD0} are for MNIST; Figs.~\ref{fig.featureMapCIFAR10FGSM0} and \ref{fig.featureMapCIFARPGD0} are for CIFAR10. It can be seen that the FSU module does help the perturbed examples recover some important feature points. Compared to the feature maps obtained from the baseline model, the ones obtained from the FSU model have more explicit semantic information. In other words, the FSU module mitigates the influence of adversarial perturbations and weakens the ability of the attacks to deceive models. More visualization analyses on the feature maps of randomly selected testing examples perturbed by PGD and FGSM are provided in Appendix~\ref{append.visual}.

\section{Conclusions}\label{sec.conclusion}
We theoretically and empirically discovered a universal phenomenon that adversarial attacks tend to shift the distributions of feature statistics of examples. Following this discovery, we enhanced the adversarial robustness of DNNs by reconstructing feature maps based on channel-wise statistical feature mean and standard deviation of a batch of examples. Gaussian noise is used to model the uncertainty of the statistical feature mean and standard deviation, which mitigates the influence of the perturbations from attackers, and helps the reconstructed feature maps recover some domain characteristics for classification. The proposed module does not rely on adversarial information; it can be universally applied to natural training, adversarial training, attacking, predicting, and model fine-tuning for robustness enhancement, which achieves promising improvements against various attacks.

\bibliographystyle{IEEEtran}
\bibliography{Bibliography}

\begin{thebibliography}{10}
\providecommand{\url}[1]{#1}
\csname url@samestyle\endcsname
\providecommand{\newblock}{\relax}
\providecommand{\bibinfo}[2]{#2}
\providecommand{\BIBentrySTDinterwordspacing}{\spaceskip=0pt\relax}
\providecommand{\BIBentryALTinterwordstretchfactor}{4}
\providecommand{\BIBentryALTinterwordspacing}{\spaceskip=\fontdimen2\font plus
\BIBentryALTinterwordstretchfactor\fontdimen3\font minus
  \fontdimen4\font\relax}
\providecommand{\BIBforeignlanguage}[2]{{%
\expandafter\ifx\csname l@#1\endcsname\relax
\typeout{** WARNING: IEEEtran.bst: No hyphenation pattern has been}%
\typeout{** loaded for the language `#1'. Using the pattern for}%
\typeout{** the default language instead.}%
\else
\language=\csname l@#1\endcsname
\fi
#2}}
\providecommand{\BIBdecl}{\relax}
\BIBdecl

\bibitem{adv_define}
C.~Szegedy, W.~Zaremba, I.~Sutskever, J.~Bruna, D.~Erhan, I.~Goodfellow, and
  R.~Fergus, ``Intriguing properties of neural networks,'' in
  \emph{International Conference on Learning Representations}, 2014.

\bibitem{pgd}
A.~Madry, A.~Makelov, L.~Schmidt, D.~Tsipras, and A.~Vladu, ``Towards deep
  learning models resistant to adversarial attacks,'' in \emph{International
  Conference on Learning Representations}, 2018.

\bibitem{advtrain1}
A.~Shafahi, M.~Najibi, M.~A. Ghiasi, Z.~Xu, J.~Dickerson, C.~Studer, L.~S.
  Davis, G.~Taylor, and T.~Goldstein, ``Adversarial training for free!''
  \emph{Advances in Neural Information Processing Systems}, vol.~32, 2019.

\bibitem{trades}
H.~Zhang, Y.~Yu, J.~Jiao, E.~Xing, L.~El~Ghaoui, and M.~Jordan, ``Theoretically
  principled trade-off between robustness and accuracy,'' in
  \emph{International Conference on Machine Learning}, 2019, pp. 7472--7482.

\bibitem{mart}
Y.~Wang, D.~Zou, J.~Yi, J.~Bailey, X.~Ma, and Q.~Gu, ``Improving adversarial
  robustness requires revisiting misclassified examples,'' in
  \emph{International Conference on Learning Representations}, 2020.

\bibitem{yu2022Understanding}
C.~Yu, B.~Han, L.~Shen, J.~Yu, C.~Gong, M.~Gong, and T.~Liu, ``Understanding
  robust overfitting of adversarial training and beyond,'' in
  \emph{International Conference on Machine Learning}, 2022.

\bibitem{Picot2023adversarial}
M.~Picot, F.~Messina, M.~Boudiaf, F.~Labeau, I.~Ben~Ayed, and P.~Piantanida,
  ``Adversarial robustness via fisher-rao regularization,'' \emph{IEEE
  Transactions on Pattern Analysis and Machine Intelligence}, vol.~45, no.~3,
  pp. 2698--2710, 2023.

\bibitem{Zhu2023Improving}
K.~Zhu, X.~Hu, J.~Wang, X.~Xie, and G.~Yang, ``Improving generalization of
  adversarial training via robust critical fine-tuning,'' in
  \emph{International Conference on Computer Vision}, 2023.

\bibitem{cui2024decoupled}
J.~Cui, Z.~Tian, Z.~Zhong, X.~Qi, B.~Yu, and H.~Zhang, ``Decoupled
  kullback-leibler divergence loss,'' in \emph{Advances in Neural Information
  Processing Systems (NeurIPS)}, vol.~37, 2024, pp. 74\,461--74\,486.

\bibitem{Zhao2024Mitigating}
S.~Zhao, X.~Wang, and X.~Wei, ``Mitigating accuracy-robustness trade-off via
  balanced multi-teacher adversarial distillation,'' \emph{IEEE Transactions on
  Pattern Analysis and Machine Intelligence}, vol.~46, no.~12, pp. 9338--9352,
  2024.

\bibitem{Wei2024Revisiting}
X.~Wei, S.~Zhao, and B.~Li, ``Revisiting the trade-off between accuracy and
  robustness via weight distribution of filters,'' \emph{IEEE Transactions on
  Pattern Analysis and Machine Intelligence}, vol.~46, no.~12, pp. 8870--8882,
  2024.

\bibitem{Cohen2019Certified}
J.~Cohen, E.~Rosenfeld, and Z.~Kolter, ``Certified adversarial robustness via
  randomized smoothing,'' in \emph{International Conference on Machine
  Learning}, 2019, pp. 1310--1320.

\bibitem{random_xie}
C.~Xie, J.~Wang, Z.~Zhang, Z.~Ren, and A.~Yuille, ``Mitigating adversarial
  effects through randomization,'' in \emph{International Conference on
  Learning Representations}, 2018.

\bibitem{add_noise_to_input}
B.~Li, C.~Chen, W.~Wang, and L.~Carin, ``Certified adversarial robustness with
  additive noise,'' \emph{Advances in Neural Information Processing Systems},
  vol.~32, 2019.

\bibitem{rse_liu}
X.~Liu, M.~Cheng, H.~Zhang, and C.-J. Hsieh, ``Towards robust neural networks
  via random self-ensemble,'' in \emph{European Conference on Computer Vision},
  2018, pp. 369--385.

\bibitem{wca}
P.~Eustratiadis, H.~Gouk, D.~Li, and T.~Hospedales, ``Weight-covariance
  alignment for adversarially robust neural networks,'' in \emph{International
  Conference on Machine Learning}, 2021, pp. 3047--3056.

\bibitem{yang2023simple}
H.~Yang, M.~Wang, Z.~Yu, and Y.~Zhou, ``A simple stochastic neural network for
  improving adversarial robustness,'' in \emph{IEEE International Conference on
  Multimedia and Expo}, 2023, pp. 2297--2302.

\bibitem{yang2023weight}
------, ``Weight-based regularization for improving robustness in image
  classification,'' in \emph{IEEE International Conference on Multimedia and
  Expo}, 2023, pp. 1775--1780.

\bibitem{Alemi2017Deep}
A.~A. Alemi, I.~Fischer, J.~V. Dillon, and K.~Murphy, ``Deep variational
  information bottleneck,'' in \emph{International Conference on Learning
  Representations}, 2017.

\bibitem{Wang2024Adversarially}
R.~Wang, H.~Ke, M.~Hu, and W.~Wu, ``Adversarially robust neural networks with
  feature uncertainty learning and label embedding,'' \emph{Neural Networks},
  vol. 172, no. 106087, 2024.

\bibitem{advbnn}
X.~Liu, Y.~Li, C.~Wu, and C.-J. Hsieh, ``Adv-bnn: Improved adversarial defense
  through robust bayesian neural network,'' in \emph{International Conference
  on Learning Representations}, 2018.

\bibitem{Carbone2020Robustness}
G.~Carbone, M.~Wicker, L.~Laurenti, A.~Patane', L.~Bortolussi, and
  G.~Sanguinetti, ``Robustness of bayesian neural networks to gradient-based
  attacks,'' in \emph{Advances in Neural Information Processing Systems}, 2020.

\bibitem{pni_he}
Z.~He, A.~S. Rakin, and D.~Fan, ``Parametric noise injection: Trainable
  randomness to improve deep neural network robustness against adversarial
  attack,'' in \emph{IEEE/CVF Conference on Computer Vision and Pattern
  Recognition}, 2019, pp. 588--597.

\bibitem{l2p}
A.~Jeddi, M.~J. Shafiee, M.~Karg, C.~Scharfenberger, and A.~Wong,
  ``Learn2perturb: an end-to-end feature perturbation learning to improve
  adversarial robustness,'' in \emph{IEEE/CVF Conference on Computer Vision and
  Pattern Recognition}, 2020, pp. 1241--1250.

\bibitem{wu2020Adversarial}
D.~Wu, S.~T. Xia, and Y.~Wang, ``Adversarial weight perturbation helps robust
  generalization,'' in \emph{Advances in Neural Information Processing
  Systems}, 2020.

\bibitem{cw}
N.~Carlini and D.~Wagner, ``Towards evaluating the robustness of neural
  networks,'' in \emph{IEEE Symposium on Security and Privacy}, 2017, pp.
  39--57.

\bibitem{Croce2020Reliable}
F.~Croce and M.~Hein, ``Reliable evaluation of adversarial robustness with an
  ensemble of diverse parameter-free attacks,'' in \emph{International
  Conference on Machine Learning}, 2020.

\bibitem{dsu_li}
X.~Li, Y.~Dai, Y.~Ge, J.~Liu, Y.~Shan, and L.-Y. Duan, ``Uncertainty modeling
  for out-of-distribution generalization,'' in \emph{International Conference
  on Learning Representations}, 2022.

\bibitem{goodfellow2014explaining}
I.~J. Goodfellow, J.~Shlens, and C.~Szegedy, ``Explaining and harnessing
  adversarial examples,'' in \emph{3rd International Conference on Learning
  Representations (ICLR)}, 2015.

\bibitem{bim}
A.~Kurakin, I.~J. Goodfellow, and S.~Bengio, ``Adversarial machine learning at
  scale,'' in \emph{International Conference on Learning Representations},
  2017.

\bibitem{mim}
Y.~Dong, F.~Liao, T.~Pang, H.~Su, J.~Zhu, X.~Hu, and J.~Li, ``Boosting
  adversarial attacks with momentum,'' in \emph{IEEE/CVF Conference on Computer
  Vision and Pattern Recognition}, 2018, pp. 9185--9193.

\bibitem{Narodytska2016simple}
\BIBentryALTinterwordspacing
N.~Narodytska and S.~P. Kasiviswanathan, ``Simple black-box adversarial
  perturbations for deep networks,'' \emph{arXiv:1612.06299}, 2016. [Online].
  Available: \url{https://doi.org/10.48550/arXiv.1612.06299}
\BIBentrySTDinterwordspacing

\bibitem{onepixel}
J.~Su, D.~V. Vargas, and K.~Sakurai, ``One pixel attack for fooling deep neural
  networks,'' \emph{IEEE Transactions on Evolutionary Computation}, vol.~23,
  no.~5, pp. 828--841, 2019.

\bibitem{Williams2023black}
P.~N. Williams and K.~Li, ``Black-box sparse adversarial attack via
  multi-objective optimisation,'' in \emph{IEEE/CVF Conference on Computer
  Vision and Pattern Recognition}, 2023, pp. 12\,291--12\,301.

\bibitem{Croce2020Minimally}
F.~Croce and M.~Hein, ``Minimally distorted adversarial examples with a fast
  adaptive boundary,'' in \emph{International Conference on Machine Learning},
  2020.

\bibitem{Andriushchenko2020Square}
M.~Andriushchenko, F.~Croce, N.~Flammarion, and et~al., ``Square attack: {A}
  query-efficient black-box adversarial attack via random search,'' in
  \emph{International Conference on Computer Vision}, 2020, pp. 484--501.

\bibitem{Angioni2025Robustness}
\BIBentryALTinterwordspacing
D.~Angioni, L.~Demetrio, M.~Pintor, L.~Oneto, D.~Anguita, B.~Biggio, and
  F.~Roli, ``Robustness-congruent adversarial training for secure machine
  learning model updates,'' \emph{IEEE Transactions on Pattern Analysis and
  Machine Intelligence}, vol. in press, DOI: 10.1109/TPAMI.2025.3573237, 2025.
  [Online]. Available: \url{https://ieeexplore.ieee.org/document/11014530}
\BIBentrySTDinterwordspacing

\bibitem{Nix1994Estimating}
D.~A. Nix and A.~S. Weigend, ``Estimating the mean and variance of the target
  probability distribution,'' in \emph{IEEE International Conference on Neural
  Networks}, 1994, pp. 55--60.

\bibitem{Lee2023GradDiv}
S.~Lee, H.~Kim, and J.~Lee, ``{GradDiv}: {A}dversarial robustness of randomized
  neural networks via gradient diversity regularization,'' \emph{IEEE
  Transactions on Pattern Analysis and Machine Intelligence}, vol.~45, no.~2,
  pp. 2645--2651, 2023.

\bibitem{lecun1998gradient}
Y.~LeCun, L.~Bottou, Y.~Bengio, and P.~Haffner, ``Gradient-based learning
  applied to document recognition,'' \emph{Proceedings of the IEEE}, vol.~86,
  no.~11, pp. 2278--2324, 1998.

\bibitem{krizhevsky2009learning}
A.~Krizhevsky, G.~Hinton \emph{et~al.}, ``Learning multiple layers of features
  from tiny images,'' \emph{Handbook of Systemic Autoimmune Diseases}, vol.~1,
  no.~4, 2009.

\bibitem{goodfellowmulti}
I.~Goodfellow, J.~Ibarz, Y.~Bulatov, S.~Arnoud, and V.~Shet, ``Multi-digit
  number recognition from street view imagery using deep convolutional neural
  networks,'' in \emph{International Conference on Learning Representations},
  2014.

\bibitem{fgsm}
I.~J. Goodfellow, J.~Shlens, and C.~Szegedy, ``Explaining and harnessing
  adversarial examples,'' in \emph{International Conference on Learning
  Representations}, San Diego, CA, USA, 2015.

\bibitem{pcl}
A.~Mustafa, S.~Khan, M.~Hayat, R.~Goecke, J.~Shen, and L.~Shao, ``Adversarial
  defense by restricting the hidden space of deep neural networks,'' in
  \emph{IEEE/CVF International Conference on Computer Vision}, 2019, pp.
  3385--3394.

\bibitem{Wang2021Probabilistic}
Q.~Wang, F.~Liu, B.~Han, and et~al., ``Probabilistic margins for instance
  reweighting in adversarial training,'' in \emph{Advances in Neural
  Information Processing Systems}, 2021.

\bibitem{zhang2024duality}
Y.~Zhang, H.~He, J.~Zhu, H.~Chen, Y.~Wang, and Z.~Wei, ``On the duality between
  sharpness-aware minimization and adversarial training,'' in
  \emph{International Conference on Machine Learning}, 2024.

\end{thebibliography}

\newpage
\onecolumn
\section*{Appendix}

\subsection{Proof of Lemma~\ref{lamma1}}\label{append.lemma1}
\begin{proof}
We prove it from the simplest case with three examples in two-dimensional space as shown in Fig.~\ref{fig.proof}.

\begin{figure}[h] 
\centering
\includegraphics[width=0.3\linewidth]{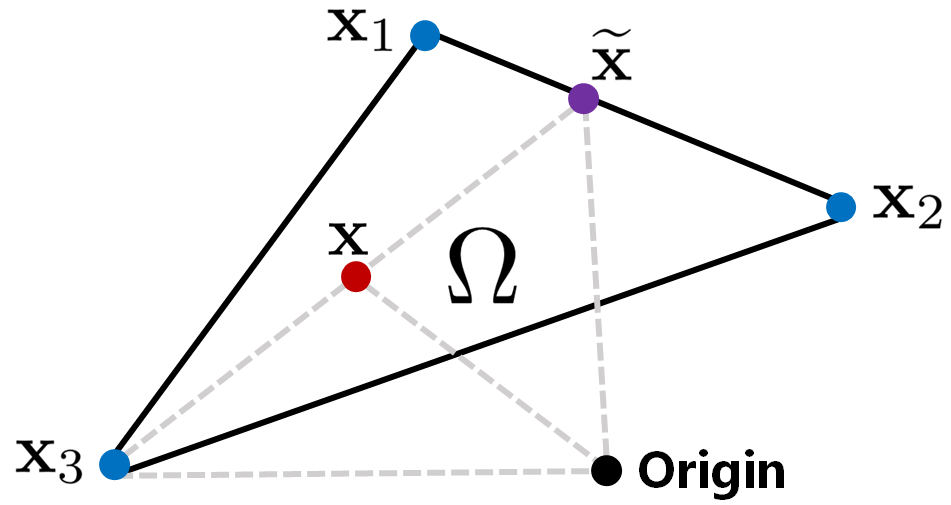}
\caption{A case with three examples in two-dimensional space.}
\label{fig.proof}
\end{figure}

For any example $\mathbf{x}$ in the convex hull of $\{\mathbf{x}_1,\mathbf{x}_2,\mathbf{x}_3\}$, i.e., $\forall\mathbf{x}\in\Omega(\mathbf{x}_1,\mathbf{x}_2,\mathbf{x}_3)$, we can find a line that crosses $\mathbf{x}$ and $\mathbf{x}_3$ and intersects with $\mathbf{x}_1\mathbf{x}_2$ at $\widetilde{\mathbf{x}}$. We can denote $\widetilde{\mathbf{x}}=\alpha\mathbf{x}_1+(1-\alpha)\mathbf{x}_2$, where $\alpha\in[0,1]$. Then,
\begin{displaymath}
\begin{aligned}
\mathbf{x}&=\beta\widetilde{\mathbf{x}}+(1-\beta)\mathbf{x}_3\\
&=\beta\left(\alpha\mathbf{x}_1+(1-\alpha)\mathbf{x}_2\right)+(1-\beta)\mathbf{x}_3\\
&=\beta\alpha\mathbf{x}_1+\beta(1-\alpha)\mathbf{x}_2+(1-\beta)\mathbf{x}_3\\
&=w_1\mathbf{x}_1+w_2\mathbf{x}_2+w_3\mathbf{x}_3,
\end{aligned}
\end{displaymath}
where $\beta\in[0,1]$, $w_1=\beta\alpha$, $w_2=\beta(1-\alpha)$, $w_3=1-\beta$, and
\begin{displaymath}
w_1+w_2+w_3=\beta\alpha+\beta-\beta\alpha+1-\beta=1.
\end{displaymath}
This proof can be easily generalized to complex cases with more examples in higher dimensional space by mathematical induction.
\end{proof}

~\\
\subsection{Proof of Lemma~\ref{lamma2}}\label{append.lemma2}
\begin{proof}
For any $\mathbf{A},\mathbf{B}\in\mathbb{W}$, denote $\mathbf{C}=\mathbf{B}\mathbf{A}$. Obviously, there are
\begin{equation}
c_{i,j}=\sum_{k=1}^{N}a_{i,k}b_{k,j}\in[0,1),
\end{equation}
\begin{equation}
\sum_{j=1}^{N}c_{i,j}=\sum_{j=1}^{N}\sum_{k=1}^{N}a_{i,k}b_{k,j}=\sum_{k=1}^{N}a_{i,k}\sum_{j=1}^{N}b_{k,j}=1\times 1=1,
\end{equation}
\begin{equation}
\sum_{i=1}^{N}\sum_{j=1}^{N}c_{i,j}=N.
\end{equation}
Thus, $\mathbf{C}\in\mathbb{W}$.
\end{proof}

~\\
\subsection{Proof of Lemma~\ref{lamma3}}\label{append.lemma3}
\begin{proof}
Since $\mathbf{W}^{(t)}=\mathbf{W}^t\mathbf{W}^{(t-1)}$, there are
\begin{itemize}
\item $w^{(t)}_{i,j}=\sum_{k=1}^{N}w^{t}_{i,k}w^{(t-1)}_{k,j}\geq \min_k\{w^{(t-1)}_{k,j}\}\sum_{k=1}^{N}w^{t}_{i,k}=\min_k\{w^{(t-1)}_{k,j}\}$,
\item $w^{(t)}_{i,j}=\sum_{k=1}^{N}w^{t}_{i,k}w^{(t-1)}_{k,j}\leq \max_k\{w^{(t-1)}_{k,j}\}\sum_{k=1}^{N}w^{t}_{i,k}=\max_k\{w^{(t-1)}_{k,j}\}$.
\end{itemize}
Therefore $w^{(t)}_{i,j}\in\big[\min_k\{w^{(t-1)}_{k,j}\},\max_k\{w^{(t-1)}_{k,j}\}\big]$, i.e., $\big[\min_i\{w^{(t)}_{i,j}\},\max_i\{w^{(t)}_{i,j}\}\big]\subseteq\big[\min_k\{w^{(t-1)}_{k,j}\},\max_k\{w^{(t-1)}_{k,j}\}\big]$, from which we know that
$\min_i\{w^{(t)}_{i,j}\}$ is monotonically increasing with respect to $t$, and $\max_i\{w^{(t)}_{i,j}\}$ is monotonically decreasing with respect to $t$. Since $\min_i\{w^{(t)}_{i,j}\}<\max_i\{w^{(1)}_{i,j}\}$ and $\max_i\{w^{(t)}_{i,j}\}>\min_i\{w^{(1)}_{i,j}\}$, we know that $\min_i\{w^{(t)}_{i,j}\}$ has upper bound and $\max_i\{w^{(t)}_{i,j}\}$ has lower bound. According to the monotonic bounded theorem, both $\min_i\{w^{(t)}_{i,j}\}$ and $\max_i\{w^{(t)}_{i,j}\}$ converge with respect to $t$, therefore the mathematical limits of $\min_i\{w^{(t)}_{i,j}\}$ and $\max_i\{w^{(t)}_{i,j}\}$ exist.

Without loss of generality, suppose $\exists\ \xi_1>\min_i\{w^{(1)}_{i,j}\}$ and $\exists\ \xi_2<\max_i\{w^{(1)}_{i,j}\}$ such that $\lim\limits_{t\rightarrow\infty}\min_i\{w^{(t)}_{i,j}\}=\xi_1$ and $\lim\limits_{t\rightarrow\infty}\max_i\{w^{(t)}_{i,j}\}=\xi_2$. That is, $\forall\varepsilon>0$,
\begin{itemize}
\item $\exists\lambda_1\in\mathcal{R}$, $0<\lambda_1<1$, $\exists N_1\in Z^{+}$, such that when $t>N_1$ there is $\big|\min_i\{w^{(t)}_{i,j}\}-\xi_1\big|<\frac{\lambda_1\varepsilon}{2}$,
\item $\exists\lambda_2\in\mathcal{R}$, $0<\lambda_2<1$, $\exists N_2\in Z^{+}$, such that when $t>N_2$ there is $\big|\max_i\{w^{(t)}_{i,j}\}-\xi_2\big|<\frac{\lambda_2\varepsilon}{2}$.
\end{itemize}
Therefore, when $t>N_1+N_2$, 
\begin{equation}
\begin{array}{lll}
&\big|\min_i\{w^{(t)}_{i,j}\}-\min_i\{w^{(t+1)}_{i,j}\}\big| \\
&= \big|\big(\min_i\{w^{(t)}_{i,j}\}-\xi_1\big)-\big(\min_i\{w^{(t+1)}_{i,j}\}-\xi_1\big)\big|\\
&\leq \big|\min_i\{w^{(t)}_{i,j}\}-\xi_1\big|+\big|\min_i\{w^{(t+1)}_{i,j}\}-\xi_1\big|\\
&< \lambda_1\varepsilon,
\end{array}
\end{equation}
\begin{equation}
\begin{array}{lll}
&\big|\max_i\{w^{(t+1)}_{i,j}\}-\max_i\{w^{(t)}_{i,j}\}\big| \\
&= \big|\big(\max_i\{w^{(t+1)}_{i,j}\}-\xi_2\big)-\big(\max_i\{w^{(t)}_{i,j}\}-\xi_2\big)\big|\\
&\leq \big|\max_i\{w^{(t+1)}_{i,j}\}-\xi_2\big|+\big|\max_i\{w^{(t)}_{i,j}\}-\xi_2\big|\\
&< \lambda_2\varepsilon.
\end{array}
\end{equation}
Since $\big[\min_i\{w^{(t+1)}_{i,j}\},\max_i\{w^{(t+1)}_{i,j}\}\big]\subseteq\big[\min_i\{w^{(t)}_{i,j}\},\max_i\{w^{(t)}_{i,j}\}\big]$, we can denote 
\begin{itemize}
\item $\min_i\{w^{(t+1)}_{i,j}\} = \lambda_1\max_i\{w^{(t)}_{i,j}\}+(1-\lambda_1)\min_i\{w^{(t)}_{i,j}\}$,
\item $\max_i\{w^{(t+1)}_{i,j}\} = (1-\lambda_2)\max_i\{w^{(t)}_{i,j}\}+\lambda_2\min_i\{w^{(t)}_{i,j}\}$.
\end{itemize}
Therefore,
\begin{itemize}
\item $\big|\min_i\{w^{(t)}_{i,j}\}-\lambda_1\max_i\{w^{(t)}_{i,j}\}-(1-\lambda_1)\min_i\{w^{(t)}_{i,j}\}\big| < \lambda_1\varepsilon$, thus $\big|\min_i\{w^{(t)}_{i,j}\}-\max_i\{w^{(t)}_{i,j}\}\big| < \varepsilon$,
\item $\big|(1-\lambda_2)\max_i\{w^{(t)}_{i,j}\}+\lambda_2\min_i\{w^{(t)}_{i,j}\}-\max_i\{w^{(t)}_{i,j}\}\big| < \lambda_2\varepsilon$, thus $\big|\max_i\{w^{(t)}_{i,j}\}-\min_i\{w^{(t)}_{i,j}\}\big| < \varepsilon$.
\end{itemize}
This means, when $t>N_1+N_2$, $\forall\varepsilon>0$ there is $|\xi_1-\xi_2|<\varepsilon$, i.e., $\xi_1=\xi_2=\xi$, $\lim\limits_{t\rightarrow\infty}\min_i\{w^{(t)}_{i,j}\}= \lim\limits_{t\rightarrow\infty}\max_i\{w^{(t)}_{i,j}\} =\xi$. 

Finally, we get the conclusion that when $t\rightarrow\infty$, there is $w^{(t)}_{1,j}=w^{(t)}_{2,j}=\ldots=w^{(t)}_{N,j}=\xi$, i.e., the values in each column of $\mathbf{W}^{(t)}$ become equal.
\end{proof}

~\\
\subsection{Proof of Proposition~\ref{proposition1}}\label{append.prop1}
\begin{proof}
Let $\mathbf{x}_i=[x_{i1},\ldots,x_{id}]$ (which is a row vector) be the $i$-th clean example, and $\mathbf{X}=[\mathbf{x}_1;\ldots;\mathbf{x}_N]$ be the clean data matrix. According to Lemma~\ref{lamma1} and Assumption~\ref{assumption1}, there are
\begin{equation}\label{eq.advSamples}
\begin{aligned}
\mathbf{x}_1^{\prime}&=w_{11}\mathbf{x}_1+w_{12}\mathbf{x}_2+\ldots+w_{1N}\mathbf{x}_N=\mathbf{w}_{1}\mathbf{X}=[x_{11}^{\prime},\ldots,x_{1d}^{\prime}],\\
\mathbf{x}_2^{\prime}&=w_{21}\mathbf{x}_1+w_{22}\mathbf{x}_2+\ldots+w_{2N}\mathbf{x}_N=\mathbf{w}_{2}\mathbf{X}=[x_{21}^{\prime},\ldots,x_{2d}^{\prime}],\\
\ldots\\
\mathbf{x}_N^{\prime}&=w_{N1}\mathbf{x}_1+w_{N2}\mathbf{x}_2+\ldots+w_{NN}\mathbf{x}_N=\mathbf{w}_{N}\mathbf{X}=[x_{N1}^{\prime},\ldots,x_{Nd}^{\prime}],\\
\end{aligned}
\end{equation} 
where
\begin{equation}\label{eq.constraintsW}
\begin{array}{llll}
&\mathbf{w}_{i}=[w_{i,1},w_{i,2},\ldots,w_{i,N}],\ i=1,\ldots,N,\\
&w_{i,j}\in[0,1),\ i=1,\ldots,N,\ j=1,\ldots,N,\\
&\sum_{j=1}^{N}w_{i,j}=1,\ i=1,\ldots,N,\\
&\sum_{i=1}^{N}\sum_{j=1}^{N}w_{i,j}=N.\\
\end{array}
\end{equation} 
Denote $\mathbf{W}=[\mathbf{w}_{1};\mathbf{w}_{2};\ldots;\mathbf{w}_{N}]$, (\ref{eq.advSamples}) can be re-written into the matrix form, i.e.,
\begin{equation}\label{eq.advSamplesMatrix}
\mathbf{X}^{\prime}=\mathbf{W}\mathbf{X}.
\end{equation} 
Let $\alpha=[\frac{1}{d},\frac{1}{d},\ldots,\frac{1}{d}]$, then for a clean example $\mathbf{x}_i$ and its corresponding adversarial example $\mathbf{x}_i^{\prime}$, their feature means are calculated as 
\begin{equation}
\begin{aligned}
&\mu_i=\frac{1}{d}\sum_{j=1}^{d}x_{ij}=\mathbf{x}_i\alpha^{\top},\\
&\mu_i^{\prime}=\frac{1}{d}\sum_{j=1}^{d}x_{ij}^{\prime}=\mathbf{x}_i^{\prime}\alpha^{\top}.
\end{aligned}
\end{equation}
Denote $\bm{\mu}=[\mu_1,\mu_2,\ldots,\mu_N]^{\top}$ and $\bm{\mu}{\prime}=[\mu_1^{\prime},\mu_2^{\prime},\ldots,\mu_N^{\prime}]^{\top}$, there are
\begin{equation}
\begin{aligned}
&\bm{\mu}=\mathbf{X}\alpha^{\top},\\
&\bm{\mu}^{\prime}=\mathbf{X}^{\prime}\alpha^{\top}=\mathbf{W}\mathbf{X}\alpha^{\top}=\mathbf{W}{\bm\mu}.
\end{aligned}
\end{equation}
Further let $\beta=[\frac{1}{N},\frac{1}{N},\ldots,\frac{1}{N}]$, the expectaions of $\mu$ and $\mu^{\prime}$ are estimated as
\begin{equation}\label{eq.meanClnAdv}
\begin{aligned}
&\mathbb{E}(\mu)=\frac{1}{Nd}\sum_{i=1}^{N}\sum_{j=1}^{d}x_{ij}=\beta\mathbf{X}\alpha^{\top}=\Phi_{\mu},\\
&\mathbb{E}(\mu^{\prime})=\frac{1}{Nd}\sum_{i=1}^{N}\sum_{j=1}^{d}x_{ij}^{\prime}=\beta\mathbf{X}^{\prime}\alpha^{\top}=\beta\mathbf{W}\mathbf{X}\alpha^{\top}=\Phi_{\mu^{\prime}}.
\end{aligned}
\end{equation} 
The variance of $\{\mu_1,\mu_2,\ldots,\mu_N\}$ is estimated as
\begin{equation}\label{eq.stdCln}
\begin{aligned}
&{\rm var}(\mu)\\
&=\frac{1}{N}\sum_{i=1}^{N}(\mu_i-\Phi_{\mu})^2\\
&=\frac{1}{N}\sum_{i=1}^{N}(\mu_i^2-2\mu_i\Phi_{\mu}+\Phi_{\mu}^2)\\
&=\frac{1}{N}\sum_{i=1}^{N}(\mu_i^2-\Phi_{\mu}^2)\\
&=\Sigma_{\mu}^2.
\end{aligned}
\end{equation} 
Similarly, the variance of $\{\mu_1^{\prime},\mu_2^{\prime},\ldots,\mu_N^{\prime}\}$ is estimated as
\begin{equation}\label{eq.stdAdv}
{\rm var}(\mu^{\prime})=\frac{1}{N}\sum_{i=1}^{N}\left((\mu_i^{\prime})^2-(\Phi_{\mu}^{\prime})^2\right)=\Sigma_{\mu^{\prime}}^2.
\end{equation} 
We reasonably assume that the adversarial attack didn't change the expectation of the feature means, i.e., $\Phi_{\mu}=\Phi_{\mu^{\prime}}$, then the difference of the two variances is calculated as
\begin{equation}\label{eq.varDiff}
\begin{aligned}
&{\rm var}(\mu)-{\rm var}(\mu^{\prime})\\
&=\frac{1}{N}\sum_{i=1}^{N}\left(\mu_i^2-\Phi_{\mu}^2\right)-\frac{1}{N}\sum_{i=1}^{N}\left((\mu_i^{\prime})^2-(\Phi_{\mu}^{\prime})^2\right)\\
&=\frac{1}{N}\sum_{i=1}^{N}\mu_i^2-\frac{1}{N}\sum_{i=1}^{N}(\mu_i^{\prime})^2\\
&=\frac{1}{N}\sum_{i=1}^{N}\left(\mu_i^2-(\mu_i^{\prime})^2\right)\\
&=\frac{1}{N}\left(\bm{\mu}^{\top}\bm{\mu}-(\bm{\mu}^{\prime})^{\top}(\bm{\mu}^{\prime})\right)\\
&=\frac{1}{N}\left(\bm{\mu}^{\top}\bm{\mu}-\bm{\mu}^{\top}\mathbf{W}^{\top}\mathbf{W}\bm{\mu}\right)\\
&=\frac{1}{N}\left(\bm{\mu}^{\top}(\mathbf{I}-\mathbf{W}^{\top}\mathbf{W})\bm{\mu}\right).
\end{aligned}
\end{equation}
Let $\mathbf{A}=\mathbf{W}^{\top}\mathbf{W}$ ($\mathbf{A}$ is a real symmetric matrix), and perform orthogonal decomposition to $\mathbf{A}$, there is
\begin{equation}\label{eq.decompose}
\mathbf{A}=\mathbf{Q}\Lambda\mathbf{Q}^{\top},
\end{equation}
where $\mathbf{Q}$ is an orthogonal matrix, $\Lambda=diag(\lambda_1,\lambda_2,\ldots,\lambda_N)$, $\lambda_1\geq\lambda_2\geq\ldots\geq\lambda_N\geq 0$ are the eigenvalues of $\mathbf{A}$ and $\sum_{i=1}^{N}\lambda_i=tr(\mathbf{W}^{\top}\mathbf{W})=\sum_{i=1}^{N}\sum_{j=1}^{N}w_{i,j}^2$. Furthermore, since $w_{i,j}\in[0,1)$ and $\sum_{i=}^{N}\sum_{j=1}^{N}w_{i,j}=N$, it is obvious that $\sum_{i=1}^{N}\lambda_i\leq N$. Let $\mathbf{z}=\mathbf{Q}^{\top}{\bm\mu}=[z_1,z_2,\ldots,z_N]^{\top}$, then
\begin{equation}\label{eq.varDif2}
\begin{aligned}
&{\rm var}(\mu)-{\rm var}(\mu^{\prime})\\
&=\frac{1}{N}\left(\mathbf{z}^{\top}\mathbf{Q}^{\top}\mathbf{Q}\mathbf{z}-\mathbf{z}^{\top}\Lambda\mathbf{z}\right)\\
&=\frac{1}{N}\left(\mathbf{z}^{\top}\mathbf{z}-\mathbf{z}^{\top}\Lambda\mathbf{z}\right)\\
&=\frac{1}{N}\left((1-\lambda_1)z_1^2+(1-\lambda_2)z_2^2+\ldots+(1-\lambda_N)z_N^2\right).
\end{aligned}
\end{equation}
We discuss the situation that a series of attacking steps $\mathbf{W}^1,\mathbf{W}^2,\ldots, \mathbf{W}^t$ finally lead to the attacking results, i.e.,
\begin{equation}\label{eq.advSamplesMatrix2}
\mathbf{X}^{\prime}=\mathbf{W}^t\ldots\mathbf{W}^2\mathbf{W}^1\mathbf{X}=\mathbf{W}^{(t)}\mathbf{X}.
\end{equation} 
Suppose $\mathbf{W}^1$ is full rank. Based on traditional textbook, and according to Lemma~\ref{lamma2} and Lemma~\ref{lamma3}, when $t\rightarrow\infty$, there is 
\begin{equation}
\begin{aligned}
N&=rank(\mathbf{W}^1)\geq rank(\mathbf{W}^2\mathbf{W}^1)\geq rank(\mathbf{W}^3\mathbf{W}^2\mathbf{W}^1)\\
&\geq\ldots\geq rank(\mathbf{W}^{(t)})=1.
\end{aligned}
\end{equation}
For the extreme case of $t\rightarrow\infty$, it is obvious that ${\rm var}(\mu)-{\rm var}(\mu^{\prime})> 0$ holds with probability 1. Correspondingly, from the initial state to the extreme state, the rank of $\mathbf{W}^{(t)}$ becomes smaller and smaller, i.e., the zero eigenvalues become more and more, thereby the must-positive terms in Eq.~(\ref{eq.varDif2}) become more and more. During this process, since $\sum_{i=1}^{N}\lambda_i\leq N$ always holds, we can say that ${\rm var}(\mu)-{\rm var}(\mu^{\prime})\geq 0$ holds with a high probability. 
\end{proof}

{\bf Geometric Interpretation} This ``high probability” can be interpreted geometrically. From Eq.~(\ref{eq.varDiff}) we know that ${\rm var}(\mu)-{\rm var}(\mu^{\prime})\geq0$ holds when $\bm{\mu}^{\top}\bm{\mu}\geq\bm{\mu}^{\top}\mathbf{W}^{\top}\mathbf{W}\bm{\mu}$, i.e., $|\bm{\mu}|\geq|\mathbf{W}\bm{\mu}|=|\bm{\mu}^{\prime}|$. Taking the origin as the center and $|\bm{\mu}|$ as the radius of a hypersphere in $N$-dimensional space, obviously, there are $|\bm{\mu}|>|\bm{\mu}^{\prime}|$ when $\bm{\mu}^{\prime}$ is inside the hypersphere, and $|\bm{\mu}|<|\bm{\mu}^{\prime}|$ when $\bm{\mu}^{\prime}$ is outside the hypersphere. Note that the requirements for $\mathbf{W}$ in (\ref{eq.constraintsW}) will constrain $\bm{\mu}^{\prime}$ in a hypercube area, where the following Fig.~\ref{fig.GeometricInterpretation} provides some two-dimensional cases. No matter where $\bm{\mu}$ is located on the hypersphere, the hypercube always has at least half of its volume inside the hypersphere, and in many cases, the volume inside the hypersphere (i.e., gray area) is much larger than that outside it (i.e., pink area).

\begin{figure}[h]
\centering
\subfloat[\small Two-Dimensional Case 1]{\includegraphics[width=0.35\linewidth]{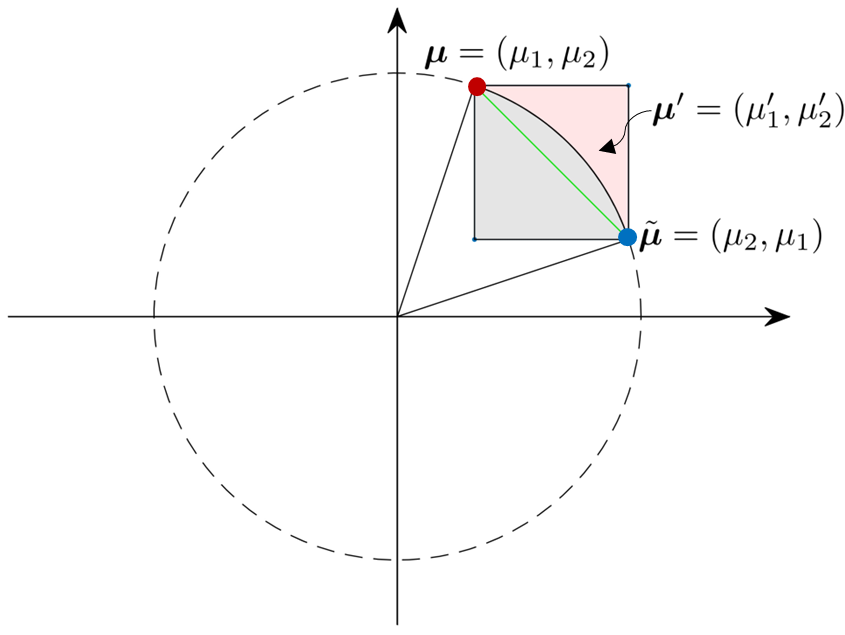}}\
\subfloat[\small Two-Dimensional Case 2]{\includegraphics[width=0.35\linewidth]{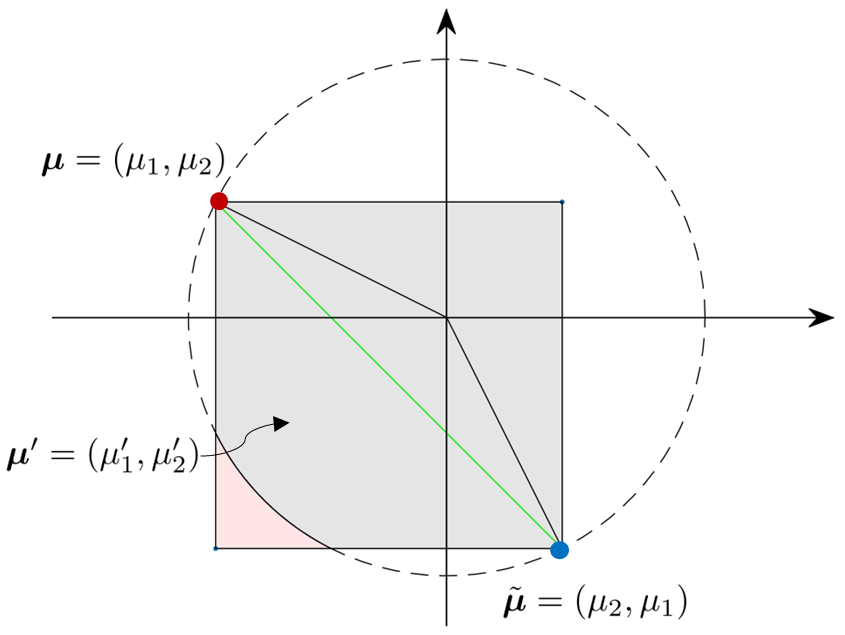}}\\
\subfloat[\small Two-Dimensional Case 3]{\includegraphics[width=0.35\linewidth]{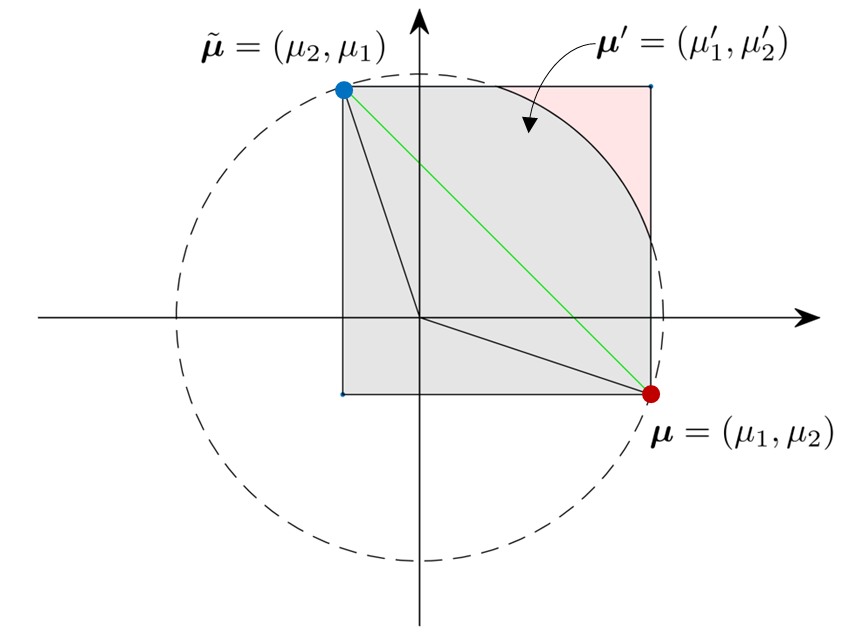}}\
\subfloat[\small Two-Dimensional Case 4]{\includegraphics[width=0.35\linewidth]{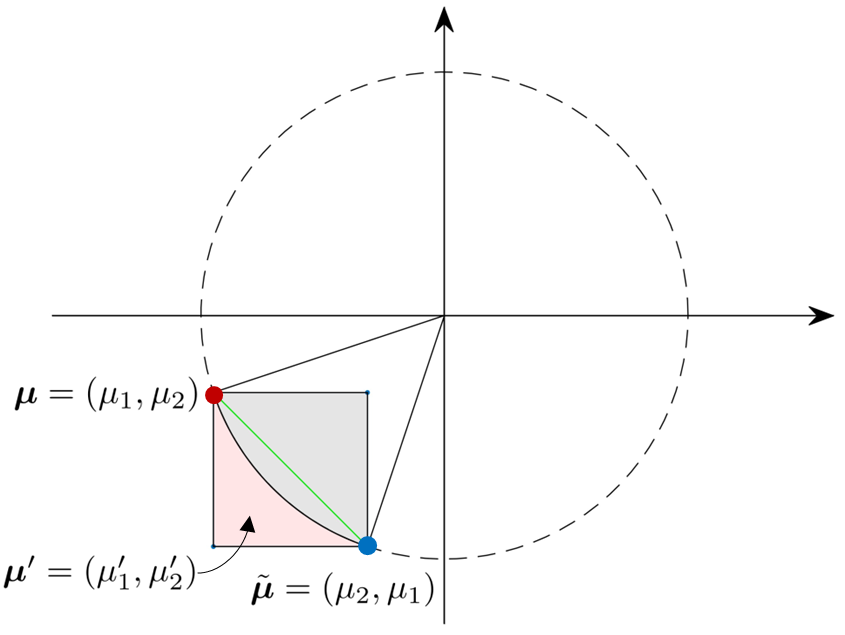}}\\
\vspace{0.3cm}
\caption{Geometric interpretation for Proposition~\ref{proposition1}.}
\label{fig.GeometricInterpretation}
\vspace{1.0cm}
\end{figure}

~\\
\subsection{Proof of Proposition~\ref{proposition3}}\label{append.prop3}
\begin{proof}
Following the proof of Proposition~\ref{proposition1}, we discuss the situation that a series of attacking steps $\mathbf{W}^1,\mathbf{W}^2,\ldots,\mathbf{W}^t$ finally lead to the attacking results, i.e., $\mathbf{X}^{\prime}=\mathbf{W}^t\ldots\mathbf{W}^2\mathbf{W}^1\mathbf{X}=\mathbf{W}^{(t)}\mathbf{X}$. There is
\begin{equation}
\begin{aligned}
&\lim_{t\rightarrow\infty}{\rm var}(\sigma^{\prime 2})\\
&=\lim_{t\rightarrow\infty}\frac{1}{N}\sum_{i=1}^{N}(\sigma_i^{\prime 2}-\Phi^2_{\sigma^{\prime 2}})\\
&=\lim_{t\rightarrow\infty}\frac{1}{N}\sum_{i=1}^{N}\left[\frac{1}{d}(\mathbf{w}_i\mathbf{X}\mathbf{X}^{\top}\mathbf{w}_i^{\top})-(\mathbf{w}_i{\bm\mu})^2\right]^2-\Phi^2_{\sigma^{\prime 2}}\\
&=\left[\frac{1}{d}(\beta\mathbf{X}\mathbf{X}^{\top}\beta^{\top})-(\beta{\bm\mu}{\bm\mu}^{\top}\beta^{\top})\right]^2-\Phi^2_{\sigma^{\prime 2}}\\
&=0.
\end{aligned}
\end{equation}
Thus, there is $\lim_{t\rightarrow\infty}[{\rm var}(\sigma^2)-{\rm var}(\sigma^{\prime 2})]>>0$. Correspondingly, from the initial state to the extreme state, we can say that ${\rm var}(\sigma^2)-{\rm var}(\sigma^{\prime 2})\geq 0$ with a high probability. 
\end{proof}

~\\
\subsection{Proof of Theorem~\ref{thm1}}\label{append.thm1}
\begin{proof}
We derive the distributions of $\hat{\mu}_{b,c}$ and $\hat{\sigma}_{b,c}$. Based on Eq.~(\ref{eq.noise}) and Eq.~(\ref{eq.reConstuctMS}) we know that $\hat{\mu}_{b,c}$ and $\hat{\sigma}_{b,c}$ follow independent Gaussian distributions. According to the operation rules for Gaussian distributions, there are
\begin{equation}\label{eq.reConstuctMS1}
\small
\left\{
\begin{aligned}
        &\mathbb{E}[\hat{\mu}_{b,c}] = \mathbb{E}[\mu_{b,c}]+\alpha \cdot \Sigma_{\mu_c} \cdot \mathbb{E}[\varepsilon^{\mu}_{b,c}] = \Phi_{\mu_c}, \\
        &{\rm var}[\hat{\mu}_{b,c}] = {\rm var}[\mu_{b,c}]+\alpha^2 \cdot \Sigma_{\mu_c}^2 \cdot {\rm var}[\varepsilon^{\mu}_{b,c}]  = (1+\alpha^2)\Sigma_{\mu_c}^2, \\
        &\mathbb{E}[\hat{\sigma}_{b,c}] = \mathbb{E}[\sigma_{b,c}]+\beta \cdot \Sigma_{\sigma_c} \cdot \mathbb{E}[\varepsilon^{\sigma}_{b,c}] = \Phi_{\sigma_c}, \\
        &{\rm var}[\hat{\sigma}_{b,c}] = {\rm var}[\sigma_{b,c}]+\beta^2 \cdot \Sigma_{\sigma_c}^2 \cdot {\rm var}[\varepsilon^{\sigma}_{b,c}]  = (1+\beta^2)\Sigma_{\sigma_c}^2, \\
\end{aligned}\right.
\end{equation}
from which there are
\begin{equation}\label{eq.reConstuctMS2}
\left\{
\begin{aligned}
    &\hat{\mu}_{b,c}\sim\mathcal{N}\big(\Phi_{\mu_c},(1+\alpha^2)\Sigma_{\mu_c}^2\big),\\
    &\hat{\sigma}_{b,c}\sim\mathcal{N}\big(\Phi_{\sigma_c},(1+\beta^2)\Sigma_{\sigma_c}^2\big).\\
\end{aligned}\right.
\end{equation}
Since $\alpha^2\geq0$ and $\beta^2\geq0$, there are ${\rm var}[\hat{\mu}_{b,c}]\geq{\rm var}[{\mu}_{b,c}]$ and ${\rm var}[\hat{\sigma}_{b,c}]\geq{\rm var}[{\sigma}_{b,c}]$, i.e., the variances of both the feature means and standard deviations become larger after reconstruction. 
\end{proof}

\newpage
\subsection{Robustness Comparison of FSU-NT with State-of-the-Arts}\label{append.mild}

\begin{table}[h]
    \centering
    \caption{Robust accuracy (\%) under typical gradient-based mild attacks.}
    \vspace{-0.2cm}
    \label{tab:mild}
    \renewcommand\arraystretch{0.95}
    \scalebox{0.95}{
    \begin{tabular}{lp{1.2cm}p{1.6cm}p{1.6cm}p{1.6cm}p{1.6cm}p{1.6cm}r}
    \toprule
    Method & AT & Cln. & FGSM$_\infty$ & BIM$_\infty$ & MIM$_\infty$ & PGD$_\infty$ & Time\\
    &&&(TorchAttack)&(TorchAttack)&(TorchAttack)&(TorchAttack)&\\
    &&&&($2/255$)&($2/255$)&($2/255$)&\\
    \midrule
    \multicolumn{8}{c}{\pmb{MNIST} ($\epsilon=0.3$)\quad 40 steps for BIM, MIM and PGD} \\
    \cdashline{1-8}[1pt/1pt]\noalign{\smallskip}
    PCL & Y & 99.21 & 70.26 & 26.71 & 26.75 & 10.06 & 73499s\\
    TRADES & Y & 99.53 & 97.55 & 96.28 & 96.33 & 97.29 & 9868s\\
    AT-AWP & Y & 98.73 & 95.47 & 94.83 & 94.82 & 96.18 & 177064s\\
    MAIL-TRADES & Y & 99.14 & 97.29 & 96.28 & 96.41 & 97.14 & 48132s\\
    MLCAT$_{\rm WP}$ & Y & 98.88 & 97.24 & 96.38 & 96.16 & 96.85 & 100989s\\
    AT+RiFT & Y & 99.44 & 95.20 & 94.02 & 94.26 & 94.87 & 37364s\\
    DKL  & Y & 98.80 & 97.07 & 96.91 & 96.73 & 97.02 & 63728s\\
    SAM & N & 99.19 & 08.44 & 08.68 & 09.25 & 08.87 & {\bf 5530s}\\
    FSU-NT (Ours) & N & 99.16 & 90.59 & 91.63 & 85.86 & 85.42 & {\bf 1306s}\\
    \midrule
    \multicolumn{8}{c}{\pmb{CIFAR10} ($\epsilon=8/255$)\quad 10 steps for BIM, MIM and PGD} \\
    \cdashline{1-8}[1pt/1pt]\noalign{\smallskip}
    PCL & Y & 89.85 & 54.32 & 17.47 & 18.62 & 15.17 & 187194s \\
    TRADES & Y & 85.76 & 61.61 & 53.67 & 55.76 & 54.29 & 160872s \\
    AT-AWP & Y & 84.82 & 58.90 & 52.28 & 53.85 & 52.74 & 210223s\\
    MAIL-TRADES & Y & 84.03 & 55.96 & 50.89 & 51.92 & 51.31 & 101026s\\
    MLCAT$_{\rm WP}$ & Y & 85.11 & 58.87 & 55.77 & 56.16 & 56.30 & 429540s\\
    AT+RiFT & Y & 76.01 & 40.79 & 34.28 & 35.77 & 34.53 & 127665s\\
    DKL  & Y & 79.84 & 59.06 & 55.70	 & 56.44 & 55.95 & 73323s\\
    SAM & N & 93.91 & 45.17 & 21.86 & 21.66 & 19.67 & {\bf 6496s}\\
    FSU-NT (Ours) & N & 92.36 & 73.24 & 59.11 & 57.06 & 59.68 & {\bf 7152s}\\
    \midrule
    \multicolumn{8}{c}{\pmb{SVHN} ($\epsilon=8/255$)\quad 10 steps for BIM, MIM and PGD} \\
    \cdashline{1-8}[1pt/1pt]\noalign{\smallskip}
    PCL & Y & 97.32 & 72.00 & 27.92 & 31.36 & 30.38 & 273098s\\
    TRADES & Y & 91.73 & 69.10 & 58.63 & 60.48 & 59.03 & 243508s\\
    AT-AWP & Y & 93.39 & 66.61 & 55.78 & 58.82 & 56.84 & 112945s\\
    MAIL-TRADES & Y & 92.42 & 70.74 & 60.38 & 61.69 & 60.77 & 60007s\\
    MLCAT$_{\rm WP}$ & Y & 19.59 & 19.59 & 19.59 & 19.59 & 19.59 & 265291s\\
    AT+RiFT & Y & 91.33 & 53.80 & 41.56 & 44.62 & 42.13 & 75817s\\
    DKL  & Y & 95.14 & 90.88 & 65.34 & 73.05 & 67.99 & 105841s\\
    SAM & N & 19.59 & 19.59 & 19.59 & 19.59 & 19.59 & {\bf 12358s}\\
    FSU-NT (Ours) & N & 96.25 & 80.06 & 64.98 & 58.92 & 67.83 & {\bf 4985s}\\
    \midrule
    \multicolumn{8}{c}{\pmb{CIFAR100} ($\epsilon=8/255$)\quad 10 steps for BIM, MIM and PGD} \\
    \cdashline{1-8}[1pt/1pt]\noalign{\smallskip}
    TRADES & Y & 58.93 & 33.61 & 31.14 & 31.73 & 31.40 & 163013s \\
    AT-AWP & Y & 53.61 & 32.72 & 30.91 & 31.37 & 31.11 & 77060s\\
    MAIL-TRADES & Y & 61.33 & 31.29 & 27.77 & 28.45 & 28.03 & 114857s\\
    MLCAT$_{\rm WP}$ & Y & 63.51 & 33.68 & 29.48 & 30.46 & 29.70 & 158484s\\
    AT+RiFT & Y & 44.35 & 20.53 & 18.05 & 18.50 & 18.16 & 125939s\\
    DKL  & Y & 60.44 & 35.21 & 32.72 & 33.24 & 32.89 & 65737s\\
    SAM & N & 77.40 & 32.73 & 13.81 & 14.77 & 12.98 & {\bf 6566s}\\
    FSU-NT (Ours) & N &70.03 & 49.39 & 36.06 & 42.20 & 34.76 & {\bf 11026s}\\
    \bottomrule
    \end{tabular}
    }
\begin{minipage}{14.5cm}
\vspace{0.1cm}
\footnotesize{Note 1: ``Y" / ``N" represents adversarial training / non-adversarial training. \\
Note 2: On dataset SVHN, MLCAT$_{\rm wp}$ and SAM have trapped into a fixed pattern, and the perturbations generated by all attacks lose effectiveness.}
\end{minipage}
\end{table}

\newpage
\subsection{More Sensitivity Figures}\label{append.sensitivity}

\begin{figure*}[h]
\centering
\subfloat[\scriptsize SVHN - FGSM$_\infty$ Attack ($\epsilon=8/255$)]{\includegraphics[width=0.49\linewidth]{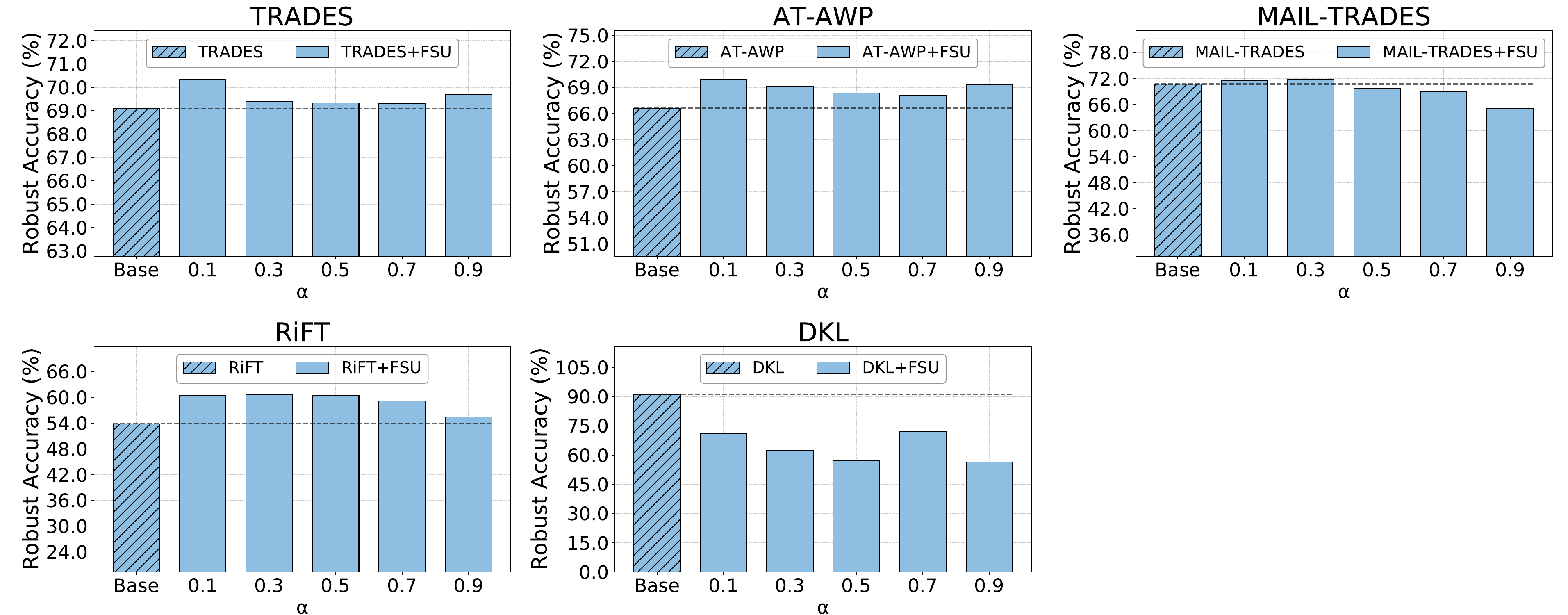}}\quad
\subfloat[\scriptsize SVHN - PGD$_\infty$-100 Attack ($\epsilon=8/255$, step size $2/255$)]{\includegraphics[width=0.49\linewidth]{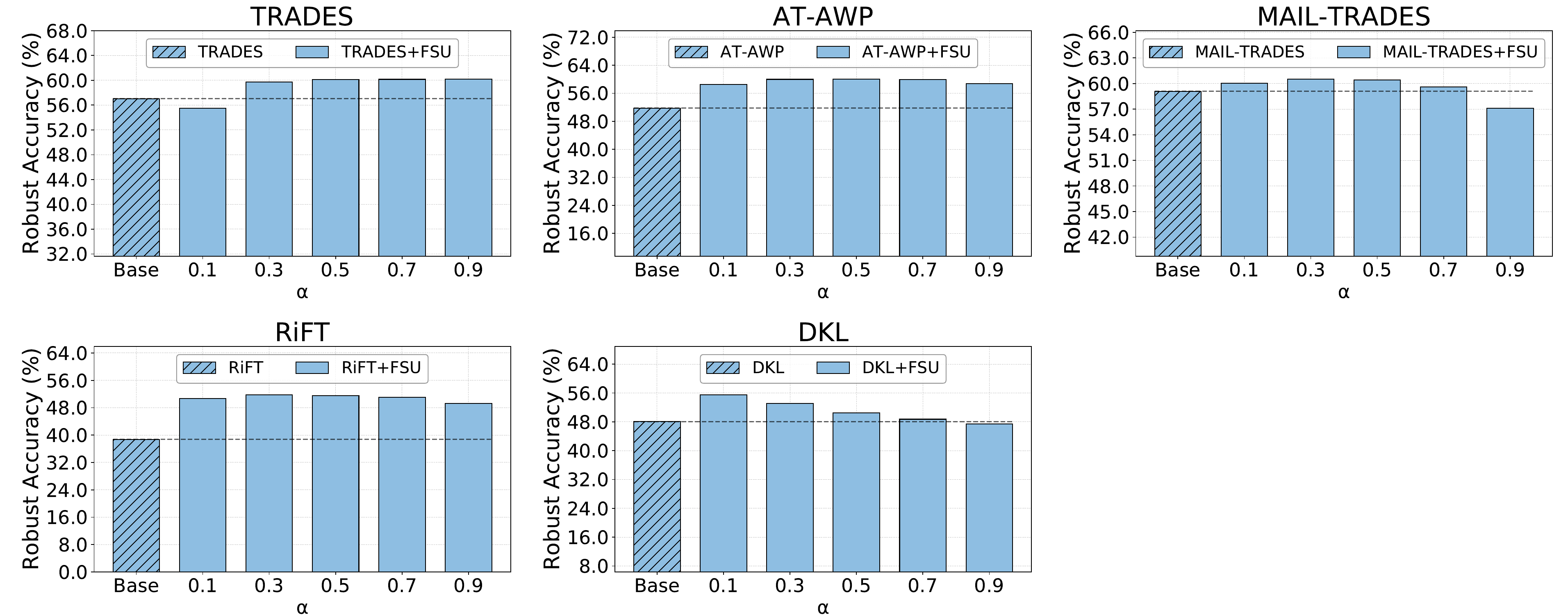}}\\
\subfloat[\scriptsize SVHN - AA$_\infty$ Attack ($\epsilon=8/255$)]{\includegraphics[width=0.49\linewidth]{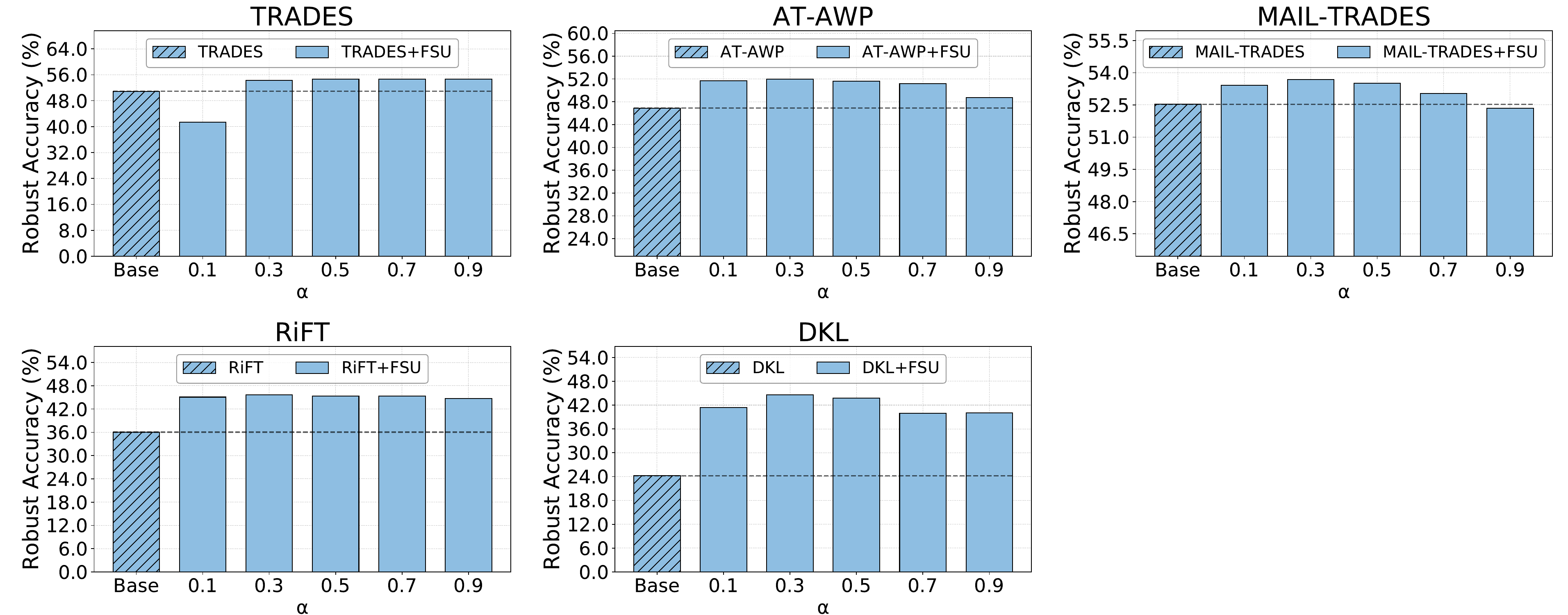}}\quad
\subfloat[\scriptsize SVHN - CW$_2$ Attack ($\epsilon=8/255$, iterations $50$, learning rate $0.01$)]{\includegraphics[width=0.49\linewidth]{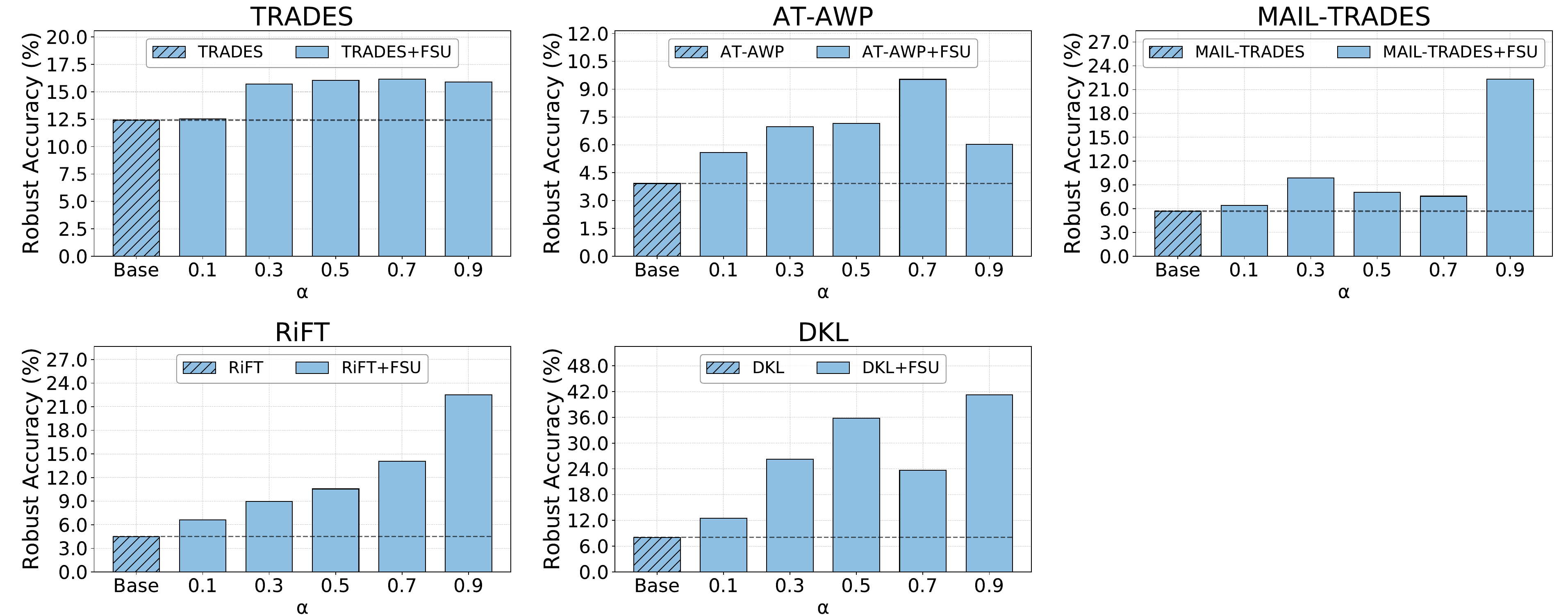}}
\caption{Hyper-parameter sensitivity analysis on dataset SVHN.}
\label{fig.MoreSensitivity_svhn}
\end{figure*}

\begin{figure*}[h]
\centering
\subfloat[\scriptsize CIFAR100 - FGSM$_\infty$ Attack ($\epsilon=8/255$)]{\includegraphics[width=0.49\linewidth]{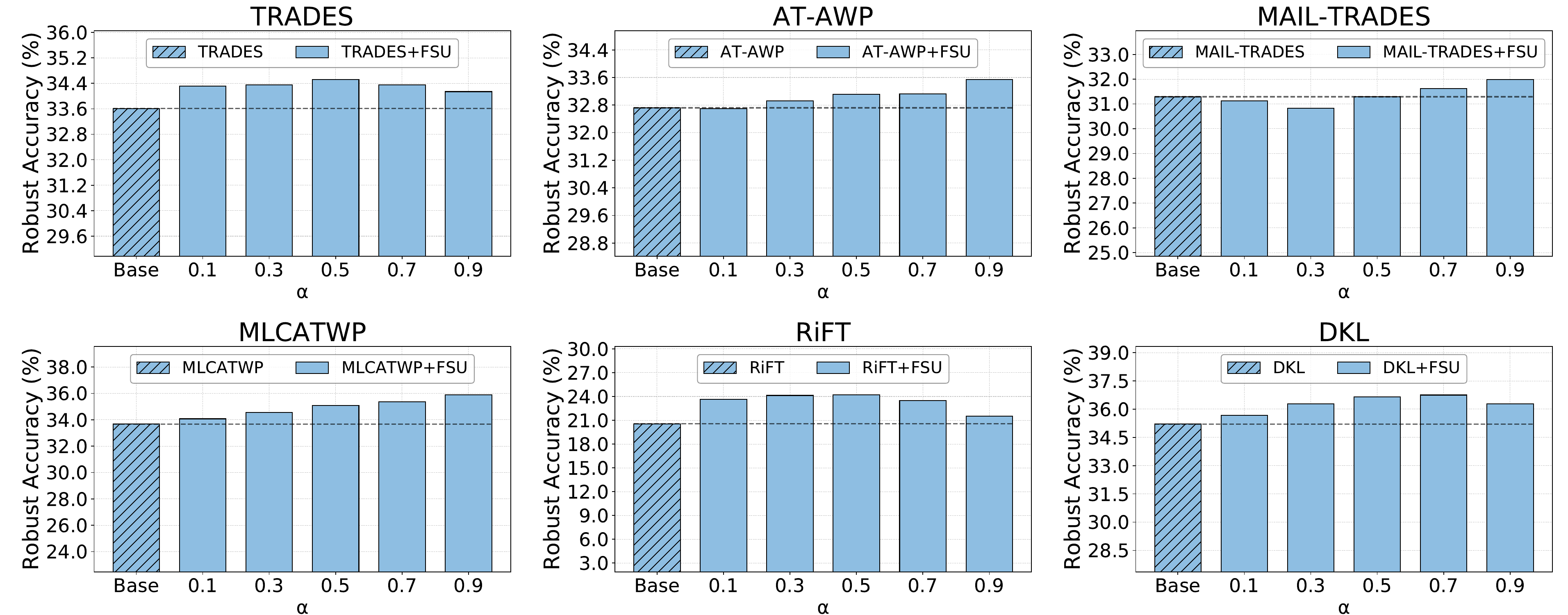}}\quad
\subfloat[\scriptsize CIFAR100 - PGD$_\infty$-100 Attack ($\epsilon=8/255$, step size $2/255$)]{\includegraphics[width=0.49\linewidth]{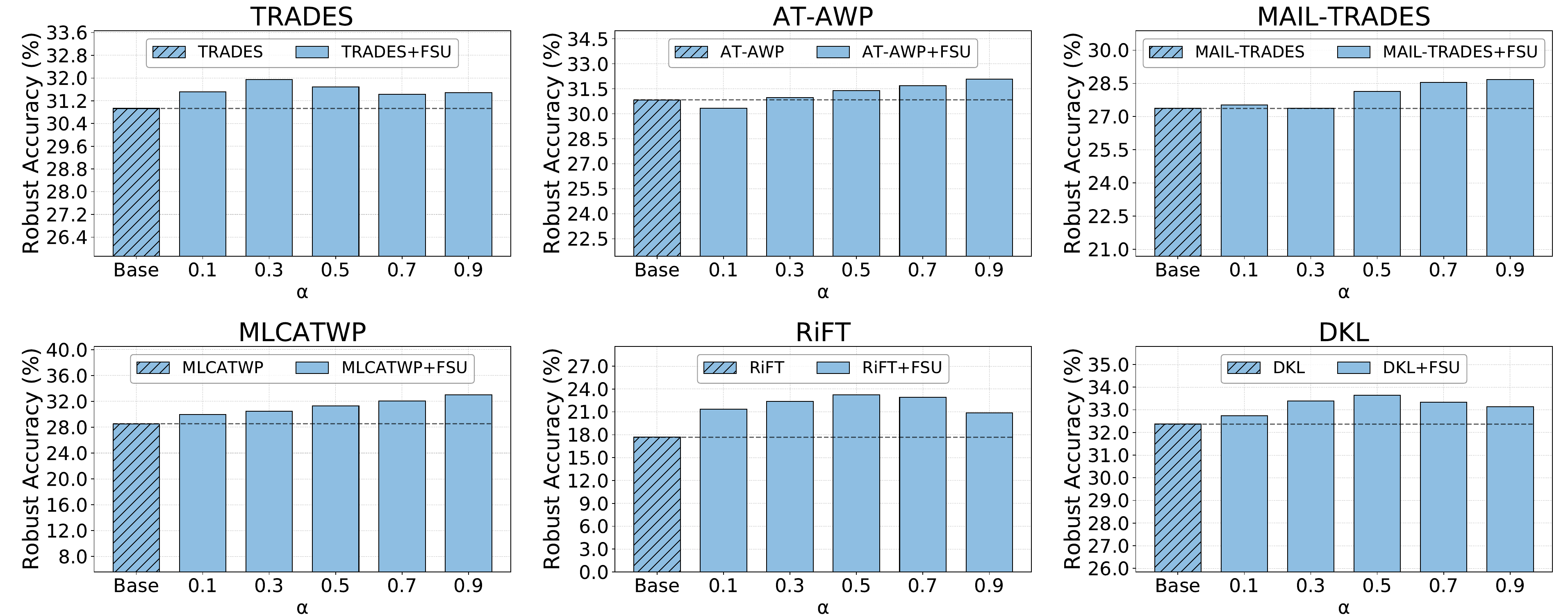}}\\
\subfloat[\scriptsize CIFAR100 - AA$_\infty$ Attack ($\epsilon=8/255$)]{\includegraphics[width=0.49\linewidth]{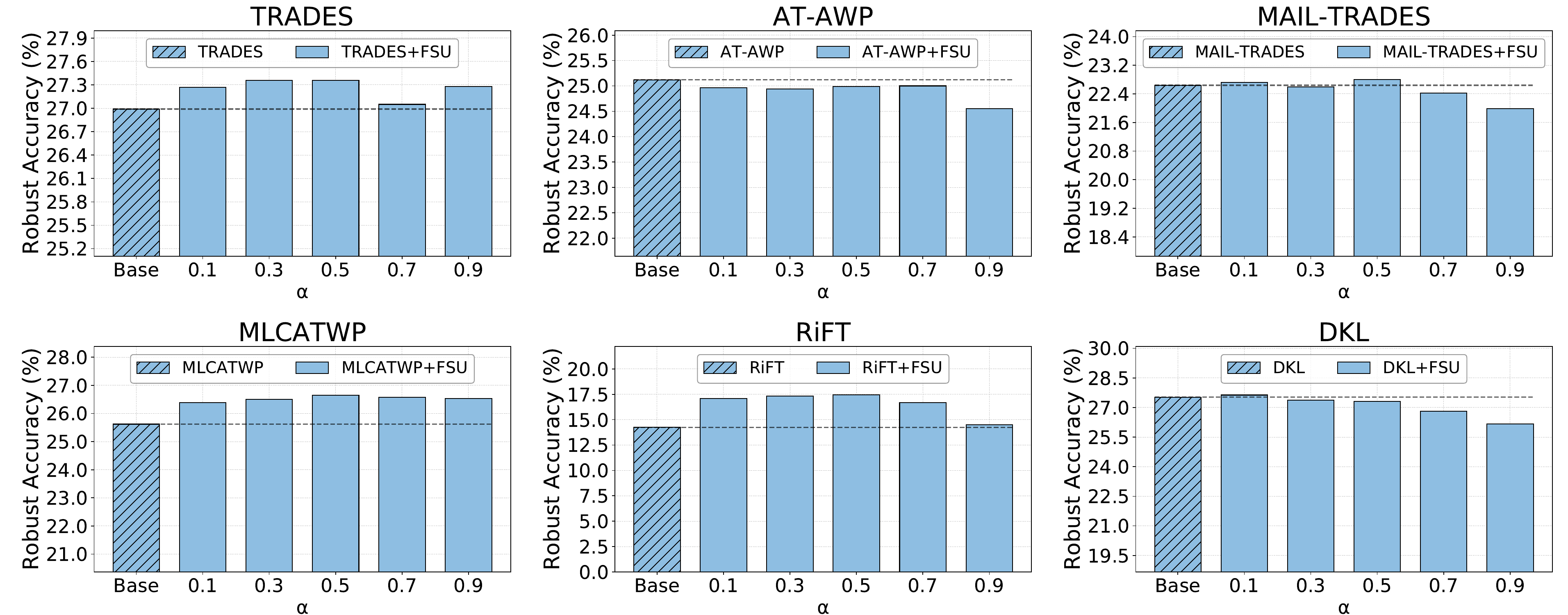}}\quad
\subfloat[\scriptsize CIFAR100 - CW$_2$ Attack ($\epsilon=8/255$, iterations $50$, learning rate $0.01$)]{\includegraphics[width=0.49\linewidth]{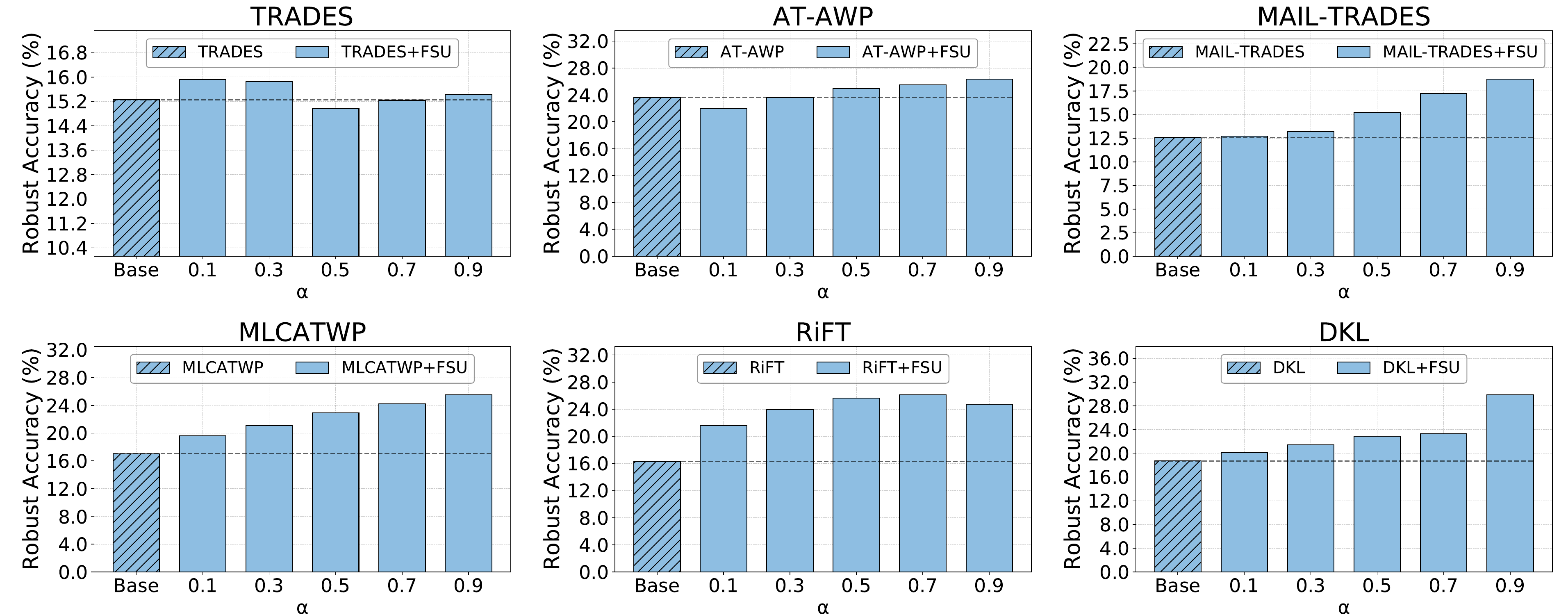}}
\caption{Hyper-parameter sensitivity analysis on dataset CIFAR100.}
\label{fig.MoreSensitivity_cifar100}
\end{figure*}

\newpage
\subsection{More Calibration Figures for Distribution Shift}\label{append.calibration}

\begin{figure*}[h]
\centering
\subfloat[\scriptsize Mean Shift (Baseline-Train)]{\includegraphics[width=0.245\linewidth]{non-mnist-pgd-train-mean.pdf}}\
\subfloat[\scriptsize Mean Shift (Baseline-Test)]{\includegraphics[width=0.245\linewidth]{non-mnist-pgd-test-mean.pdf}}\
\subfloat[\scriptsize Std Shift (Baseline-Train)]{\includegraphics[width=0.245\linewidth]{non-mnist-pgd-train-std.pdf}}\
\subfloat[\scriptsize Std Shift (Baseline-Test)]{\includegraphics[width=0.245\linewidth]{non-mnist-pgd-test-std.pdf}}\\
\subfloat[\scriptsize Mean Shift (FSU-Train)]{\includegraphics[width=0.245\linewidth]{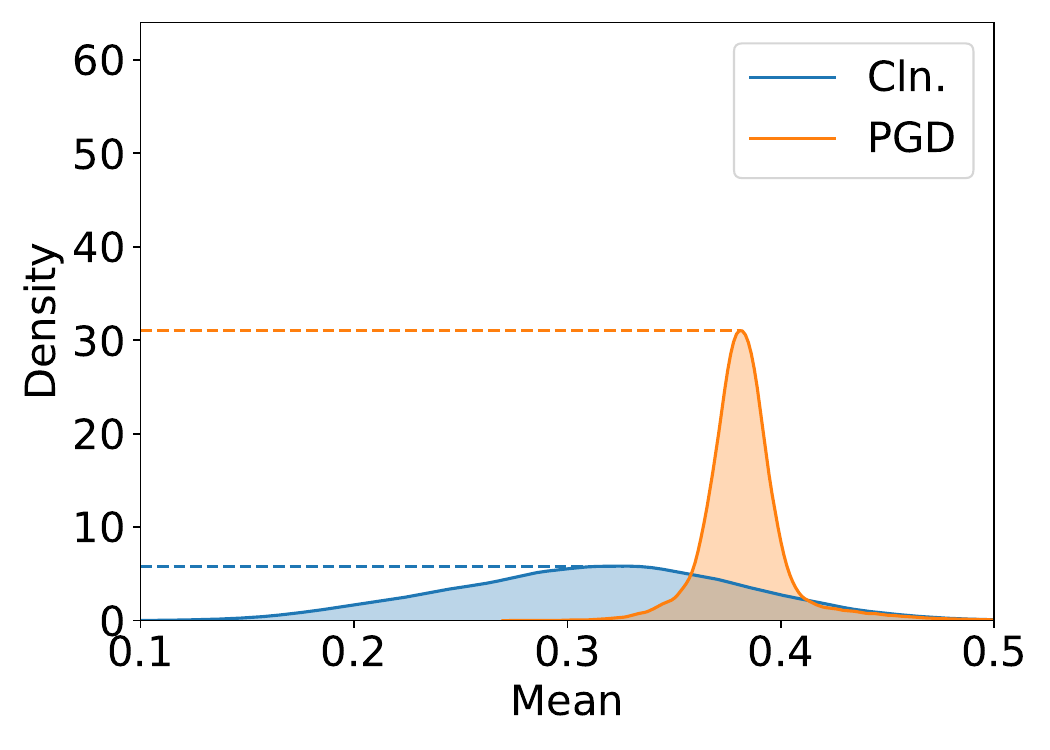}}\
\subfloat[\scriptsize Mean Shift (FSU-Test)]{\includegraphics[width=0.245\linewidth]{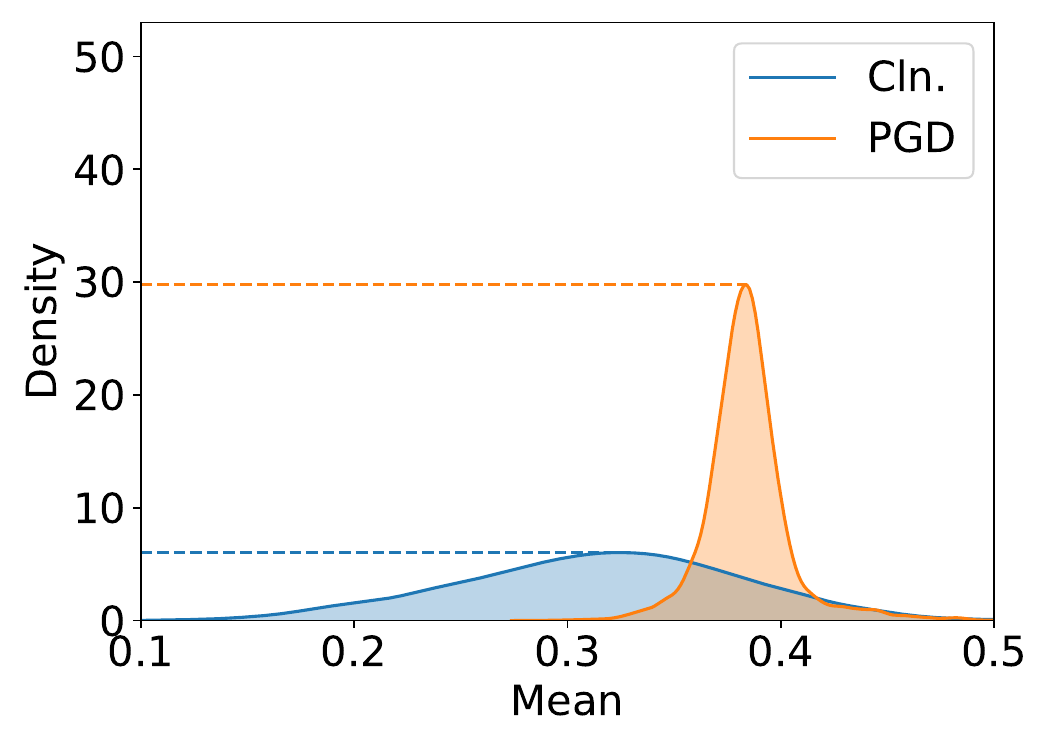}}\
\subfloat[\scriptsize Std Shift (FSU-Train)]{\includegraphics[width=0.245\linewidth]{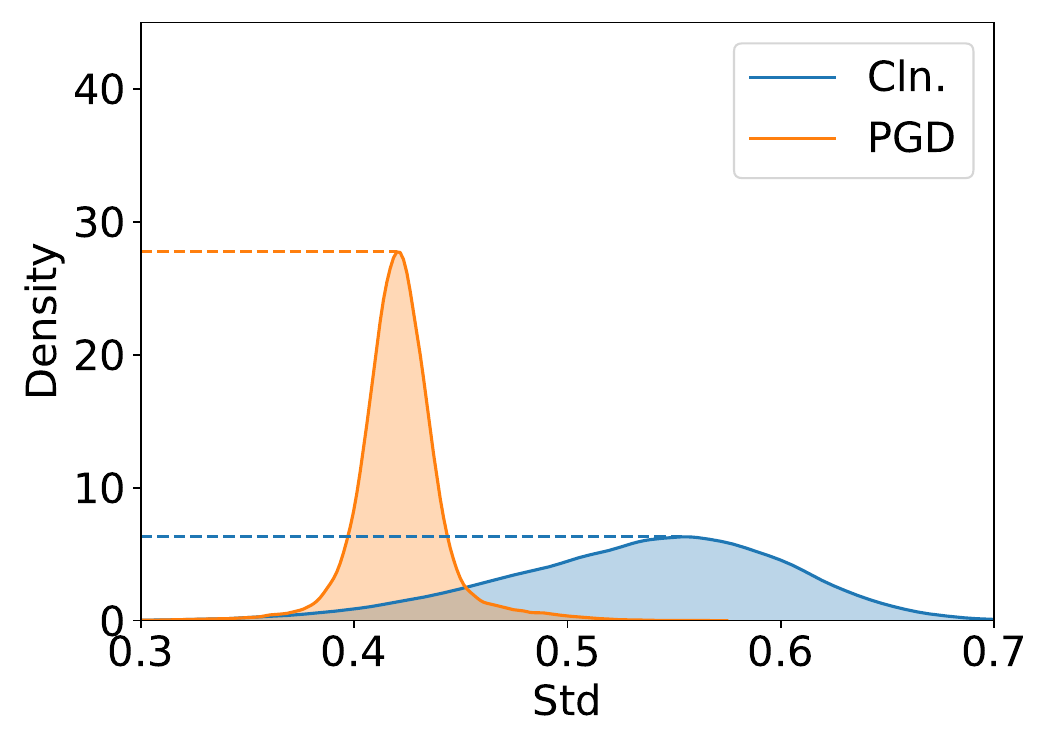}}\
\subfloat[\scriptsize Std Shift (FSU-Test)]{\includegraphics[width=0.245\linewidth]{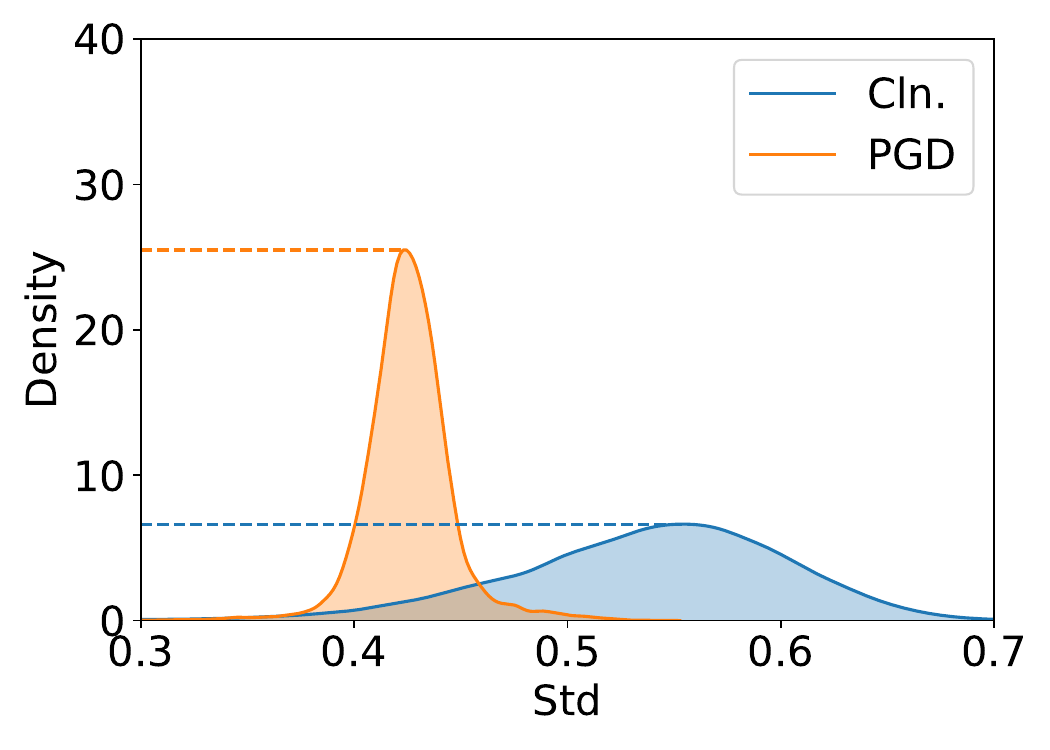}}
\caption{Calibration for distribution shifts of feature statistics in dataset MNIST (average of 84 channels in latent space, LeNet-5 model). ``Cln." and ``PGD" represent clean examples and examples perturbed by PGD-40 attack, respectively. ``Baseline" and ``FSU" mean the naturally trained models without and with the FSU module, respectively. ``Train" and ``Test" represent training data and testing data, respectively.}
\label{fig.featureDisMnist}
\end{figure*}

\begin{figure*}[h]
\centering
\subfloat[\scriptsize Mean Shift (Baseline-Train)]{\includegraphics[width=0.245\linewidth]{non-svhn-pgd-train-mean.pdf}}\
\subfloat[\scriptsize Mean Shift (Baseline-Test)]{\includegraphics[width=0.245\linewidth]{non-svhn-pgd-test-mean.pdf}}\
\subfloat[\scriptsize Std Shift (Baseline-Train)]{\includegraphics[width=0.245\linewidth]{non-svhn-pgd-train-std.pdf}}\
\subfloat[\scriptsize Std Shift (Baseline-Test)]{\includegraphics[width=0.245\linewidth]{non-svhn-pgd-test-std.pdf}}\\
\subfloat[\scriptsize Mean Shift (FSU-Train)]{\includegraphics[width=0.245\linewidth]{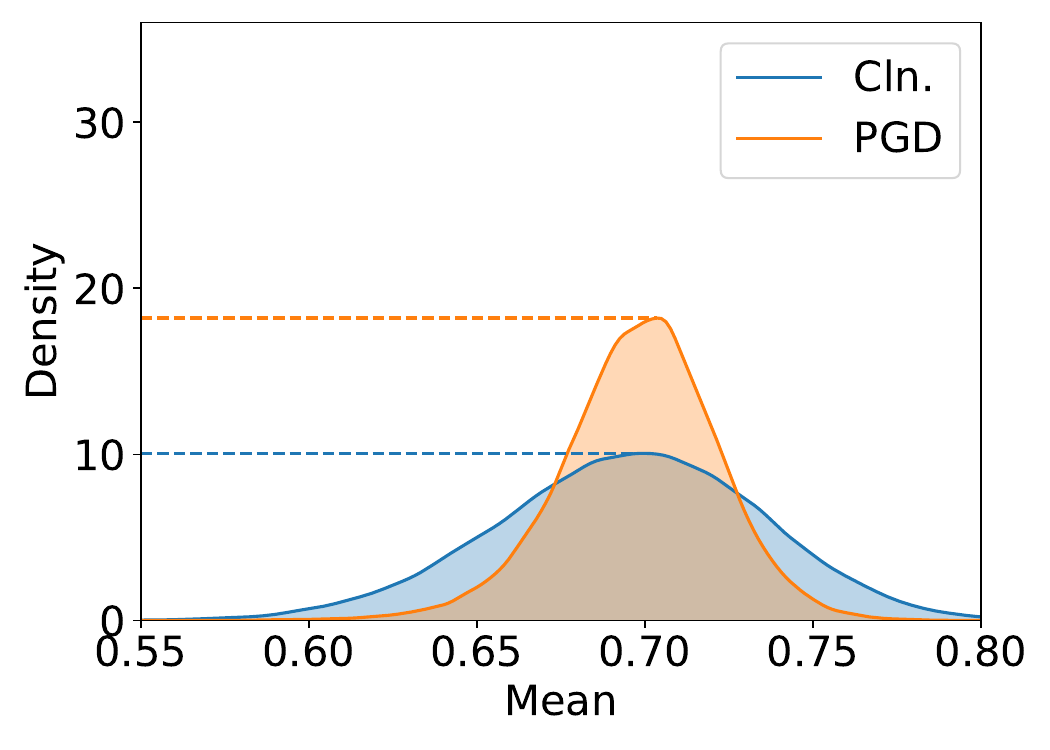}}\
\subfloat[\scriptsize Mean Shift (FSU-Test)]{\includegraphics[width=0.245\linewidth]{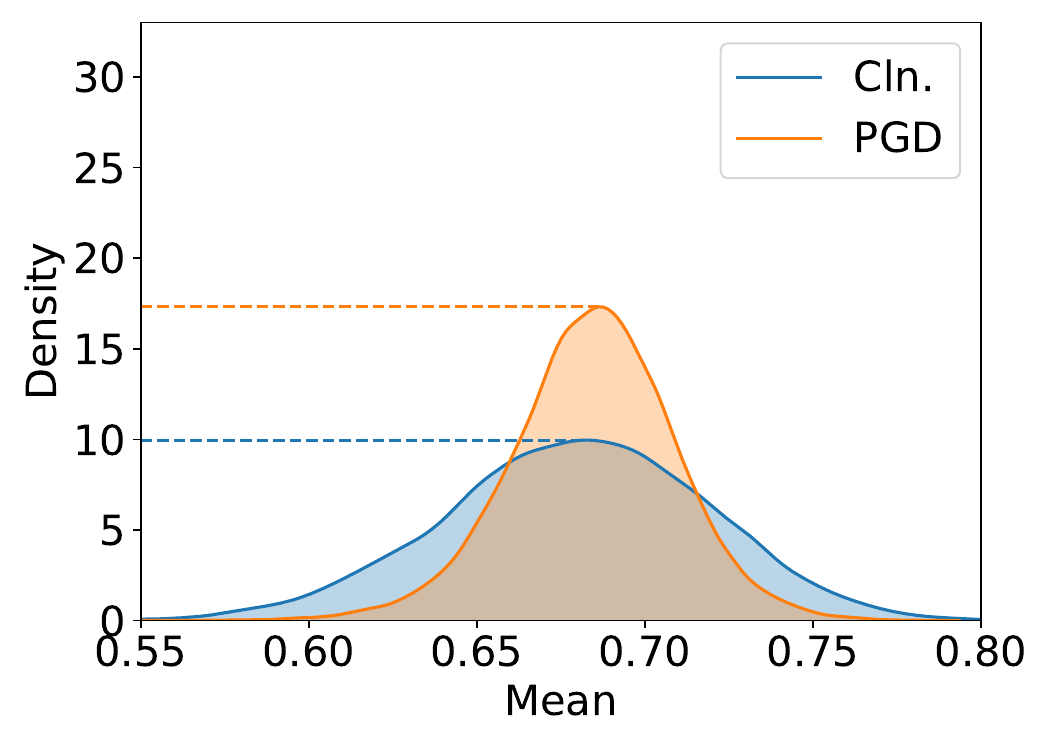}}\
\subfloat[\scriptsize Std Shift (FSU-Train)]{\includegraphics[width=0.245\linewidth]{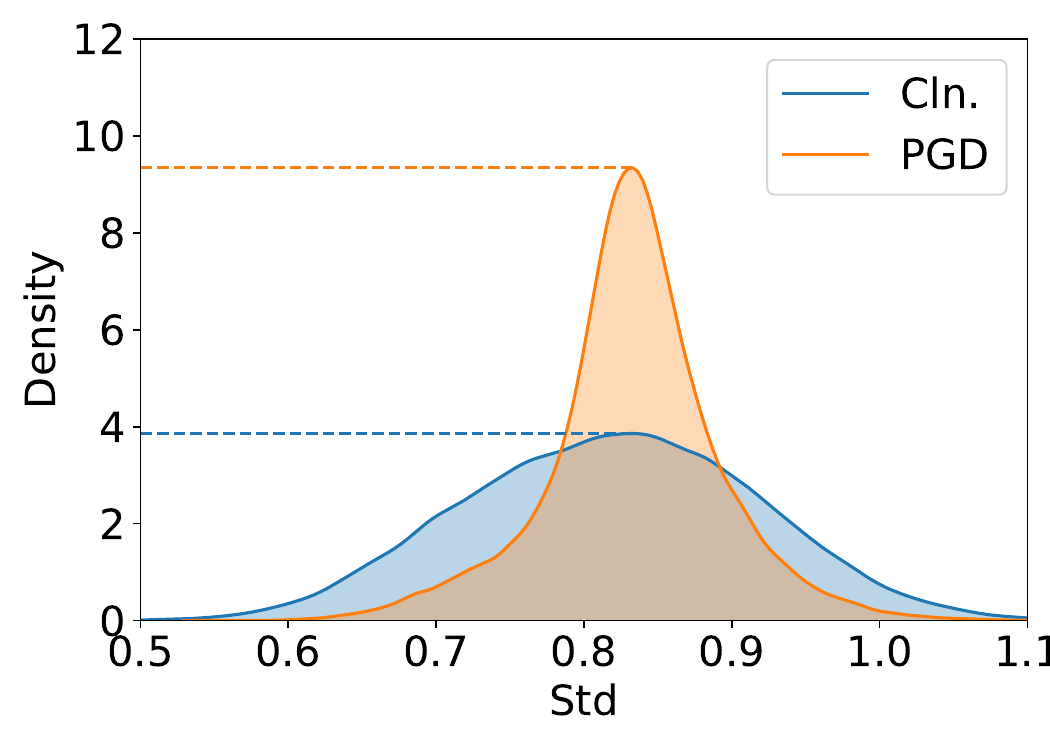}}\
\subfloat[\scriptsize Std Shift (FSU-Test)]{\includegraphics[width=0.245\linewidth]{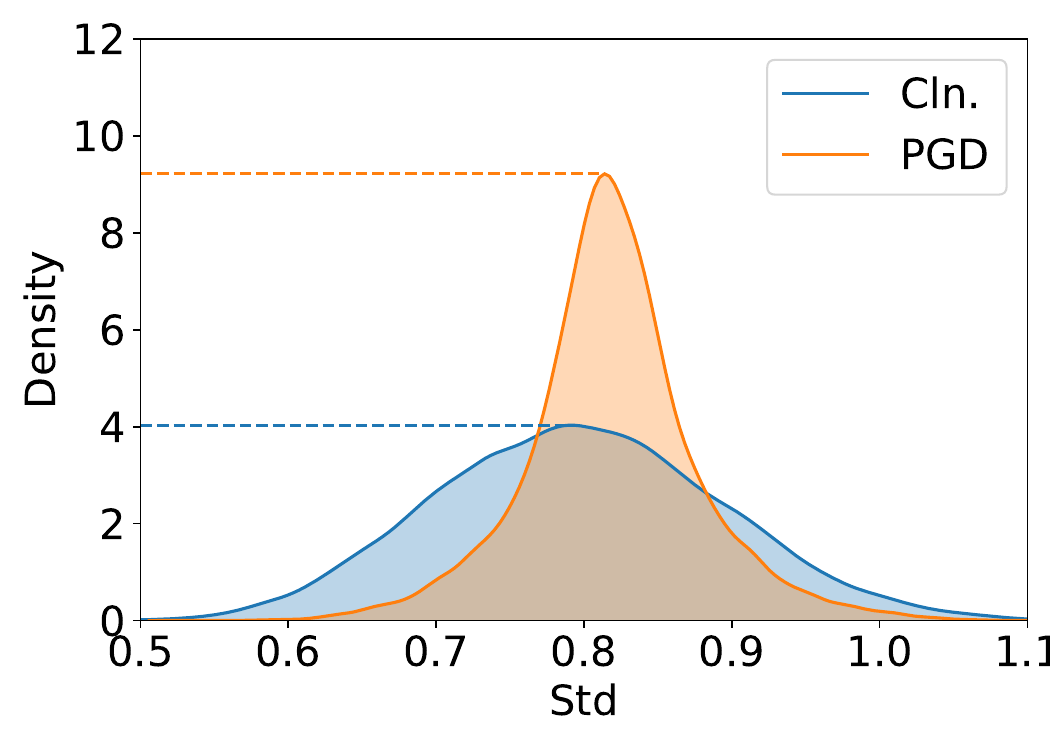}}
\caption{Calibration for distribution shifts of feature statistics in dataset SVHN (average of 512 channels in latent space, ResNet-18 model). ``Cln." and ``PGD" represent clean examples and examples perturbed by PGD-10 attack, respectively. ``Baseline" and ``FSU" mean the naturally trained models without and with the FSU module, respectively. ``Train" and ``Test" represent training data and testing data, respectively.}
\label{fig.featureDisSvhn}
\end{figure*}

\newpage
\subsection{More Feature Map Visualizations}\label{append.visual}

Figs.~\ref{fig.featureMapMNIST1PGD}$\sim$\ref{fig.featureMapCIFAR10FGSM} provide more details of the feature maps of randomly selected testing examples perturbed by PGD and  FGSM. Note that the feature maps are extracted from the layers corresponding to the reconstruction positions. It can be seen that the FSU module does help the perturbed examples recover some feature points that posses important semantic information. Especially, from different attacking iterations of PGD in Figs.~\ref{fig.featureMapMNIST1PGD}$\sim$\ref{fig.featureMapMNIST2PGD}, we can see that the adversarial perturbations become stronger and stronger without the FSU module, while with the FSU module, the perturbations have been mitigated to a great extent. Furthermore, from Figs.~\ref{fig.featureMapMNISTFGSM}$\sim$\ref{fig.featureMapCIFAR10FGSM}, we can see that both the baseline model and the FSU model can provide high-quality feature maps for the clean examples, while under FGSM attack, the feature quality provided by the FSU model is much higher than that of the baseline model.

\begin{figure*}[h]
\centering
\subfloat[Cln.]{\includegraphics[width=0.1\linewidth]{sample0.png}}
\subfloat[Ep 5]{\includegraphics[width=0.1\linewidth]{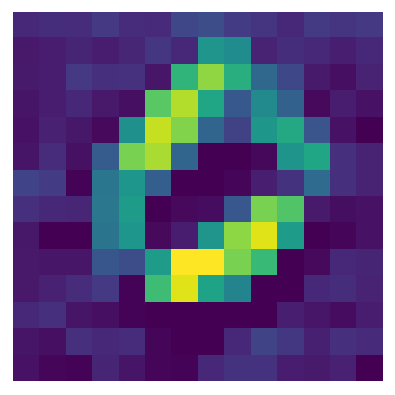}}
\subfloat[Ep 10]{\includegraphics[width=0.1\linewidth]{0-base-iter10.png}}
\subfloat[Ep 15]{\includegraphics[width=0.1\linewidth]{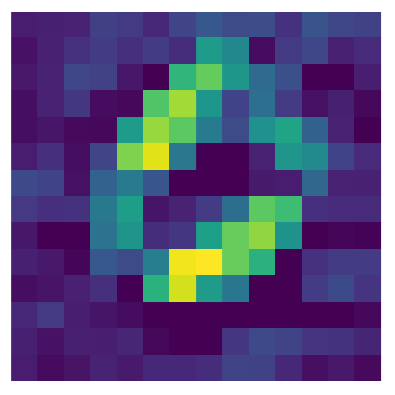}}
\subfloat[Ep 20]{\includegraphics[width=0.1\linewidth]{0-base-iter20.png}}
\subfloat[Ep 25]{\includegraphics[width=0.1\linewidth]{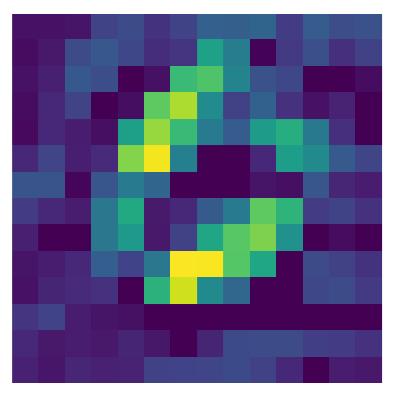}}
\subfloat[Ep 30]{\includegraphics[width=0.1\linewidth]{0-base-iter30.png}}
\subfloat[Ep 35]{\includegraphics[width=0.1\linewidth]{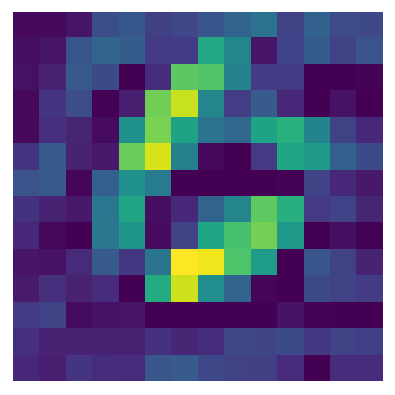}}
\subfloat[Ep 40]{\includegraphics[width=0.1\linewidth]{0-base-iter40.png}}\\
\subfloat[Cln.]{\includegraphics[width=0.1\linewidth]{sample0.png}}
\subfloat[Ep  5]{\includegraphics[width=0.1\linewidth]{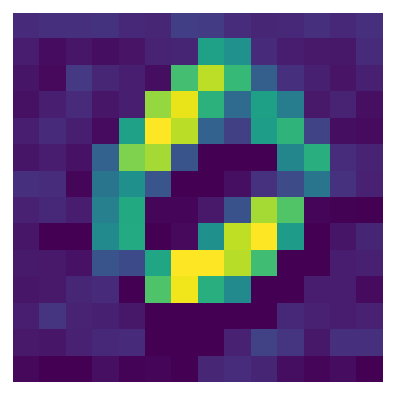}}
\subfloat[Ep 10]{\includegraphics[width=0.1\linewidth]{0-ours-iter10.png}}
\subfloat[Ep 15]{\includegraphics[width=0.1\linewidth]{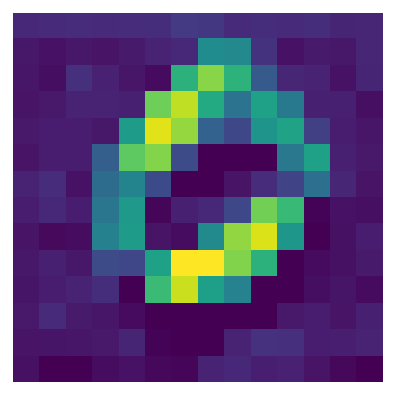}}
\subfloat[Ep 20]{\includegraphics[width=0.1\linewidth]{0-ours-iter20.png}}
\subfloat[Ep 25]{\includegraphics[width=0.1\linewidth]{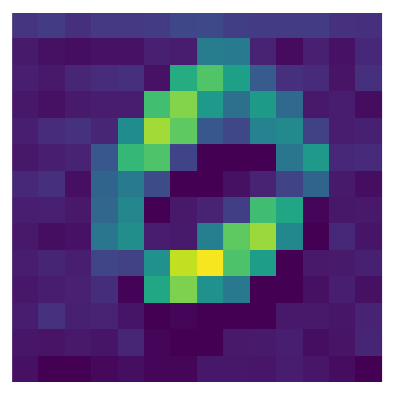}}
\subfloat[Ep 30]{\includegraphics[width=0.1\linewidth]{0-ours-iter30.png}}
\subfloat[Ep 35]{\includegraphics[width=0.1\linewidth]{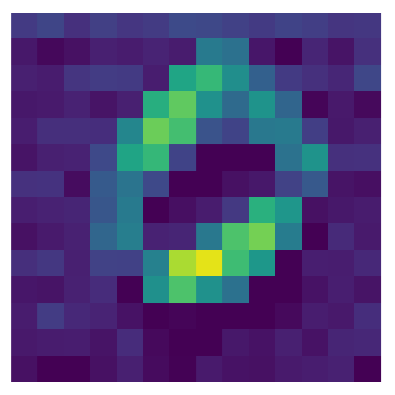}}
\subfloat[Ep 40]{\includegraphics[width=0.1\linewidth]{0-ours-iter40.png}}
\caption{Feature maps of a randomly selected channel for a randomly selected testing example in MNIST during PGD-40 attack (with LeNet-5 backbone). (a)$\sim$(i) Baseline-NT. (j)$\sim$(r) FSU-NT. (Ep  means Epoch.)}
\label{fig.featureMapMNIST1PGD}
\end{figure*}

\begin{figure*}[h]
\centering
\subfloat[Cln.]{\includegraphics[width=0.1\linewidth]{sample2.png}}
\subfloat[Ep 5]{\includegraphics[width=0.1\linewidth]{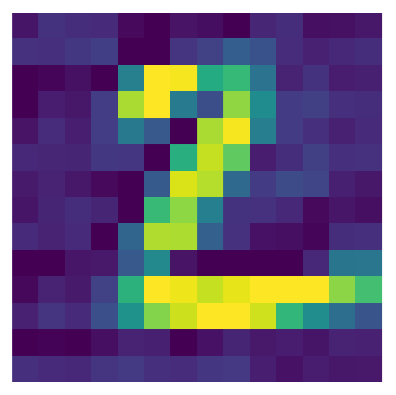}}
\subfloat[Ep 10]{\includegraphics[width=0.1\linewidth]{2-base-iter10.png}}
\subfloat[Ep 15]{\includegraphics[width=0.1\linewidth]{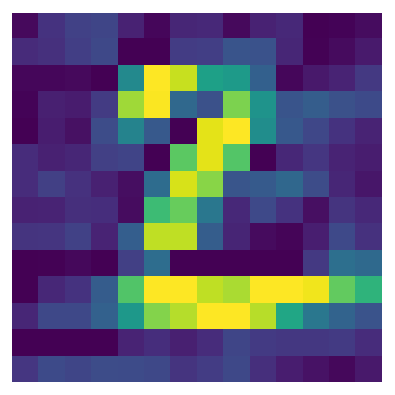}}
\subfloat[Ep 20]{\includegraphics[width=0.1\linewidth]{2-base-iter20.png}}
\subfloat[Ep 25]{\includegraphics[width=0.1\linewidth]{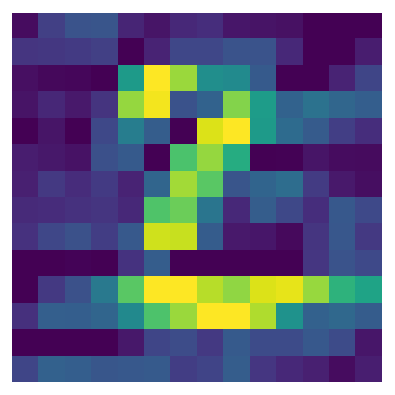}}
\subfloat[Ep 30]{\includegraphics[width=0.1\linewidth]{2-base-iter30.png}}
\subfloat[Ep 35]{\includegraphics[width=0.1\linewidth]{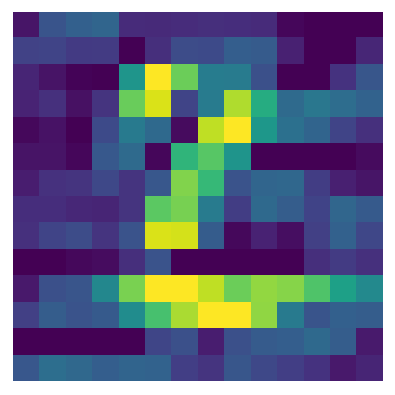}}
\subfloat[Ep 40]{\includegraphics[width=0.1\linewidth]{2-base-iter40.png}}\\
\subfloat[Cln.]{\includegraphics[width=0.1\linewidth]{sample2.png}}
\subfloat[Ep 5]{\includegraphics[width=0.1\linewidth]{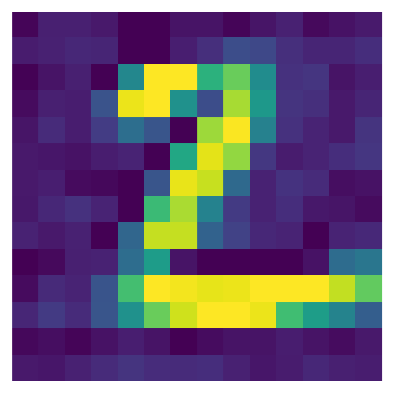}}
\subfloat[Ep 10]{\includegraphics[width=0.1\linewidth]{2-ours-iter10.png}}
\subfloat[Ep 15]{\includegraphics[width=0.1\linewidth]{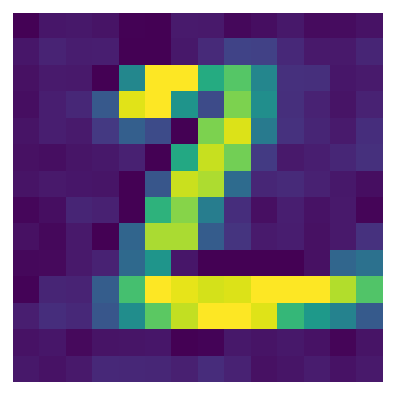}}
\subfloat[Ep 20]{\includegraphics[width=0.1\linewidth]{2-ours-iter20.png}}
\subfloat[Ep 25]{\includegraphics[width=0.1\linewidth]{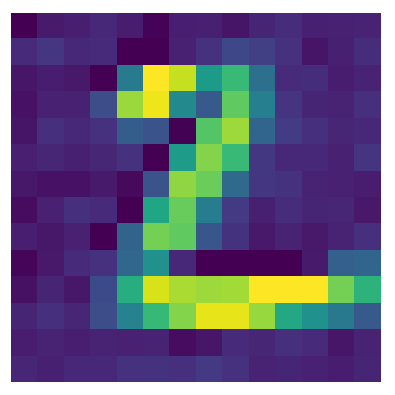}}
\subfloat[Ep 30]{\includegraphics[width=0.1\linewidth]{2-ours-iter30.png}}
\subfloat[Ep 35]{\includegraphics[width=0.1\linewidth]{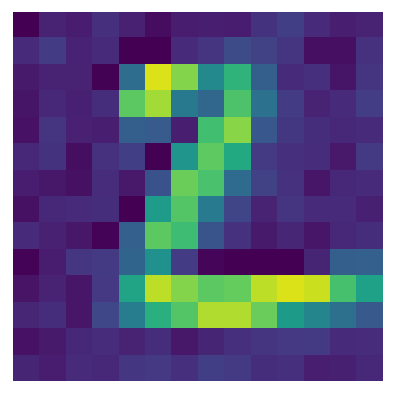}}
\subfloat[Ep 40]{\includegraphics[width=0.1\linewidth]{2-ours-iter40.png}}
\caption{Feature maps of a randomly selected channel for a randomly selected testing example in MNIST during PGD-40 attack (with LeNet-5 backbone). (a)$\sim$(i) Baseline-NT. (j)$\sim$(r) FSU-NT. (Ep  means Epoch.)}
\label{fig.featureMapMNIST2PGD}
\end{figure*}

\begin{figure*}[h]
\centering
\subfloat[Cln: Baseline-NT]{\includegraphics[width=0.22\linewidth]{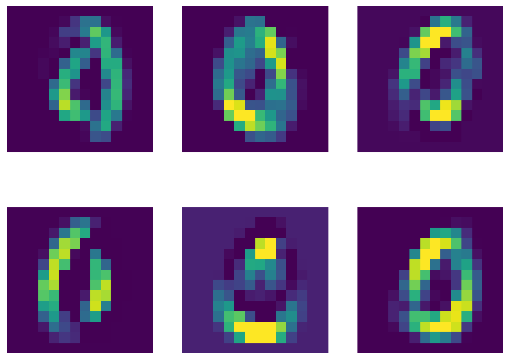}}\quad
\subfloat[Cln: FSU-NT]{\includegraphics[width=0.22\linewidth]{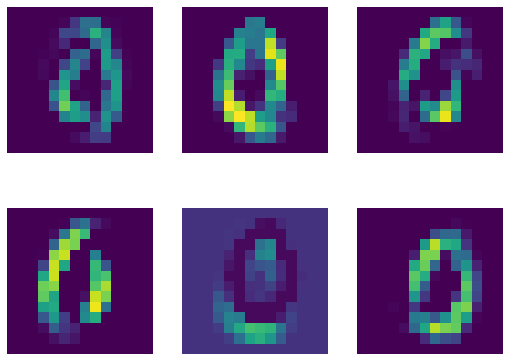}}\quad\quad
\subfloat[FGSM: Baseline-NT]{\includegraphics[width=0.22\linewidth]{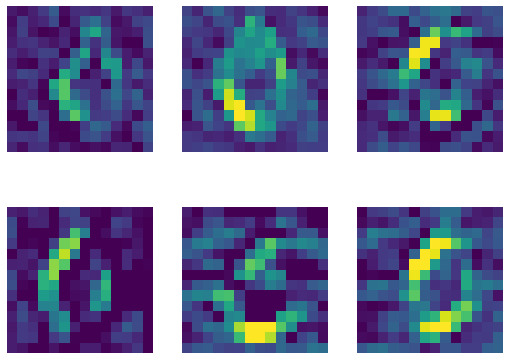}}\quad
\subfloat[FGSM: FSU-NT]{\includegraphics[width=0.22\linewidth]{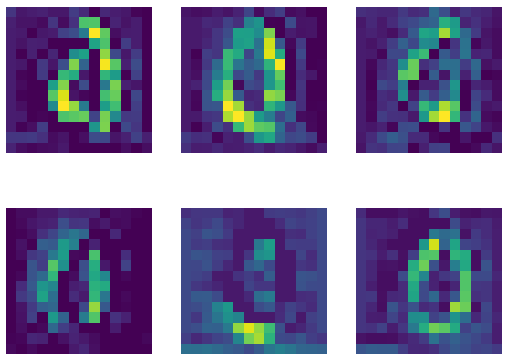}}\\
\subfloat[Cln: Baseline-NT]{\includegraphics[width=0.22\linewidth]{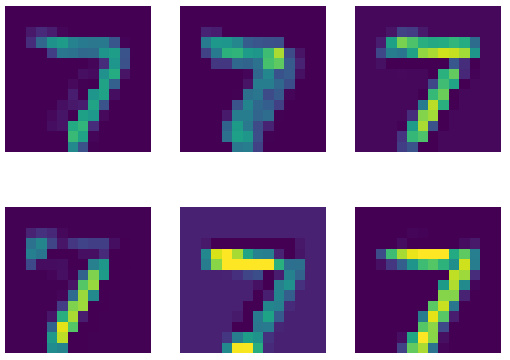}}\quad
\subfloat[Cln: FSU-NT]{\includegraphics[width=0.22\linewidth]{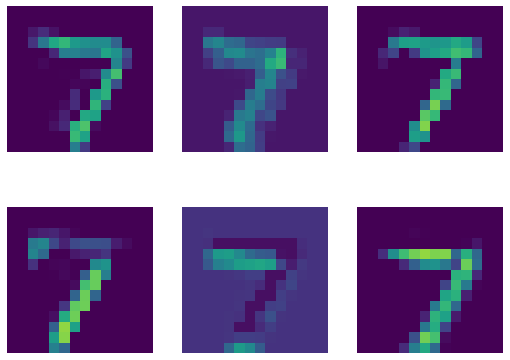}}\quad\quad
\subfloat[FGSM: Baseline-NT]{\includegraphics[width=0.22\linewidth]{adv_origin_7_3.png}}\quad
\subfloat[FGSM: FSU-NT]{\includegraphics[width=0.22\linewidth]{adv_dsu_7.png}}\\
\subfloat[Cln: Baseline-NT]{\includegraphics[width=0.22\linewidth]{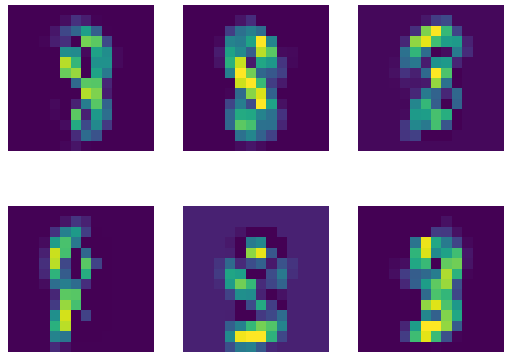}}\quad
\subfloat[Cln: FSU-NT]{\includegraphics[width=0.22\linewidth]{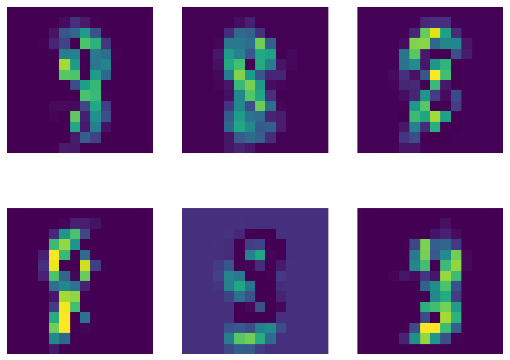}}\quad\quad
\subfloat[FGSM: Baseline-NT]{\includegraphics[width=0.22\linewidth]{adv_origin_8_3.png}}\quad
\subfloat[FGSM: FSU-NT]{\includegraphics[width=0.22\linewidth]{adv_dsu_8.png}}
\caption{Feature maps of six channels for randomly selected testing examples in MNIST (with LeNet-5 backbone).}
\label{fig.featureMapMNISTFGSM}
\end{figure*}

\begin{figure*}[h]
\centering
\subfloat[FGSM: Baseline-NT]{\includegraphics[width=0.3\linewidth]{adv_origin_car1.png}}\quad\quad
\subfloat[FGSM: Baseline-NT]{\includegraphics[width=0.3\linewidth]{adv_origin_car2.png}}\quad\quad
\subfloat[FGSM: Baseline-NT]{\includegraphics[width=0.3\linewidth]{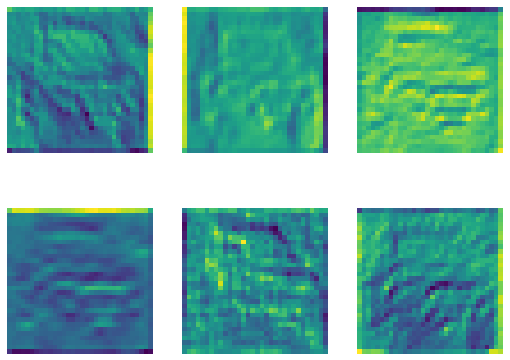}}\quad\quad
\subfloat[FGSM: FSU-NT]{\includegraphics[width=0.3\linewidth]{adv_dsu_car1.png}}\quad\quad
\subfloat[FGSM: FSU-NT]{\includegraphics[width=0.3\linewidth]{adv_dsu_car2.png}}\quad\quad
\subfloat[FGSM: FSU-NT]{\includegraphics[width=0.3\linewidth]{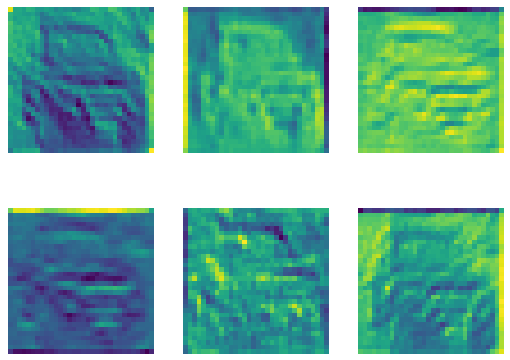}}
\caption{Feature maps of six channels for randomly selected testing examples in CIFAR10 (with ResNet-18 backbone).}
\label{fig.featureMapCIFAR10FGSM}
\end{figure*}

\end{document}